\documentclass[twoside,11pt]{article}
\RequirePackage[OT1]{fontenc}
\RequirePackage{amsthm,amsmath}

\usepackage{jmlr2e}

\usepackage{graphicx}  

\usepackage[utf8]{inputenc} 
\usepackage[T1]{fontenc}    
\usepackage{hyperref} 
\usepackage{nicefrac}
\usepackage{url}   
\usepackage{natbib}
\usepackage{booktabs}       
\usepackage{amsfonts}       
\usepackage{nicefrac}       
\usepackage{microtype}      
\usepackage{xcolor}         
\usepackage{ amssymb }
\usepackage{amsthm}
\usepackage{algorithm}
\usepackage{multirow}
\usepackage{url}  
\usepackage{verbatim}
\usepackage{color}
\usepackage{graphicx}  
\usepackage{algorithm}
\usepackage{algorithmic}
\newsavebox{\algbox}
\usepackage{mathtools}
\usepackage{caption}
\usepackage{subcaption}
\usepackage{scalerel,stackengine}

\newcommand{\f}[1]{\boldsymbol{#1}}
\newcommand{\bb}[1]{\mathbb{#1}}
\newcommand{\fl}[1]{\mathbf{#1}}
\newcommand{\ca}[1]{\mathcal{#1}}

\newcommand{\s}[1]{\mathsf{#1}}

\newcommand{\Su}{\mathsf{occ}}

\newcommand{\supp}[1]{\mathsf{supp}(#1)}

\newcommand{\remove}[1]{}

\begin{document}

\title{Support Recovery in Mixture Models with Sparse Parameters}

\author{\name Arya Mazumdar \email arya@ucsd.edu \\
       \addr 
       University of California, San Diego\\
       La Jolla, CA 92093 USA
       \AND
       \name Soumyabrata\ Pal \email soumyabratapal13@gmail.com \\
       \addr University of Massachusetts Amherst\\
       }

\editor{}

\maketitle

\begin{abstract}
  Mixture models are widely used to fit complex and multimodal datasets. In this paper we study mixtures with high dimensional sparse  parameter vectors and consider the problem of support recovery of those vectors. While parameter learning in mixture models is well-studied, the sparsity constraint remains relatively unexplored. Sparsity of parameter vectors is a natural constraint in variety of settings, and  support recovery is a major step towards parameter estimation.
We provide efficient algorithms for support recovery that have a logarithmic sample complexity dependence on the dimensionality of the latent space. Our algorithms are quite  general, namely they are applicable to 1) mixtures of many different canonical distributions including Uniform, Poisson, Laplace, Gaussians, etc. 2) Mixtures of linear regressions and linear classifiers with Gaussian covariates under different assumptions on the unknown parameters. In most of these settings, our results are the first guarantees on the problem while in the rest, our results provide  improvements on existing works.
\end{abstract}


\section{Introduction}\label{sec:intro}
Mixture models are standard tools for probabilistic modeling of heterogeneous data, and have been studied theoretically for more than a century.  Mixtures are used in practice for modeling data across different fields, such as, astronomy, genetics, medicine, psychiatry, economics, and marketing among many others \citep{moosman2000finite}. Mixtures with finite number of components are especially successful in modeling datasets having a group structure, or  presence of a subpopulation within the overall population. Often,  mixtures can handle situations where a single parametric family cannot provide a satisfactory model for local variations in the observed data~\citep{titterington1985statistical}.

The literature on algorithmically learning  mixture distributions is quite vast and comes in different flavors. Computational and statistical aspects of learning mixtures perhaps starts with \cite{dasgupta1999learning}, and since then have been the subject
of intense investigation in both computer science and
statistics~\citep{achlioptas2005spectral,kalai2010efficiently,belkin2010polynomial,arora2001learning,moitra2010settling,feldman2008learning,chan2014efficient,acharya2017sample,hopkins2018mixture,diakonikolas2018list,kothari2018robust,hardt2015tight}.
A large portion of this literature is devoted to \emph{density estimation} or PAC-learning, where the goal is simply to
find a distribution that is close in some distance (e.g., TV distance)
to the data-generating mechanism. The results on density estimation can be further
subdivided into \emph{proper} and \emph{improper learning} approaches
depending on whether the algorithm outputs a distribution from the
given mixture family or not. These two guarantees turn out to be quite different.

A significant part of the literature on the other hand is devoted to
\emph{parameter estimation}, where the goal is to identify the mixing
weights and the parameters of each component from samples. Apart from Gaussian mixtures, where all types of results
exist, prior work for other mixture families largely focuses on
density estimation, and very little is known for parameter estimation
outside of Gaussian mixture models.  
In this paper, our focus is to facilitate parameter estimation in Gaussian mixtures and beyond. We consider the setting where the parameters of the mixture are themselves high dimensional, but sparse (i.e., have few nonzero entries). Sparsity is a natural regularizer in high dimensional parameter estimation problems and have been considered in the context of mixtures in \cite{verzelen2017detection,arias2017simple,azizyan2013minimax}, where it is assumed only few dimensions of the component means are relevant for de-mixing. In this paper we consider a slightly different model where we assume the means themselves are sparse. The former problem can be reduced to our setting if one of the component means is known.

There are parameter estimation problems in other data subpopulation modeling, where functional relationships in data can be thought of as mixture of simple component models. Most prominent among these is the {\em mixed linear regression} problem~\citep{de1989mixtures}. In this setting, each sample is a tuple of (covariates, label). The label is stochastically generated by picking a linear relation uniformly from a set of two or more linear functions, evaluating this function on the covariates and possibly adding noise. The goal is to learn the set of  unknown linear functions. 
The problem has been studied widely \citep{chaganty2013spectral, faria2010fitting, stadler2010l,li2018learning,kwon2018global,viele2002modeling,yi2014alternating,yi2016solving}, with an emphasis on the EM algorithm and other alternating minimization (AM) techniques. It is interesting that \cite{stadler2010l} argued to impose sparsity on the solutions, implying that each linear function depends on only a small number of variables. In this paper we are concerned with exactly this same problem.

Similar to mixed linear regressions, there can be {\em mixed linear classifications.} In that setting, the labels are binary (or other categorical). The  works in this domain is limited, with notable exceptions~\citep{sun2014learning,sedghi2016provable}.

We consider the high dimensional parameter learning problem in a very general mixture model that covers all of the above settings. We assume the parameter vectors to be sparse, and focus on recovering the support of the vectors. 

It is worth mentioning that mixtures of sparse linear regressions and classifiers were also considered in some recent works that focus on a query-based model, i.e., where the covariates can be designed as queries~\citep{yin2019learning,kris2019sampling,mazumdar2020recovery,gandikota2020recovery,gandikota2021support,polyanskii2021learning}. The query based setting is quite different from our unsupervised setting, because in the former one can use the same covariates again and again to get potentially different labels, and thus identify the components. However, as we will see, some tools developed in \citep{gandikota2021support} can still be relevant for support recovery in the current setting where we cannot dictate the covariates.

 An interesting application of learning mixtures with sparse parameters is in high-dimensional clustering problems where cluster centers actually belong to a low-dimensional space. This is similar in spirit with sparse-PCA \citep{johnstone2009consistency}; our objective is to identify a few important input features, so one can easily interpret its meaning. Our techniques can also be seen as a novel method for feature selection that can significantly speed up a learning algorithm. 

Another practical application comes up naturally in recommendation systems where multiple users rate/purchase/evaluate items. User tastes can differ, and that can be modeled by a few unknown parameter vectors. It makes sense for the unknown vectors to be sparse, because most users have an affinity towards a few particular features of items among many possible. Sparse mixtures were motivated with such an application in the query based setting in \citep{gandikota2020recovery,gandikota2021support}.

Note that, support recovery is an effective way to reduce the dimension of the ambient space, and therefore can be considered as a key step towards parameter estimation. We study the support recovery problem in three different canonical mixture models as described above: mixtures of distributions (MD), mixtures of linear regressions (MLR), and mixtures of linear classifiers (MLC). The three models will differ somewhat in analysis as they pose different challenges; however there will be commonalities in the key techniques.  We provide two flavors of results for support recovery namely, 1) \textit{Exact support recovery:} where we recover the supports of all unknown sparse parameters corresponding to all components of the mixture, 2) \textit{Maximal support recovery:} where we recover the maximal supports from the poset of all supports of the  parameter vectors (i.e., supports that are not subsets of any other). Some of the algorithms and results of this paper has appeared in the shorter conference version \cite{DBLP:conf/aistats/PalM22}. In this version, the results have been extended and discussed in significantly more details. To formally specify the problems and state the results we need to define certain quantities.



\subsection{Notations}\label{sec:notation}
We write $[n]$ to denote the set $\{1,2,\dots,n\}$. We will use $\fl{1}_n,\fl{0}_n$ to denote an all one vector and all zero vector of dimension $n$ respectively. We will use $\ca{Q}([n])$ to denote the power set of $[n]$ i.e. $\ca{Q}([n]) = \{\ca{C}\mid \ca{C}\subseteq [n]\}$. The default base for logarithms is 2, unless otherwise specified.

For any vector $\fl{v} \in \bb{R}^n$, we use $\fl{v}_i$ to denote the $i^{\s{th}}$ coordinate of $\fl{v}$ and for any ordered set $\ca{S} \subseteq [n]$, we will use the notation $\fl{v}_{\mid \ca{S}} \in \bb{R}^{|\ca{S}|}$ to denote the vector $\fl{v}$ restricted to the indices in $\ca{S}$. Furthermore, we will use $\s{supp}(\fl{v}) \triangleq \{i \in [n]:\fl{v}_i \neq 0\}$ 
to denote the support of $\fl{v}$ and $\left|\left|\fl{v}\right|\right|_0 \triangleq \left|\s{supp}(\fl{v})\right|$ to denote the size of the support. Let $\s{sign}:\bb{R}\rightarrow \{-1,+1\}$ be a function that returns the sign of a real number i.e. for any input $x \in \bb{R}$, 
\begin{align*}
    \s{sign}(x) = 
\begin{cases}
 1 \quad \text{if $x \ge 0$}\\
 -1 \quad \text{if $x < 0$}
\end{cases}.
\end{align*}
Consider a multi-set of $n$-dimensional vectors $\ca{U} \equiv \{\fl{u}^{(1)},\fl{u}^{(2)},\dots,\fl{u}^{(\ell)}\}$. We will write $\ca{S}_{\ca{U}}(i)\triangleq \{\fl{u}\in \ca{U}:\fl{u}_i \neq 0\}$ to denote the multi-set of vectors in $\ca{U}$ that has a non-zero entry at the $i^{\s{th}}$ index. Furthermore, for an ordered set $\ca{C}\subseteq [n]$ and vector $\fl{a}\in \{0,1\}^{\left|\ca{C}\right|}$, we will also write $\s{occ}_{\ca{U}}(\ca{C},\fl{a}) \triangleq \sum_{\fl{u}\in \ca{U}} \mathbf{1}[\fl{u}_{\mid \ca{C}}=\fl{a}]$ to denote the number of vectors in $\ca{U}$ that equal $\fl{a}$ when restricted to the indices in $\ca{C}$. 
For a matrix $\fl{M}\in \bb{R}^{m \times n}$, we will use $\fl{M}_i$ to denote the $i^{\s{th}}$ column of $\fl{M}$. Let ${\fl{A}}_{\ca{U}} \in \{0,1\}^{n \times \ell}$ denote the support matrix of $\ca{U}$ where each column vector $\fl{A}_i \in \{0,1\}^n$ represents the support of the vector $\fl{u}^{(i)} \in \ca{U}$. For ease of notation, we will omit the subscript $\ca{U}$ when the set of vectors is clear from the context. 

We write $\ca{N}(\mu,\sigma^2)$ to denote a Gaussian distribution with mean $\mu$ and variance $\sigma^2$. We will denote the cumulative distribution function of a random variable $Z$ by $\phi:\bb{R}\rightarrow [0,1]$ i.e. $\phi(a)=\int_{-\infty}^{a}p(z)dz$ where $p(\cdot)$ is the density function of $Z$. Also, we will denote $\s{erf}:\bb{R}\rightarrow \bb{R}$ to be the error function defined by $\s{erf}(z)=\frac{2}{\sqrt{\pi}}\int_{0}^{z}\exp(-t^2)dt$. Since the error function $\s{erf}$ is bijective, we define $\s{erf}^{-1}(\cdot)$ to be the inverse of the $\s{erf}(\cdot)$ function. Finally, for a fixed set $\ca{B}$ we will write $X\sim_{\s{Unif}} \ca{B}$ to denote a random variable $X$ that is uniformly sampled from the elements in $\ca{B}$.


\subsection{Formal Problem Statements} 
\label{sec:formal}
 Let $\ca{V}$ be a multi-set of $\ell$ unknown $k$-sparse vectors $\fl{v}^{(1)},\fl{v}^{(2)}, \dots,\fl{v}^{(\ell)} \in \bb{R}^n$ such that $\left|\left|\fl{v}^{(i)}\right|\right|_0 \le k$ for all $i \in [\ell]$. We consider the following problems described below:

\paragraph{Mixtures of Distributions with Sparse Latent Parameters (MD):}

Consider a class of distributions $\ca{P} \equiv \{\fl{P}(\theta)\}_{\theta \in \Theta}$ parameterized by some $\theta \in \Theta$ where $\Theta \subseteq \bb{R}$. We assume that all distributions in  $\ca{P}$ satisfy the following property:
$\bb{E}_{x\sim \fl{P}(\theta)} x^{\ell}$ can be written as a polynomial in $\theta$ of degree exactly $\ell$, From Table 2 in \citep{belkin2010polynomial}, we know that many well-known distributions satisfy this property (further discussion later). 
A sample $\fl{x}\sim \ca{P}_d$ is generated as follows:
\begin{align*}
    t \sim_{\s{Unif}} [\ell] \; \text{and} \; \fl{x}_i \mid t \sim \fl{P}(\fl{v}^{(t)}_i) \text{ independently }\forall i\in [n]. 
\end{align*}
In other words, $\fl{x}$ is generated according to a uniform mixture of distributions each having a sparse unknown parameter vector. 
Consider  $\fl{x}^{(1)},\fl{x}^{(2)},\dots,\fl{x}^{(m)}\in \bb{R}^n$, $m$ i.i.d. copies of $\fl{x}$, that we can use to recover $\ca{V}$.

Here are some examples of this setting:
\begin{enumerate}
\item $\fl{P}(\theta)$ can be a Gaussian distribution with mean $\theta$ and known variance $\sigma^2$. This setting corresponds to a mixture of high-dimensional isotropic Gaussian distributions with sparse means.

\item $\fl{P}(\theta)$ can be a uniform distribution with range $[\theta,b]$ for a fixed and known $b$.

\item $\fl{P}(\theta)$ can be a Poisson distribution with mean $\theta$.

\end{enumerate}  

\paragraph{Mixtures of Sparse Linear Regressions (MLR).}

Consider $m$ samples
$$(\fl{x}^{(1)},y^{(1)}),(\fl{x}^{(2)},y^{(2)}),\dots,(\fl{x}^{(m)},y^{(m)}) \in \bb{R}^n \times \bb{R}$$ which are generated independently according to a distribution $\ca{P}_{r}$ defined as follows: for $(\fl{x},y)\sim \ca{P}_{r}$, we have 
\begin{align*}
    &\fl{x}_i \sim \ca{N}(0,1) \text{ independently for all } i\in [n] \\
    &\fl{v} \sim_{\s{Unif}} \ca{V} \text{ and } y \mid \fl{x},\fl{v} \sim \ca{N}(\langle\fl{v},\fl{x}\rangle,\sigma^2).
\end{align*}
In other words, each entry of $\fl{x}$ is sampled independently from $\ca{N}(0,1)$ and for a fixed $\fl{x}$, the conditional distribution of $y$ given $\fl{x}$ is a Gaussian with mean $\langle\fl{v},\fl{x}\rangle$ and known variance $\sigma^2$ where $\fl{v}$ is uniformly sampled from the multi-set $\ca{V}$.

\paragraph{Mixtures of Sparse Linear Classifiers (MLC).}
 
Consider $m$ samples $$(\fl{x}^{(1)},y^{(1)}),(\fl{x}^{(2)},y^{(2)}),\dots,(\fl{x}^{(m)},y^{(m)}) \in \bb{R}^n \times \{-1,+1\}$$ which are generated independently according to a distribution $\ca{P}_{c}$ defined as follows: for $(\fl{x},y)\sim \ca{P}_{c}$, we have 
\begin{align*}
    &\fl{x}_i \sim \ca{N}(0,1) \text{ independently for all } i\in [n] \\
    & \fl{v} \sim_{\s{Unif}} \ca{V} \text{ and } z \sim \ca{N}(0,\sigma^2) \text{ and } y  = \s{sign}(\langle\fl{v},\fl{x}\rangle +z).
\end{align*}
In other words, each entry of $\fl{x}$ is sampled independently from $\ca{N}(0,1)$ and for a fixed $\fl{x}$, the conditional distribution of $y$ given $\fl{x}$ is $+1$ if $\langle \fl{v},\fl{x} \rangle \ge -z$ and $-1$ otherwise; here, $\fl{v}$ is uniformly sampled from the multi-set of unknown vectors $\ca{V}$ and $z$ denotes zero mean Gaussian noise with variance $\sigma^2$. 

 Our goal in all the three problems described above is to recover the support of unknown vectors $\fl{v}^{(1)},\fl{v}^{(2)}, \dots,\fl{v}^{(\ell)} \in \ca{V}$
 with minimum number of samples $m$. More formally, we look at two distinct notions of support recovery:
 \begin{defn}[Exact Support Recovery]\label{def:exact}
 We will say that an algorithm achieves Exact Support Recovery in the MLC/MLR/MD setting if it can recover the supports of all the unknown vectors in $\ca{V}$ exactly.
 \end{defn}
 
 Note that, $\{\s{supp}(\fl{v}): \fl{v} \in\ca{V}\}$ is a poset according to containment or set-inclusion ($\subseteq$). A maximal element of this poset is one that is not subset of any other element. When the supports are all different, each of them is maximal. 
 
 Let $\s{Maximal}(\ca{V})$ be the unique set of all maximal elements of the  poset $\{\s{supp}(\fl{v}): \fl{v} \in\ca{V}\}$.

 \begin{defn}[Maximal Support Recovery]\label{def:partial}
 We will say that an algorithm achieves Maximal Support Recovery in the MLR/MLC/MD setting if it can recover $\s{Maximal}(\ca{V})$, i.e.,   all the maximal elements of the poset $\{\s{supp}(\fl{v}): \fl{v} \in\ca{V}\}$. 
 \end{defn}

Note that in Definition \ref{def:partial}, the objective is to recover supports of the largest set of vectors in $\ca{V}$, where no support is included completely in another support; this is easier than exact support recovery (Definition \ref{def:exact}). 

\begin{rmk}\label{rmk:interesting}
If every unknown vector $\fl{v}\in\ca{V}$ had a unique non-zero index $i\in [n]$ i.e. $\fl{v}_i \neq 0$ and $\fl{v}'_i =0$ for all $\fl{v}'\in \ca{V}\setminus \{\fl{v}\}$, then maximal support recovery is equivalent to exact support recovery. This condition, also known as the separability condition, has been commonly used in the literature for example in unique non-negative matrix factorization \citep{arora2016computing,donoho2004does,slawski2013matrix} and approximate parameter recovery in MLC in the query-based setting \citep{gandikota2020recovery}.
\end{rmk}



Note that a trivial approach to the support recovery problem is to first recover the union of support and then apply existing parameter estimation guarantees in the corresponding mixture setting. However, note that this approach crucially requires parameter estimation results for the corresponding family of mixtures which may be unavailable. We have provided a detailed discussion on our results and other relevant work
including the alternate approach outlined above
in Appendix \ref{app:relevant}.

\subsection{Discussion on Our Results and Other Related Works}\label{app:relevant}

Note that all the sample complexity guarantees that we present in this paper scale polylogarithmically with the ambient dimension $n$. Below, we discuss more about each of the individual settings that we study in the unsupervised (mixtures of distributions) and in the supervised case (mixtures of linear regression and linear classifiers).

\paragraph*{Mixtures of Distributions:}


Our technique of learning the supports of the latent parameter vectors in mixture of simple distributions is based on the {\em method of moments}~\citep{hsu2013learning,hardt2015tight}. This method works in general, as long as moments of the distribution of each coordinate can be described as a polynomial in the component parameters.  \cite{belkin2010polynomial} showed (see Table 2 in \citep{belkin2010polynomial}) that most common distributions, including Gaussian, Uniform, Poisson, and Laplace distributions, satisfy this assumption. Our results in this part that include sample complexity guarantees for both exact support recovery (see Theorem \ref{thm:md}) and maximal support recovery (see Theorem \ref{thm:general_md}) are not only applicable to many canonical distributions but also makes progress towards quantifying the sufficient number of moments in the  general problem defined in Sec.~\ref{sec:formal}.


An alternate approach to the support recovery problem is to first recover the {\em union} of supports of the unknown parameters and then apply known parameter estimation guarantees to identify the support of each of the unknown vectors after reducing the dimension of the problem. Note that this approach crucially requires parameter estimation results for the corresponding family of mixtures which may be unavailable. To the best of our knowledge, most constructive sample complexity guarantees for parameter estimation in mixture models without separability assumptions correspond to mixtures of Gaussians \citep{kalai2010efficiently,belkin2010polynomial,moitra2010settling,hardt2015tight,feller2016weak, ho2016convergence, manole2020uniform,heinrich2018strong}. Moreover, most known results correspond to mixtures of Gaussians with two components. The only known results for parameter estimation in mixtures of Gaussians with more than $2$ components is \cite{moitra2010settling} but as we describe later, using the alternate approach with the guarantees in  \cite{moitra2010settling} results in a polynomial dependence on the sparsity. On the contrary, our sample complexity guarantees scales logarithmically with the sparsity or dimension (for constant $\ell$), see Corollary~\ref{coro:md}, which is a significant improvement over the alternate approach.

For other than Gaussian distributions, \cite{belkin2010polynomial, krishnamurthy20a} studied parameter estimation under the same moment-based assumption that we use.
  However, \cite{belkin2010polynomial} use non-constructive arguments from algebraic geometry because of which, their results did not include bounds on the sufficient number of moments for learning the parameters in a mixture model.  \cite{krishnamurthy20a} resolve this question to a certain extent for these aforementioned families of mixture models as they quantify the sufficient number of moments for parameter estimation under the restrictive assumption that the latent parameters lie on an integer lattice. Therefore, our results for these distributions form the first guarantees for support recovery.

\paragraph*{Mixtures of Linear Regression}

For the support recovery problem in the sparse mixtures of linear regressions (MLR) setting, we provide a suite of results under different assumptions. In the first part, we study the exact support recovery problem when the unknown sparse parameters are binary (see Theorem \ref{thm:binary}) and the maximal support recovery problem when 1) the unknown sparse parameters have non-negative values (see Corollary \ref{coro:non-negative}), or 2) the unknown sparse parameters are distributed according to a Gaussian (see Corollary \ref{coro:gaussian}). 
As in the MD setting, an alternate approach for the support recovery problem is to first find the union of support of the unknown parameters and then apply existing parameter estimation guarantees to recovery the support of each of the unknown linear functions. The state of the art guarantees in MLR for parameter estimation is given by \cite{li2018learning} providing a sample complexity guarantee which is linear in the dimension (linear in sparsity when restricted to the union of support). Our results for support recovery are polynomial in sparsity and are therefore worse than the parameter estimation guarantees of \citep{li2018learning} applied to our sparse setting (see Theorem \ref{thm:mlr_general}) when the sparsity is large. On the other hand, the sample complexity guarantees of \citep{li2018learning} scales exponentially with $\ell^2$ and polynomially with the inverse of the failure probability. In contrast, our sample complexity guarantees are polynomial in $\ell$ and logarithmic in the inverse of the failure probability.  

In the second part, we study the support recovery problem in the MLR setting under  mild assumptions on the unknown sparse vectors.  We show a improvement in the sample complexity guarantees (as compared to that of Theorem \ref{thm:mlr_general}) in the dependency on the number of components $\ell$ (see Theorem \ref{thm:mlr_general_new}). In particular, our sample complexity guarantees are polynomial in $\ell$ whereas the guarantees of \citep{li2018learning} has an  exponential dependence on $\ell^2$. Moreover, our sample complexity guarantee has a polynomial dependence on all other parameters such as sparsity and failure probability. In terms of the sparsity parameter, our bound is weaker than Theorem \ref{thm:mlr_general} since the sample complexity dependence of the latter on the sparsity is linear which is optimal.

\paragraph*{Mixtures of Linear Classifiers}

Unlike the MLR and MD setting, mixture of linear classifiers (MLC) is far less studied.  It is understandably more difficult to analyze than MLR since only the sign of the linear function of the covariates is retained. We study the exact support recovery problem in sparse MLC (see Theorem \ref{thm:mlc}) under the setting that all the parameters of the unknown vectors are either nonnegative or they are all nonpositive.
Although this assumption might seem restrictive, note that theoretical work in the MLC setting is extremely limited. To the best of our knowledge, there are only two relevant papers \citep{sun2014learning,sedghi2016provable} that have studied this problem. In  \citep{sun2014learning}, the authors do not make any assumptions on sparsity and provide an algorithm for recovering the subspace in which the parameter vectors corresponding to the unknown linear functions lie. In contrast, support recovery is a different objective and hence is incomparable to the subspace recovery guarantees. The second work \citep{sedghi2016provable} uses tensor decomposition based methods to provide sample complexity guarantees for learning the parameter vectors; but their sample complexity is inversely proportional to the square of the minimum eigenvalue of the matrix comprising the unknown parameter vectors as columns. This is an unwanted dependence as it implies that if the parameter vectors are linearly dependent, then the algorithm will require infinite samples to recover the parameter vectors. On the other hand, our support recovery guarantees do not have any such assumption on the parameters. Moreover, unlike the MD setting, it is not evident in MLC how to recover the union of support of the unknown sparse vectors. Hence the sample complexity obtained by applying the results in \citep{sedghi2016provable} directly will lead to a polynomial dependence on the dimension of the latent space which is undesirable (ideally, we require a  logarithmic dependence on the latent space dimension). Our results exhibit such dependence on the dimension and also does not assume linear independence of the parameter vectors. We believe this to be an important progress towards further understanding of theoretical properties of mixtures where the response is a mixture of nonlinear functions of the covariates.

\paragraph*{Main Technical Contribution} 

As discussed earlier, our unsupervised setting is different from the query-based setting of \cite{gandikota2021support}, where the focus is also support recovery. However, we crucially use a general technique introduced in \cite{gandikota2021support} (see Lemma \ref{thm:prelim1}) for exact support recovery. Namely, support recovery is possible if we can estimate some subset statistics.

But computing estimates of these subset statistics   to invoke the guarantees given in Lemma \ref{thm:prelim1} is a difficult problem.
For the three settings, namely MD/MLR/MLC, we provide  distinct and novel techniques
to compute these quantities.
Our approach to compute the sufficient statistics in MD setting involve a two-step approach with polynomial identities : 1) first, using the method of moments, we compute estimates of the power sum polynomial of degree $p$ in the variables $\{\prod_{i \in \ca{C}} \fl{v}_i^2\}_{\fl{v}\in \ca{V}}$ for all subsets $\ca{C} \subset [n]$ up to a certain size; 2) secondly, we use an elegant connection via Newton's identities to compute estimates on the elementary symmetric polynomial in the variables $\{\prod_{i \in \ca{C}} \fl{v}_i^2\}_{\fl{v}\in \ca{V}}$ which in turn allows us to compute the sufficient statistics. In MLR, for a set $\ca{C}\subseteq [n]$, we again analyze an interesting quantity namely $y^{\left|\ca{C}\right|} \cdot\Big(\prod_{i\in \ca{C}}\fl{x}_i\Big)$ that reveals the sufficient statistic for invoking Lemma \ref{thm:prelim1}. In MLC, our method is quite different and it involves conditioning on the event that certain coordinates of the covariate have large values. If this event is true, then analyzing the response variables  reveals the sufficient statistics for invoking Lemma \ref{thm:prelim1}.

\paragraph{Organization:} The rest of the paper is organized as follows: in Section \ref{sec:prelims}, we provide the necessary preliminary lemmas for support recovery (exact and maximal). In Section \ref{sec:high_level}, we provide our main results on exact support recovery and discuss our core approaches in each of the settings namely MD/MLR/MLC at a high level. For example, see Corollary \ref{coro:md}, Theorem~\ref{thm:mlr_general}, and Theorem~\ref{thm:mlc} for representative results in the three settings respectively. In Section  \ref{sec:dedup}, we have provided additional results on maximal support recovery for MD and MLR settings respectively. In Appendix \ref{sec:detailed_md}, \ref{sec:detailed_mlc} and \ref{sec:detailed_mlr}, we provide the detailed proofs of all our results in the MD, MLC and MLR settings respectively.
 In Appendix \ref{app:prelim}, we provide the missing proofs of lemmas in Section \ref{sec:prelims} and in Appendix \ref{app:lemma1}, we provide the proof of Lemma \ref{thm:prelim1} proved in \citep{gandikota2021support}.
In Appendix \ref{app:technical}, we provide a few technical lemmas that are used in the main proofs.

\section{Preliminaries}\label{sec:prelims}

To derive our support recovery results, we will crucially use the result of Lemma \ref{thm:prelim1} below which has been proved in \citep{gandikota2021support}. Recall the definition of $\s{occ}(\ca{C},\fl{a})$ in Sec.~\ref{sec:notation}. Lemma \ref{thm:prelim1} states that if $\s{occ}(\ca{C},\fl{a})$ is known for all sets $\ca{C}\subseteq [n]$ up to a cardinality of $\log \ell+1$, then it is possible to recover the support of all the unknown vectors in $\ca{V}$. We restate the result according to our terminology. 

\begin{algorithm}[!htbp]
\caption{\textsc{Support Recovery} \label{algo:t-iden-supp-rec}}
\begin{algorithmic}[1]
\REQUIRE $|\Su(C, \f a)|$ for every $C \subset [n]$, $|C| = t, \; t \in\{p, p+1\}$, $p = \lfloor\log \ell \rfloor$,  and every $\fl{a} \in \{0,1\}^{p} \cup \{0,1\}^{p+1}$.
\STATE Set $\text{count} = 1,i = 1 $.
\WHILE{$\text{count}\le \ell$}
    \IF{$|\Su(C, \f a)| = w$, and $|\Su(C\cup\{j\}, (\fl{a}, 1))|\in \{0,w\}$ for all $j \in [n]\setminus C$}
        \STATE Set $\supp{\fl{u}^{i}}|_C =  \fl{a}$\\
        \STATE For every $j \in [n]\setminus C$, set $\supp{\fl{u}^{i}}|_j = b$, where $|\Su(C\cup\{j\}, ({\f a}, b))| = w$.
        \STATE Set $\s{Multiplicity}^i = w$.
        \STATE For all $\fl{t}\in \{0,1\}^{p} \cup \{0,1\}^{p+1},S \subseteq [n]$ such that $|S|\in \{p, p+1\}$, update
\begin{align*}
    \left|\Su(S, \fl{t})\right| \leftarrow \left|\Su(S, \fl{t})\right| 
    - \left|\Su(C, \fl{a})\right|\times\mathbf{1}[ \supp{\fl{u}^{i}}|_S = \fl{t}]
\end{align*}
        \STATE $\text{count} = \text{count} + w$.
        \STATE $i=i + 1$.
    \ENDIF
\ENDWHILE
\STATE Return $\s{Multiplicity}^j$ copies of $\supp{\fl{u}^{j}}$ for all $j < i$.
\end{algorithmic}
\end{algorithm}

\begin{lemma}\label{thm:prelim1}[Corollary 1 in \citep{gandikota2021support}]
Let $\ca{V}$ be  a set of $\ell$ unknown vectors in $\bb{R}^n$. Then, if $\s{occ}(\ca{C},\fl{a})$ is provided as input for all sets $\ca{C}\subset [n], |\ca{C}|\le \log \ell+1$ and for all $\fl{a}\in \{0,1\}^{|\ca{C}|}$, then there exists an algorithm (see Algorithm \ref{algo:t-iden-supp-rec}) that can recover the support of the unknown vectors in $\ca{V}$. 
\end{lemma}

For the sake of completeness, we provided the detailed proof of Lemma \ref{thm:prelim1} in Appendix \ref{app:lemma1}.

{\color{black}
\begin{rmk}
Lemma \ref{thm:prelim1} provides an unconditional guarantee for recovering the support of the unknown vectors in $\ca{V}$. In other words, in the worst case, we only need to know $\s{occ}(\ca{C},\fl{a})$ for all sets of size $|\ca{C}|\le \log \ell+1$. However, \citep{gandikota2021support}[Theorems 1,2 and 4],  significantly relaxed sufficient conditions for recovering the support of $\ca{V}$ under different structural assumptions were also provided. 
As noted in \citep{gandikota2021support}, these additional conditions are  mild and in most cases, if $\s{occ}(\ca{C},\fl{a})$ is known for all sets $\ca{C}\subseteq [n]$ up to a cardinality of $3$, then it is possible to recover the support of all the unknown vectors in $\ca{V}$. 
\end{rmk} 

}

Next, we describe another result, Lemma \ref{lem:prelim1}, proved in \citep{gandikota2021support} that is also going to be useful for us. The main takeaway from Lemma \ref{lem:prelim1}  is that computing $\left|\cup_{i \in \ca{C}}\ca{S}(i)\right|$ (which represents the number of unknown vectors in $\ca{V}$ having non-zero values in at least one entry corresponding to $\ca{C}$) for all sets smaller than a fixed size (say $t$)  is sufficient to compute $\s{occ}(\ca{C},\fl{a})$ for all subsets $\ca{C}\subseteq [n], |\ca{C}|\le t$ and all vectors $\fl{a}\in \{0,1\}^{|\ca{C}|}$. 
In addition, we provide a result in  Lemma \ref{lem:prelim1} where we show that it is also possible to compute $\s{occ}(\ca{C},\fl{a})$ if the quantities $\left|\cap_{i \in \ca{C}}\ca{S}(i)\right|$ (which represents the number of unknown vectors in $\ca{V}$ having non-zero values in all entries corresponding to $\ca{C}$) are provided for all subsets $\ca{C}\subseteq [n]$ satisfying $|\ca{C}|\le t$.

\begin{lemma}[Partially proved in \citep{gandikota2021support}]\label{lem:prelim1}
 Let $\ca{V}$ be  a set of $\ell$ unknown vectors in $\bb{R}^n$. If $|\bigcup_{i \in \ca{C}}\ca{S}(i)|$ is provided as input for all sets $\ca{C}\subset [n], |\ca{C}|\le t$ or alternatively $|\bigcap_{i \in \ca{C}}\ca{S}(i)|$ is provided as input for all sets $\ca{C}\subset [n], |\ca{C}|\le t$, then we can compute $\s{occ}(\ca{C},\fl{a})$ for all sets $\ca{C}\subseteq [n], |\ca{C}|\le t, \fl{a}\in \{0,1\}^{|\ca{C}|}$.
\end{lemma}

\begin{coro}\label{coro:prelim1}
Let $\ca{V}$ be  a set of $\ell$ unknown $k$-sparse vectors in $\bb{R}^n$. Suppose, for each $\ca{C}\subseteq [n], |\ca{C}|\le \log \ell+1$, we can compute $ \left|\cup_{i \in \ca{C}}\ca{S}(i)\right|$  (or alternatively $|\bigcap_{i \in \ca{C}}\ca{S}(i)|$) with probability $1-\gamma$ using $\s{T} \log \gamma^{-1}$ samples where $\s{T}$ is independent of $\gamma$. Then, there exists an algorithm (see Algorithm \ref{algo:intersection}) that can achieve Exact Support Recovery with probability at least $1-\gamma$ using $O(\s{T}\log
 (\gamma^{-1}(n+(\ell k)^{\log \ell+1})))$
samples.
\end{coro}


\begin{algorithm}[!htbp]
\caption{\textsc{Exact Support recovery using access to  estimates of $\left|\cap_{i \in \ca{C}}\ca{S}(i)\right|$ (or alternatively $\left|\cup_{i \in \ca{C}}\ca{S}(i)\right|$) that are correct with high probability} \label{algo:intersection}}
\begin{algorithmic}[1]
\REQUIRE For $\ca{C}\subseteq [n]$, access to estimates of $\left|\cap_{i \in \ca{C}}\ca{S}(i)\right|$ (or alternatively $\left|\cup_{i \in \ca{C}}\ca{S}(i)\right|$) that are correct with high probability.

\STATE For each $i\in [n]$, compute an estimate of $\left|\ca{S}(i)\right|$.

\STATE Compute $\ca{T}=\{i \in [n]\mid \text{estimate}(\ca{S}(i))>0\}$.

\STATE Compute estimates of $\left|\cap_{i \in \ca{C}}\ca{S}(i)\right|$ (or alternatively $\left|\cup_{i \in \ca{C}}\ca{S}(i)\right|$) for all subsets $\ca{C}\subseteq \ca{T}, |\ca{C}|\le \log \ell+1$. 

\STATE Compute $\s{occ}(\ca{C},\fl{a})$ for all subsets $\ca{C}\subseteq \ca{T}, \left|\ca{C}\right|\le \log \ell+1, \fl{a}\in \{0,1\}^{|\ca{C}|}$ using the computed estimates of $\left|\cap_{i \in \ca{C}}\ca{S}(i)\right|$ (or alternatively $\left|\cup_{i \in \ca{C}}\ca{S}(i)\right|$).

\STATE Use Algorithm \ref{algo:t-iden-supp-rec} to recover the support of all unknown vectors in $\ca{V}$.

\end{algorithmic}
\end{algorithm}

 In the next few lemmas, we characterize the set $\s{Maximal}(\ca{V})$ and show some useful properties. We start with the following definition:

\begin{defn}[$t$-good]
  A binary matrix $\fl{A} \in \{0,1\}^{n \times \ell}$ with all distinct columns is called  $t$-good if for every column $\fl{A}_i$, there exists a 
  set $S \subset [n]$ of $t$-indices such that $\fl{A}_i|_S = \fl{1}_t$, and $\fl{A}_j|_S \neq \fl{1}_t$ for all $j \neq i$. A set $U \subset \ca{Q}([n])$ is $t$-good if its $n \times |U|$ incidence matrix is $t$-good.

\end{defn}

Notice that if any set 
is $t$-good then it must be $r$-good for all $r\ge t$. In Lemma \ref{lem:good}, we show that $\s{Maximal}(\ca{V})$ is $(\ell-1)$-good and in Lemma  \ref{lem:partial2}, we provide sufficient conditions for maximal support recovery of the set of unknown vectors $\ca{V}$. 

\begin{lemma}\label{lem:good}
For any set of $\ell$ unknown vectors $\ca{V}$, $\s{Maximal}(\ca{V})$ must be $(\ell-1)$-good.
\end{lemma}
\begin{proof}
Note that, any set of $\s{Maximal}(\ca{V})\subset \ca{Q}([n])$ is not contained in any other. For any two $A,A' \in \s{Maximal}(\ca{V}),$ there exists some $i \in A$ such that $i \notin A'$. Therefore, for a fixed $A\in \s{Maximal}(\ca{V})$, for each $A' \in \s{Maximal}(\ca{V})\setminus \{A\}$, we can have at most $\ell -1$ elements, that are all in $A$, but not all in any other set.
\end{proof}

\begin{lemma}\label{lem:partial}
  If it is known whether $\left|\cap_{i \in \ca{C}}\ca{S}(i)\right|>0$ or not for all sets $\ca{C}\subseteq [n], |\ca{C}|\le s+1$, then there exists an algorithm that achieves maximal support recovery of the set of unknown vectors $\ca{V}$ provided $\s{Maximal}(\ca{V})$ is known to be $s$-good for $s\le \ell-1$ and $\left|\s{Maximal}(\ca{V})\right|\ge 2$.
\end{lemma}

\begin{lemma}\label{lem:partial2}
 If it is known whether $\left|\cap_{i \in \ca{C}}\ca{S}(i)\right|>0$ or not for all sets $\ca{C}\subseteq [n], |\ca{C}|= \ell$, then there exists an algorithm (see Algorithm \ref{algo:partial1}) that achieves maximal support recovery of the set of unknown vectors $\ca{V}$.
\end{lemma}

\begin{coro}\label{coro:partial2}
Let $\ca{V}$ be  a set of $\ell$ unknown $k$-sparse vectors in $\bb{R}^n$. Suppose with probability $1-\gamma$, for each $\ca{C}\subseteq [n], |\ca{C}|\le \ell$, we can compute if $ \left|\cap_{i \in \ca{C}}\ca{S}(i)\right|>0$ correctly with $\s{T} \log \gamma^{-1}$ samples where $\s{T}$ is independent of $\gamma$.  Then, there exists an algorithm (see Algorithm \ref{algo:partial_highprob}) that can achieve maximal support recovery with probability at least $1-\gamma$ using $O(\s{T}\log
 (\gamma^{-1}(n+(\ell k)^{\ell})))$ samples. 
\end{coro}

\begin{algorithm}[!htbp]
\caption{\textsc{Maximal Support recovery using the quantities $\fl{1}[\left|\cap_{i \in \ca{C}}\ca{S}(i)\right|>0]$}
\label{algo:partial1}}
\begin{algorithmic}[1]
\REQUIRE For every $\ca{C}\subseteq [n],|\ca{C}|\le \ell$, the quantities $\fl{1}[\left|\cap_{i \in \ca{C}}\ca{S}(i)\right|>0]$ are provided as input

\STATE Set $\ca{T}=\phi$

\WHILE{There exists a set $\ca{C}\subseteq [n],|\ca{C}|\le \ell-1$ such that $\fl{v}_{\mid \ca{C}}\neq \fl{1}_{|\ca{C}|}$ and $\fl{1}[\left|\cap_{i \in \ca{C}}\ca{S}(i)\right|>0]=1$.}

\STATE Set $\ca{U}=\ca{C}$.  
\FOR{$j\in[n]\setminus \ca{C}$}
\IF{$\fl{1}[\left|\cap_{i \in \ca{C}\cup\{j\}}\ca{S}(i)\right|>0]=1$}
\STATE Set $\ca{U}\leftarrow \ca{U} \cup \{j\}$
\ENDIF
\ENDFOR
\STATE Set $\ca{T}\leftarrow \ca{T}\cup \{\fl{v}\}$ where $\fl{v}\in \{0,1\}^n$ and $\s{supp}(\fl{v})=\ca{U}$.
\ENDWHILE

\STATE Return $\ca{T}$.

\end{algorithmic}
\end{algorithm}

\begin{algorithm}[!htbp]
 \caption{\textsc{Maximal Support recovery using access to  estimates of $\fl{1}[\left|\cap_{i \in \ca{C}}\ca{S}(i)\right|>0]$ that are correct with high probability} \label{algo:partial_highprob}}
\begin{algorithmic}[1]
\REQUIRE For $\ca{C}\subseteq [n]$, access to estimates of $\fl{1}[\left|\cap_{i \in \ca{C}}\ca{S}(i)\right|>0]$ that are correct with high probability.

\STATE For each $i\in [n]$, compute an estimate of $\fl{1}[\left|\cap_{i \in \ca{C}}\ca{S}(i)\right|>0]$.

\STATE Compute $\ca{T}=\{i \in [n]\mid \text{estimate}(\fl{1}[\left|\ca{S}(i)\right|>0]))=\text{True}\}$.

\STATE Compute estimates of $\fl{1}[\left|\cap_{i \in \ca{C}}\ca{S}(i)\right|>0]$  for all subsets $\ca{C}\subseteq \ca{T}, |\ca{C}|\le \ell$.

\STATE Use Algorithm \ref{algo:partial1} to recover the support of all unknown vectors in $\ca{V}$.

\end{algorithmic}
\end{algorithm}

\begin{rmk}
Corollary \ref{coro:partial2} describes the sample complexity for maximal support recovery using Lemma \ref{lem:partial2} which provides the worst-case guarantees as $\s{Maximal}(\ca{V})$ is $(\ell-1)$-good for all sets $\ca{V}$. We can also provide improved guarantees for maximal support recovery provided $\s{Maximal}(\ca{V})$ is known to be $s$-good by using Lemma \ref{lem:partial}. However, for the sake of simplicity of exposition, we have only provided results for maximal support recovery in mixture models using Corollary \ref{coro:partial2}. 
\end{rmk}

All the missing proofs of this section (other than that of Lemma \ref{thm:prelim1}) can be found in Appendix \ref{app:prelim}.




\section{Exact Support Recovery}\label{sec:high_level}

\subsection{Mixtures of Distributions}

In this section, we will present our main results and high level techniques for exact support recovery in the MD setting. The detailed proofs of all results in this section can be found in Section \ref{sec:detailed_md}.
We will start by introducing some additional notations specifically for this setting.
\paragraph{Additional Notations for MD:}  Recall that  $\bb{E}_{x\sim \fl{P}(\theta)} x^{t}$ can be written as a polynomial in $\theta$ of degree $t$. We will write 
\begin{align*}
    q_{t}(\theta) \triangleq \bb{E}_{x\sim \fl{P}(\theta)} x^{t}=\sum_{i \in [t+1]} \beta_{t,i} \theta^{i-1}
\end{align*}
 to denote this aforementioned polynomial where we use $\{\beta_{t,i}\}_{i\in [t+1]}$ to denote its coefficients. 
 For all sets $\ca{A}\subseteq [n]$, we will write $\ca{Q}_i(\ca{A})$ to denote all subsets of $\ca{A}$ of size at most $i$ i.e. $\ca{Q}_i(\ca{A}) = \{\ca{C}\mid \ca{C}\subseteq \ca{A}, |\ca{C}| \le i\}$. 
 Let us define the function $\pi:\ca{Q}([n])\times [n]\rightarrow [n]$ to denote a function that takes as input a set $\ca{C}\subseteq [n]$, an index $r\in \ca{C}$ and returns as output the position of $r$ among all elements in $\ca{C}$ sorted in ascending order. In other words, for a fixed set $\ca{C}$ and all $j\in [|\ca{C}|]$, $\pi(\ca{C},\cdot)$ maps the $j^{\s{th}}$ smallest index in $\ca{C}$ to $j$; for example, if $\ca{C}=\{3,5,9\}$, then $\pi(\ca{C},3)=1,\pi(\ca{C},5)=2$ and $\pi(\ca{C},9)=3$.

We will write $\bb{Z}^{+}$ to denote the set of non-negative integers and $(\bb{Z}^{+})^{n}$ to denote the set of all $n$-dimensional vectors having entries which are non-negative integers. 
For two vectors $\fl{u},\fl{t}\in (\bb{Z}^{+})^{n}$, we will write $\fl{u} \le \fl{t}$  if $\fl{u}_i \le \fl{t}_i$ for all $i\in [n]$; similarly, we will write $\fl{u} < \fl{t}$ if $\fl{u}_i < \fl{t}_i$ for some $i\in [n]$. For any fixed subset $\ca{C}\subseteq [n]$ and vectors $\fl{u},\fl{t}\in (\bb{Z}^{+})^{\left|\ca{C}\right|}$, we will write $\zeta_{\fl{t},\fl{u}}$ to denote the quantity $\zeta_{\fl{t},\fl{u}}\triangleq \prod_{i \in \ca{C}} \beta_{\fl{t}_{\pi(\ca{C},i)},\fl{u}_{\pi(\ca{C},i)}+1}.$ For any $\fl{u},\fl{z} \in (\bb{Z}^{+})^{|\ca{C}|}$ satisfying $\fl{u} < \fl{z}$, we will define a path $\s{M}$ to be a sequence of vectors $\fl{z}_1 > \fl{z}_2>\dots >\fl{z}_m$ such that $\fl{z}_1,\fl{z}_2,\dots ,\fl{z}_m \in (\bb{Z}^{+})^n$, $\fl{z}_1=\fl{z}$ and $\fl{z}_m=\fl{u}$. Let $\ca{M}(\fl{z},\fl{u})$ be the set of all paths starting from $\fl{z}$ and ending at $\fl{u}$. We will also write a path $\s{M}\in\ca{M}(\fl{z},\fl{u})$ uniquely as a set of $m-1$ ordered tuples $\{(\fl{z}_1,\fl{z}_2),(\fl{z}_2,\fl{z}_3),\dots,(\fl{z}_{m-1},\fl{z}_m)\}$
where each tuple consists of adjacent vectors in the path sequence. We will also write $\ca{T}(\s{M})\equiv\{\fl{z}_1,\fl{z}_2,\dots,\fl{z}_m\}$ to denote the set of elements in the path.

 We start with the following assumption which states that every unknown vector is bounded within an euclidean ball and furthermore, the magnitude of every non-zero co-ordinate of all unknown vectors is bounded from below:

\begin{algorithm}[!htbp]
\caption{ \textsc{Estimate($m,B$)} Estimating $\bb{E}X$ for $X \sim \ca{P}$ \label{algo:estimate}}
\begin{algorithmic}[1]
\REQUIRE I.i.d samples $x^{(1)},x^{(2)},\dots,x^{(m)} \sim \ca{P}$ 
\STATE Set $t=m/B$
\FOR{$i=1,2,\dots,B$}
\STATE Set Batch $i$ to be the samples $x^{(j)}$ for $j \in \{it+1,it+2,\dots,(i+1)t\}$.
\STATE Set $S^i_{1}=\sum_{j \in \text{ Batch }i} \frac{x^{(j)}}{t}$
\ENDFOR
 
\STATE Return $\s{median}(\{S^i_1\}_{i=1}^{B})$
\end{algorithmic}
\end{algorithm}

\begin{algorithm}[!htbp]
\caption{\textsc{Recover $|\bigcap_{i \in \ca{C}}\ca{S}(i)|$ in MD setting} \label{algo:md}}
\begin{algorithmic}[1]
\REQUIRE Samples $\fl{x}^{(1)},\dots,\fl{x}^{(m)} \sim \ca{P}_d$. Set $\ca{C}\subseteq [n]$.

\STATE For every $\fl{z}\le 2\ell\fl{1}_{|\ca{C}|}$, compute estimate $\widehat{U}^{\fl{z}}$ of $\bb{E}\prod_{i \in \ca{C}}\fl{x}_i^{\fl{z}_{\pi(\ca{C},i)}}$ using Algorithm \ref{algo:estimate} on the set of samples $\{(\fl{x}_i^{(j)})^{\fl{z}_{\pi(\ca{C},i)}}\}_{j=1}^{m}$.

\STATE For every $\fl{z}\le 2\ell\fl{1}_{|\ca{C}|}$, compute an estimate $\widehat{V}^{\fl{z}}$ of $\sum_{j \in [\ell]}\prod_{i \in \ca{C}}(\fl{v}^{(j)}_i)^{\fl{z}_{\pi(\ca{C},i)}}$ recursively using equation
$
    \ell \widehat{U}^{\fl{z}}-\sum_{\fl{u}<\fl{z}} \zeta_{\fl{z},\fl{u}}\cdot \widehat{V}^{\fl{u}} =\zeta_{\fl{z},\fl{z}} \cdot \widehat{V}^{\fl{z}}. 
$
\STATE For every $t\in [\ell]$, compute an estimate $\widehat{\s{A}}_{\ca{C},t}$ of $\sum_{\substack{\ca{C}'\subseteq [\ell] \\ \left|\ca{C}'\right|=t}} \prod_{\substack{i \in \ca{C}\\ j \in \ca{C}'}} (\fl{v}^{(j)}_i)^{2} $ recursively using Newton's identity  
$
    t\widehat{\s{A}}_{\ca{C},t} = \sum_{p=1}^t (-1)^{p+1} \widehat{\s{A}}_{\ca{C},t-p} \widehat{V}^{2p\fl{1}_{|\ca{C}|}}.
$
\STATE Return $\max_{t\in [\ell]} t. \fl{1}[\widehat{\s{A}}_{\ca{C},t}>0]$.

\end{algorithmic}
\end{algorithm}

\begin{assumption}\label{assum:weak2}
We will assume that all unknown vectors in the set $\ca{V}$ are bounded within a ball of known radius $R$ i.e. $\left|\left|\fl{v}^{(i)}\right|\right|_2 \le R$ for all $i\in [\ell]$. Furthermore, the  magnitude of all non-zero entries of all unknown vectors in $\ca{V}$ is bounded from below by $\delta$ i.e. $\min_{\fl{v}\in \ca{V}}\min_{i:\fl{v}_i \neq 0}\left|\fl{v}_i\right|\ge \delta$.
\end{assumption}

Now, we show our main lemma in this setting where we characterize the sufficient number of samples to compute $\left|\bigcap_{i \in \ca{C}}\ca{S}(i)\right|$ for each set $\ca{C}\subseteq [n]$ with high probability in terms of the coefficients of the polynomials $\{q_t(\theta)\}_t$:

\begin{lemma}\label{lem:md_crucial}
Suppose Assumption \ref{assum:weak2} is true. Let  
\begin{align*}
  & \Phi \triangleq \frac{\delta^{2\ell|\ca{C}|}}{2\Big(3\max(\ell R^{2\ell|\ca{C}|},2^{\ell}R^{\ell+|\ca{C}|})\Big)^{(\ell-1)}\ell!}\times 
  \Bigg(\max_{ \fl{z}\le 2\ell\fl{1}_{|\ca{C}|}} \frac{\ell }{\zeta_{\fl{z,\fl{z}}}}+\sum_{\fl{u}<\fl{z}} \sum_{\s{M}\in \ca{M}(\fl{z},\fl{u}) }\frac{\ell  \prod_{(\fl{r},\fl{s})\in \s{M}}\zeta_{\fl{r},\fl{s}}}{\prod_{\fl{r}\in \ca{T}(\s{M})}\zeta_{\fl{r},\fl{r}}}\Bigg)^{-1}\\
  &g_{\ell,\ca{V}} \triangleq \frac{\max_{\fl{z}\le 2\ell\fl{1}_{|\ca{C}|}}\bb{E}\prod_{i \in \ca{C}}\fl{x}_i^{2\fl{z}_{\pi(\ca{C},i)}}}{\Phi^{ 2}}
\end{align*}
where $g_{\ell,\ca{V}}$ is a constant that is independent of $k$ and $n$ but depends on $\ell$.
  There exists an algorithm (see Algorithm \ref{algo:md}) that can compute $\left|\bigcap_{i \in \ca{C}}\ca{S}(i)\right|$ exactly for each set $\ca{C}\subseteq [n]$ with probability at least $1-\gamma$ using $O\Big(\log (\gamma^{-1}(2\ell)^{|\ca{C}|})g_{\ell,\ca{V}}\Big)$ samples generated according to $\ca{P}_d$.
\end{lemma}

In order to prove Lemma \ref{lem:md_crucial}, we first show that (see Lemma \ref{lem:poly}) for each fixed ordered set $\ca{C}\subseteq [n]$ and each vector $\fl{t}\in (\bb{Z}^{+})^{\left|\ca{C}\right|}$, we must have 
  \begin{align}\label{eq:md}
      \bb{E} \prod_{i \in \ca{C}}\fl{x}_i^{\fl{t}_{\pi(\ca{C},i)}} = \frac{1}{\ell}\sum_{\fl{u} \le \fl{t}}\zeta_{\fl{t},\fl{u}} \cdot \Big(\sum_{j \in [\ell]}\prod_{i \in \ca{C}}  (\fl{v}^{(j)}_i)^{\fl{u}_{\pi(\ca{C},i)}} \Big).
 \end{align}  
Note that each summand in equation \ref{eq:md} is a product of the powers of the co-ordinates of the same unknown vector. In Lemma \ref{lem:recurse1}, we show that for each set $\ca{C}\subseteq [n]$ and any vector $\fl{t}\in (\bb{Z}^{+})^{\left|\ca{C}\right|}$, we can compute $\sum_{j \in [\ell]}\prod_{i \in \ca{C}} (\fl{v}^{(j)}_i)^{\fl{t}_{\pi(\ca{C},i)}}$ via a recursive procedure provided for all $\fl{u}\in (\bb{Z}^{+})^{\left|\ca{C}\right|}$ satisfying $\fl{u}\le\fl{t}$, the quantity $\bb{E}\prod_{i \in \ca{C}} \fl{x}_i^{\fl{u}_{\pi(\ca{C},i)}}$ is pre-computed. This implies that we can compute $\sum_{j \in [\ell]}\prod_{i \in \ca{C}} (\fl{v}^{(j)}_i)^{2p}$ for all $p \in [\ell]$ from the quantities $\bb{E}\prod_{i \in \ca{C}} \fl{x}_i^{\fl{u}_{\pi(\ca{C},i)}}$ for all $\fl{u}\le 2p\fl{1}_{|\ca{C}|}$. It is easy to recognize $\sum_{\fl{v}\in \ca{V}} \Big(\prod_{i\in \ca{C}}  \fl{v}_i^{2}\Big)^{p}$ as the power sum polynomial of degree $p$ in the variables $\{\prod_{i \in \ca{C}} \fl{v}_i^2\}_{\fl{v}\in \ca{V}}$. Now, let us define the quantity $\s{A}_{\ca{C},t}$ for a fixed ordered set $\ca{C}$ and parameter $t\in [\ell]$ as follows:
\begin{align*}
 \s{A}_{\ca{C},t} \triangleq \sum_{\substack{\ca{C}'\subseteq [\ell] \\ \left|\ca{C}'\right|=t}} \prod_{\substack{i \in \ca{C}\\ j \in \ca{C}'}} (\fl{v}^{(j)}_i)^{2}   
\end{align*}
Notice that $\s{A}_{\ca{C},t} > 0$ if and only if there exists a subset $\ca{C}'\subseteq [\ell],|\ca{C}'|=t$ such that $\fl{v}^{(j)}_i \neq 0$ for all $i\in \ca{C},j \in \ca{C}'$. Hence, the maximum value of $t$ such that $\s{A}_{\ca{C},t} > 0$ is the number of unknown vectors in $\ca{V}$ having non-zero value in all the indices in $\ca{C}$. In other words, we have that 
\begin{align*}
    \left|\bigcap_{i \in \ca{C}}\ca{S}(i)\right| = \max_{t\in [\ell]} t\cdot\fl{1}[\s{A}_{\ca{C},t} > 0].  
\end{align*}
Notice that $\s{A}_{\ca{C},t}$ is the elementary symmetric polynomial of degree $t$ in the variables $\{\prod_{i \in \ca{C}} \fl{v}_i^2\}_{\fl{v}\in \ca{V}}$. We can use Newton's identities to state that for all $t\in [\ell]$, 
\begin{align*}
    t\s{A}_{\ca{C},t} = \sum_{p=1}^t (-1)^{p+1} \s{A}_{\ca{C},t-p} \Big(\sum_{\fl{v}\in \ca{V}} \Big(\prod_{i\in \ca{C}}  \fl{v}_i^{2}\Big)^{p}\Big)
\end{align*}
using which, we can recursively compute $\s{A}_{\ca{C},t}$ for all $t\in [\ell]$ ($\s{A}_{\ca{C},0}=1$) and hence $\left|\bigcap_{i \in \ca{C}}\ca{S}(i)\right|$ if we were given $\sum_{\fl{v}\in \ca{V}} \Big(\prod_{i\in \ca{C}}  \fl{v}_i^{2}\Big)^{p} $ as input for all $p\in [\ell]$ (see Lemma \ref{lem:recurse2}).
Lemma \ref{lem:md_crucial} follows from making these set of computations robust. We next show Theorem \ref{thm:md} which follows from applying Lemma \ref{lem:md_crucial} and Corollary \ref{coro:prelim1}.

\begin{thm}\label{thm:md}
Let $\mathcal{V}$ be a set of $\ell$ unknown vectors in $\bb{R}^n$ satisfying Assumption \ref{assum:weak2}. Let $\ca{F}_m = \ca{Q}_1([n]) \cup \ca{Q}_m(\cup_{\fl{v}\in \ca{V}}\s{supp}(\fl{v}))$ and 
\begin{align*}
  &\Phi_{m} \triangleq \frac{\delta^{2\ell m}}{2\Big(3\ell\max(R^{2\ell m},2^{\ell}R^{\ell+m})\Big)^{(\ell-1)}\ell!} \times \Bigg(\max_{ \fl{z}\le 2\ell\fl{1}_{m}} \frac{\ell }{\zeta_{\fl{z,\fl{z}}}}+\sum_{\fl{u}<\fl{z}} \sum_{\s{M}\in \ca{M}(\fl{z},\fl{u}) }\frac{\ell \prod_{(\fl{r},\fl{s})\in \s{M}}\zeta_{\fl{r},\fl{s}}}{\prod_{\fl{r}\in \ca{T}(\s{M})}\zeta_{\fl{r},\fl{r}}}\Bigg)^{-1} \\
  &f_{\ell,\ca{V}}\triangleq \max_{\substack{\fl{z}\le 2\ell\fl{1}_{\log \ell+1} \\ \ca{C}\in \ca{F}_{\log \ell+1}}}\frac{\bb{E}\prod_{i \in \ca{C}}\fl{x}_i^{2\fl{z}_{\pi(\ca{C},i)}}}{\Phi_{\log \ell+1}^{ 2}}.
\end{align*}
Here $f_{\ell,\ca{V}}$ is a constant that is independent of $k$ and $n$ but depends on $\ell$. Then, there exists an algorithm (see Algorithm \ref{algo:md} and \ref{algo:intersection}) that achieves Exact Support Recovery with probability at least $1-\gamma$ using $O\Big(\log (\gamma^{-1}(2\ell)^{\log \ell+1}(n+(\ell k)^{\log \ell+1}))f_{\ell,\ca{V}}\Big)$ samples generated according to $\ca{P}_d$. 
\end{thm}

{\color{black}
\begin{rmk}
We can relax Assumption \ref{assum:weak2} in Theorem \ref{thm:md} without much further work. For our proofs to work out verbatim, it is sufficient to just have the following condition be true: given the latent variable $t$ denoting the mixture component, coordinates of the random vector $\fl{x}\sim \ca{P}_d$ must be $(\log \ell+1)$-wise independent (any $\log \ell+1$ co-ordinates are independent). However, for the sake of simplicity, we have provided the setting where all co-ordinates of $\fl{x}\mid t$ are independent.
\end{rmk}
\textbf{Example:} \textit{Consider the setting when we obtain $m$ i.i.d samples $\fl{x}^{(1)},\fl{x}^{(2)},\dots,\fl{x}^{(m)}\in \bb{R}^n$ from a high dimensional Gaussian mixture $\ca{D}=\frac{1}{2}\ca{N}(\f{\mu}^{(1)},\sigma^2\fl{I})+\frac{1}{2}\ca{N}(\f{\mu}^{(2)},\sigma^2\fl{I})$ with two components where $\f{\mu}^{(1)},\f{\mu}^{(2)}\in \bb{R}^n$ satisfying $\|\f{\mu}^{(1)}\|_0,\|\f{\mu}^{(2)}\|_0 \le k$ are unknown and $\sigma>0$ is known. Our goal is to recover the support of $\f{\mu}^{(1)},\f{\mu}^{(2)}$ while minimizing the number of samples $m$. For $\fl{x}\sim \ca{D}$, for all $i\in [n]$, we have that $\bb{E}\fl{x}_i^2 = \sigma^2+((\f{\mu}_i^{(1)})^2+(\f{\mu}_i^{(2)})^2)/2$; for all $i,j \in [n],i \neq j$, we have
\begin{align*}
    &\bb{E}\fl{x}_i^2\fl{x}_j^2 = \sigma^2(\bb{E}\fl{x}_i^2+\bb{E}\fl{x}_j^2)-\sigma^4\\
    &+\Big(\frac{(\f{\mu}_i^{(1)})^2(\f{\mu}_j^{(1)})^2 
    +(\f{\mu}_i^{(2)})^2(\f{\mu}_j^{(2)})^2}{2}\Big)
\end{align*}
Hence, in the first step, for all $i\in [n]$, with probability $1-\gamma$ we compute an estimate $u_i$ of $\bb{E}\fl{x}_i^2$ (using Lemma \ref{lem:medianofmeans}) satisfying $\left|u_i-\bb{E}\fl{x}_i^2\right|\le \delta^4/(64\sigma^2)$ using $O(\delta^{-8}\sigma^4 \max_i(\sigma^4, (\f{\mu}_i^{(1)})^4, (\f{\mu}_i^{(2)})^4)\log (n\gamma^{-1}))$ samples. With this, we can infer the union of support correctly to be $\ca{S}\equiv \{i \in [n] \mid u_i-\sigma^2 \ge \delta^2/4\}$. This is because for any index $i$ in the union of support, we must have $\bb{E}\fl{x}_i^2 \ge \sigma^2+\delta^2/2$ while for any index $i$ not in the union, we have $\bb{E}\fl{x}_i^2=\sigma^2$. Next, in the second step, for all $i,j \in \ca{S}; i\neq j$, we compute an estimate $u_{ij}'$ of $ \bb{E}\fl{x}_i^2\fl{x}_j^2$ satisfying $|u_{ij}'-\bb{E}\fl{x}_i^2\fl{x}_j^2| \le \delta^4/16$ using $O(\delta^{-8} \max_{i,j}(\sigma, \f{\mu}_i^{(1)}, \f{\mu}_j^{(1)}, \f{\mu}_i^{(2)}, \f{\mu}_j^{(2)})^8\log (n\gamma^{-1}))$ samples with probability at least $1-\gamma$ (see Lemma \ref{lem:medianofmeans}). In that case, if $i,j$ belongs to the support of the same vector, then we will have $|u_{ij}'-\sigma^2(u_i+u_j)+\sigma^4| \ge 13\delta^4/32$ while otherwise, we must have $|u_{ij}'-\sigma^2(u_i+u_j)+\sigma^4| \le 3\delta^4/32$. Hence, $\ca{T}=\{(i,j)\in \ca{S}, i \neq j\mid |u_{ij}'-\sigma^2(u_i+u_j)+\sigma^4| \ge 13\delta^4/32\}$. If there does not exist $i,j \in \ca{S}, i \neq j$ such that $(i,j) \not \in \ca{T}$, then we return $\s{supp}(\f{\mu}^{(1)})=\s{supp}(\f{\mu}^{(2)})=\ca{S}$ implying that both supports are same. On the other hand, if there exists $i,j \in \ca{S}, i \neq j$ such that $(i,j) \not \in \ca{T}$ then $i$ belongs to the support of one vector while $j$ belongs to the support of the other vector (both supports are not same). Let the support of one vector will be $\{s \in \ca{S}, s\neq i \mid (i,s) \in \ca{T}\}$ and the support of the other vector is $\{s \in \ca{S}, s\neq j \mid (j,s) \in \ca{T}\}$. Therefore, the sufficient sample complexity for recovering the support is $m=O(\delta^{-8} \max_{i,j}(\sigma, \f{\mu}_i^{(1)}, \f{\mu}_j^{(1)}, \f{\mu}_i^{(2)}, \f{\mu}_j^{(2)})^8\log (n\gamma^{-1}))$. Note that in this example,  the algorithm is slightly different from the one presented in Algorithm \ref{algo:md}; in, fact the algorithm follows that of maximal support recovery (see Section \ref{sec:dedup}) which is equivalent to exact support recovery for $\ell=2$ (see Remark \ref{rmk:interesting}).}  
}

Now, we provide a corollary of Theorem \ref{thm:md} specifically for mean-estimation in a mixture of distributions with constant number of components i.e. $\ell=O(1)$ . In particular, consider the setting where 
\begin{align*}
    &t \sim_{\s{Unif}} [\ell] \; \text{and} \; \fl{x}_i \mid t \sim \fl{P}(\fl{v}^{(t)}_i) \text{ independently }\forall i\in [n] \\
    &\bb{E}_{\fl{x}\sim \ca{P}_d}[ \fl{x}_i \mid t=j] = \fl{v}^{(j)}_i
\end{align*}
i.e. the mean of the $i^{\s{th}}$ co-ordinate of the random vector $\fl{x}$ distributed according to $\ca{P}_d$ is $\fl{v}^{(j)}_i$.

\begin{coro}\label{coro:md}
Consider the mean estimation problem described above.
Let $\mathcal{V}$ be a set of $\ell=O(1)$ unknown vectors in $\bb{R}^n$ satisfying Assumption \ref{assum:weak2} and   $f_{\ell,\ca{V}}$ be as defined in Theorem \ref{thm:md}.
Then, there exists an algorithm (see Algorithm \ref{algo:md} and \ref{algo:intersection}) that with probability at least $1-\gamma$, achieves Exact Support Recovery using $O\Big(\log (n\gamma^{-1})\s{poly}(\delta R^{-1})f_{\ell,\ca{V}}\Big)$ samples generated according to $\ca{P}_d$. 
\end{coro}

We can compare the sample complexity presented in Corollary \ref{coro:md} with the alternate approach for support recovery namely the two stage process of recovering the union of support followed by parameter estimation restricted to the union of support. As discussed in Section \ref{sec:intro}, most known results (other than \citep{moitra2010settling}) for parameter estimation in Gaussian mixtures without separability assumptions hold for two mixtures and are therefore not applicable for $\ell>2$.
For general value of $\ell$, the only known sample complexity guarantees for parameter estimation in mixture of Gaussians is provided in \citep{moitra2010settling}. 

Note that computing the union of support is not difficult in the MD setting. In particular, in Lemma \ref{lem:md_crucial}, the guarantees include the sample complexity of testing whether a particular index belongs to the union of support; this can be used to compute the union of support itself after taking a union bound over all indices leading to a multiplicative $\log n$ factor. 

However, for one dimensional Gaussian mixture models (1D GMM),  the parameter estimation guarantees in \citep{moitra2010settling} (See Corollary 5) are polynomial in the inverse of the failure probability. Since parameter estimation in 1D GMM is used as a framework for solving the high dimensional problem, it can be extracted that the sample complexity in $n$ dimensions must be polynomial in $n$ with degree at least $1$ to achieve a per coordinate error (error in $\ell_{\infty}$ norm). If restricted to the union of support of the unknown vectors in $\ca{V}$, then using the guarantees in \citep{moitra2010settling} directly will lead to a polynomial dependence on $\ell k$. In essence, the sample complexity of the alternate approach has a logarithmic dependence on the latent space dimension and a polynomial dependence on sparsity $k$ (for constant $\ell$).
Note that our sample complexity only has a logarithmic dependence on the dimension $n$ (and is independent of $k$ for constant $\ell$) and is therefore essentially \textit{dimension-free}. 


For other distributions, to the best of our knowledge, the only known parameter estimation results that exist in literature are \citep{belkin2010polynomial,krishnamurthy20a}. In both of these works, the authors use the same assumption that $\bb{E}_{x\sim \fl{P}(\theta)} x^{\ell}$ can be written as a polynomial in $\theta$ of degree exactly $\ell$. While the guarantees in \citep{belkin2010polynomial} are non-constructive, the results in \citep{krishnamurthy20a} need the restrictive assumption that the means must be multiple of some $\epsilon>0$ and moreover, they have an exponential dependence on the noise variance and  $\epsilon^{-1}$. Our results do not have these limitations and are therefore widely applicable.


\subsection{Mixtures of Linear Classifiers}

\begin{algorithm}[!htbp]
\caption{\textsc{Recover $|\bigcup_{i \in \ca{C}}\ca{S}(i)|$ in MLC setting} \label{algo:mlc}}
\begin{algorithmic}[1]
\REQUIRE Samples $(\fl{x}^{(1)},y^{(1)}),\dots,(\fl{x}^{(m)},y^{(m)}) \sim \ca{P}_c$. Set $\ca{C}\subseteq [n]$. Parameter $a>0$.

\STATE Find the subset of samples $\ca{T}=\{(\fl{x}^{(i)},y^{(i)}) \mid \fl{x}^{(i)}_j > a \text{ for all } i\in [m]\}$.

\STATE Compute an estimate $\widehat{P}$ of $\Pr(y=1\mid \ca{E}_{\ca{C}})$ as 
$
    \widehat{P}=\left|\ca{T}\right|^{-1}\sum_{(\fl{x},y)\in \ca{T}}\fl{1}\left[y=1\right].
$

\STATE Find $t\in [\ell]$ such that 

\begin{align*}
    \frac{1}{2}\Big(1+\frac{t}{\ell} \Big)-\frac{t}{4\ell^2} \le \widehat{P} \le  \frac{1}{2}\Big(1+\frac{t}{\ell} \Big) 
\end{align*}

\STATE Return $t$

\end{algorithmic}
\end{algorithm}

\noindent In this section, we will present our main results and high level techniques for exact support recovery in the MLC setting. The detailed proofs of all results in this section can be found in Section \ref{sec:detailed_mlc}.
 We solve the sparse recovery problem when the observed samples are generated according to $\ca{P}_c$ under the following assumption which states that the unknown vectors in $\ca{V}$ either all have non-negative entries or they all have non-positive entries.

\begin{assumption}\label{assum:pos}
The non-zero entries of unknown vectors in $\ca{V}$ are all either positive ($\fl{v}_i \ge 0$ for all $i\in [n], \fl{v}\in \ca{V}$) or they are all negative ($\fl{v}_i \le 0$ for all $i\in [n], \fl{v}\in \ca{V}$).
\end{assumption}

Although Assumption \ref{assum:pos} looks restrictive, it can often be made in practice. As an example, in the recommendation system application motivated in the introduction (Section \ref{sec:intro}),
the affinity of the users towards the different item features can be modeled by non-negative values; such a modeling assumption is similar to the motivation presented in the literature for non-negative matrix factorization \citep{ding2008convex}.

Next, if Assumption \ref{assum:pos} is satisfied, we show the sample complexity of computing $\left|\bigcup_{i\in \ca{C}}\ca{S}(i)\right|$ for each set $\ca{C}\subseteq [n]$.

\begin{lemma}\label{lem:mlc_crucial}
Suppose Assumptions \ref{assum:weak2} and \ref{assum:pos} are true. Let $a=\frac{\sqrt{2(R^2+\sigma^2)}}{\delta}\s{erf}^{-1}\Big(1-\frac{1}{2\ell}\Big)$. There exists an algorithm (see Algorithm \ref{algo:mlc} with parameter $a>0$) that can compute $\left|\bigcup_{i \in \ca{C}}\ca{S}(i)\right|$ for each set $\ca{C}\subseteq [n]$ with probability at least $1-\gamma$ using
 $O\Big( (1-\phi(a))^{-\left|\ca{C}\right|}\ell^2\log \gamma^{-1}\Big)$ i.i.d samples from $\ca{P}_c$. 
\end{lemma}

Let us present a high level proof of Lemma \ref{lem:mlc_crucial}. Without loss of generality, let us assume that all unknown vectors in $\ca{V}$ have positive non-zero entries. For a fixed set $\ca{C}\subseteq [n]$, suppose we condition on the event $\ca{E}_{\ca{C}}$ which is true when for all $j\in \ca{C}$, $\fl{x}_j>a$ for some suitably chosen $a>0$. Furthermore, let $\ca{E}_{\fl{v}}$ be the event that the particular vector $\fl{v}\in \ca{V}$ is used to generate the sample $(\fl{x},y)$. Notice that if $\fl{v}_i=0$ for all $i\in \ca{C}$, then conditioning on the event $\ca{E}_{\ca{C}}$ does not change the distribution of the response $y\mid \ca{E}_{\fl{v}}$; hence the probability of $y=1$ is exactly $1/2$ in this case. On the other hand, if $\fl{v}_i \neq 0$ for some $i\in \ca{C}$, then conditioning on the event $\ca{E}_{\ca{C}}$ does change the distribution of the response $y\mid \ca{E}_{\fl{v}}$. In particular, if $\fl{v}_i \neq 0$, note that $\langle \fl{v}_{\mid \ca{C}}, \fl{x}_{\mid \ca{C}}\rangle \ge a\delta $ and therefore $\Pr(y=1\mid \ca{E}_{\ca{C}},\ca{E}_{\fl{v}})$ must be larger than $1/2$ and is an increasing function of $a$. Of course, if $a$ is chosen to $+\infty$, then $\Pr(y=1\mid \ca{E}_{\ca{C}},\ca{E}_{\fl{v}})=1$ and therefore $2\Pr(y=1\mid \ca{E}_{\ca{C}})= 1+\ell^{-1}\left|\cup_{i\in \ca{C}}\ca{S}(i)\right|$. Thus, if $a=+\infty$, we can use the fact that $\left|\cup_{i\in \ca{C}}\ca{S}(i)\right|$ is integral to
compute $\left|\cup_{i\in \ca{C}}\ca{S}(i)\right|$ correctly from an estimate of $\Pr(y=1\mid \ca{E}_{\ca{C}})$ that is within an additive error of $1/4\ell$. Of course, we cannot choose $a=+\infty$ since no samples will satisfy the event $\ca{E}_{\ca{C}}$ in that case. However, we can choose $a$, ($a>0$) carefully so that it is small enough to
make $\Pr(\ca{E}_{\ca{C}})$ reasonably large and at the same time, $a$ is large enough to allow us to correctly compute $\left|\cup_{i\in \ca{C}}\ca{S}(i)\right|$ from a reasonably good estimate of $\Pr(y=1\mid \ca{E}_{\ca{C}})$. Next, we can again use Lemma \ref{lem:mlc_crucial} and Corollary \ref{coro:prelim1} to arrive at the main theorem for Mixtures of Linear Classifiers:

\begin{thm}\label{thm:mlc}
Let $\mathcal{V}$ be a set of $\ell$ unknown vectors in $\bb{R}^n$ satisfying Assumptions \ref{assum:weak2} and \ref{assum:pos}. Let $a=\frac{\sqrt{2(R^2+\sigma^2)}}{\delta}\s{erf}^{-1}\Big(1-\frac{1}{2\ell}\Big)$. 
Then, there exists an algorithm (see Algorithm \ref{algo:mlc} with parameter $a>0$ and \ref{algo:intersection}) that achieves Exact Support Recovery with probability at least $1-\gamma$ using $O\Big( (1-\phi(a))^{-(\log \ell+1)}\ell^2\log (\gamma^{-1}(n+(\ell k)^{\log \ell+1}))\Big)$ samples generated according to $\ca{P}_c$.
\end{thm}

The only comparable result is provided in \citep{sedghi2016provable} who provide parameter estimation guarantees in the MLC setting. However, since it is not evident how to recover the union of support in the sparse MLC setting (unlike MD/MLR); directly applying the result in \citep{sedghi2016provable} will lead to polynomial dependence on $n$ which is undesirable. Moreover, the guarantees in \citep{sedghi2016provable} also require the latent parameter vectors to be linearly independent. In contrast, our sample complexity guarantees for support recovery scale logarithmically with $n$ and also does not need the latent parameter vectors to be linearly independent (in fact they are not even required to be distinct).

\subsection{Mixtures of Linear Regressions}

\noindent Finally, we move on to exact support recovery in the mixtures of linear regression or MLR setting.
Note that the sample complexity guarantees for MLC (Theorem \ref{thm:mlc}) is also valid in the MLR setting as we can simulate MLR responses by simply taking the sign of the response in the MLR dataset. However, note that the sample complexity presented in Theorem \ref{thm:mlc} has a poor dependence on $R,\delta$
and $\ell$. 

\subsubsection{Binary parameter vectors}

\begin{algorithm}[!htbp]
\caption{\textsc{Recover $|\bigcap_{i \in \ca{C}}\ca{S}(i)|$ in MLR setting (Binary Vectors)} \label{algo:mlr_binary}}
\begin{algorithmic}[1]

\REQUIRE Samples $(\fl{x}^{(1)},y^{(1)}),(\fl{x}^{(2)},y^{(2)}),\dots,(\fl{x}^{(m)},y^{(m)}) \sim \ca{P}_r$. Set $\ca{C}\subseteq [n]$. 

\STATE Return $\s{round}\Big(\frac{\ell}{m} \cdot \sum_{j=1}^{m}  \Big(y^{(j)}\Big)^{|\ca{C}|}\Big(\prod_{i\in \ca{C}}\fl{x}^{(j)}_i\Big)\Big)$

\end{algorithmic}
\end{algorithm}

Here we solve the support recovery problem provided the unknown vectors in $\ca{V}$ are all binary and demonstrate significantly better sample complexity guarantees under this assumption. The detailed proofs of all results in this section can be found in subsection \ref{sec:detailed_mlr}.
As usual, we start with a lemma where we characterize the sample complexity of estimating $\left|\bigcap_{i \in \ca{C}}\ca{S}(i)\right|$ correctly for a given subset of indices $\ca{C}\subseteq [n]$:

\begin{lemma}\label{lem:binary_mlr}
If the unknown vectors in the set $\ca{V}$ are all binary i.e. $\fl{v}^{(1)},\fl{v}^{(2)},\dots,\fl{v}^{(\ell)}\in \{0,1\}^n$, then there exists an algorithm (see Algorithm \ref{algo:mlr_binary}) that can compute $\left|\bigcap_{i \in \ca{C}}\ca{S}(i)\right|$  using $O(\ell^2(k+\sigma^2)^{|C|/2}(\log n)^{2|C|}\log \gamma^{-1})$ i.i.d samples from $\ca{P}_r$ with  probability at least $1-\gamma$ for any set $\ca{C}\subseteq [n]$,.
\end{lemma}

We provide a high level proof of Lemma \ref{lem:binary_mlr} here.
We consider the random variable $y^{\left|\ca{C}\right|} \cdot\Big(\prod_{i\in \ca{C}}\fl{x}_i\Big)$ where $(\fl{x},y)\sim \ca{P}_{r}$. Clearly, we can write $y=\langle \fl{v},\fl{x}\rangle+\zeta$ where $\zeta \sim \ca{N}(0,\sigma^2)$ and $\fl{v}$ is uniformly sampled from the set of unknown vectors $\ca{V}$. We can show that 
\begin{align*}
    &\bb{E}_{(\fl{x},y)\sim \ca{P}_r} y^{\left|\ca{C}\right|} \cdot\Big(\prod_{i\in \ca{C}}\fl{x}_i\Big) \\
    &= \bb{E}_{\fl{x},\zeta}  \ell^{-1}\sum_{\fl{v}\in \ca{V}}\Big(\prod_{i\in \ca{C}}\fl{x}_i\Big)\cdot\Big(\langle \fl{v},\fl{x}\rangle+\zeta\Big)^{\left|\ca{C}\right|} \\
    &\bb{E} y^{\left|\ca{C}\right|} \cdot\Big(\prod_{i\in \ca{C}}\fl{x}_i\Big) = \frac{1}{\ell}\sum_{\fl{v}\in \ca{V}}\Big(\prod_{i\in \ca{C}} \bb{E}_{\fl{x}}\fl{x}_i^2\cdot\fl{v}_i\Big)= \frac{\left|\bigcap_{i \in \ca{C}}\ca{S}(i)\right|}{\ell}. 
\end{align*}

Hence, by using the fact that $\left|\bigcap_{i \in \ca{C}}\ca{S}(i)\right|$
is integral, we can estimate the quantity correctly from a reasonably good estimate of $\bb{E} y^{\left|\ca{C}\right|} \cdot\Big(\prod_{i\in \ca{C}}\fl{x}_i\Big)$. Again, by an application of Corollary \ref{coro:prelim1}, we arrive at the following theorem:

\begin{thm}\label{thm:binary}
Let $\mathcal{V}$ be a set of $\ell$ unknown binary vectors in $\{0,1\}^n$. Then, there exists an algorithm (see Algorithm \ref{algo:mlr_binary} and \ref{algo:intersection}) that achieves Exact Support Recovery with probability at least $1-\gamma$ using
\begin{align*}
    O\left(\ell^{2} (\log^4 n(k+\sigma^2))^{\frac{\log \ell+1}{2}} \log((n+(\ell k)^{\log \ell+1})\gamma^{-1})\right)
\end{align*}
 samples generated according to $\ca{P}_r$. 
\end{thm}

As in mixtures of distributions, it is possible to recover the union of support of the unknown vectors in $\ca{V}$ in the MLR setting with a small number of samples (see Lemma \ref{lem:union} in Appendix \ref{app:technical}). Therefore an alternate approach that can be used for support recovery is to recover the union of support followed by parameter estimation with the features being restricted to the union of the support. Note that if the set of unknown vectors satisfy Assumption \ref{assum:weak2}, then estimating each vector up to an $\ell_2$ norm of $\delta$ will suffice for support recovery. Hence, by using Lemma \ref{lem:union} followed by Theorem 1 in \citep{li2018learning}, we arrive at the following result for support recovery:

\begin{thm}\label{thm:mlr_general}
Let $\mathcal{V}$ be a set of $\ell$ unknown vectors satisfying Assumption \ref{assum:weak2}. Further, assume that any two distinct vectors $\fl{v},\fl{v}'\in \ca{V}$ satisfies $\|\fl{v}-\fl{v}'\|_2 \ge \Delta$. Then, there exists an algorithm that achieves Exact Support Recovery with high probability using
\begin{align*}
    O\Big(\ell k \log  \Big(\frac{\ell k}{\delta}\Big)\s{poly}\Big(\frac{\ell \sigma}{\Delta}\Big)+\Big(\frac{\sigma \ell}{\Delta}\Big)^{O(\ell^2)}+
    \ell^2(R^2+\sigma^2)(\log n)^{3}/\delta^2\Big)
\end{align*}
 samples generated according to $\ca{P}_r$.
\end{thm}

If the unknown vectors in $\ca{V}$ are restricted to being binary, then the sample complexity in  Theorem \ref{thm:mlr_general} has a linear dependence on the sparsity but on the other hand, its dependence on $\sigma,\ell$ is very poor; note that Theorem \ref{thm:mlr_general} uses parameter estimation framework in mixtures of Gaussians (\citep{moitra2010settling}) as a black-box leading to the polynomial in $\ell,\sigma$ with a possibly high degree. Moreover, the sample complexity in Theorem \ref{thm:mlr_general} has an $\exp(\ell^2)$ dependence on the number of unknown vectors which is undesirable when the number of unknown vectors $\ell$ is large. In contrast, the sample complexity of Theorem \ref{thm:binary} has a polynomial dependence on $\ell,k,\sigma$ whose degree can be precisely extracted from the expression. In particular, in the regime where $\sigma$ or $\ell$ is large, Theorem \ref{thm:binary} provides significant improvements over the guarantees in Theorem \ref{thm:mlr_general}. Finally, although not mentioned explicitly in Theorem 1 in \citep{li2018learning}, it can be extracted that the sample complexity is polynomial in $\gamma^{-1}$ where $\gamma$ is the failure probability; this leads to a similar dependence on the failure probability in Theorem \ref{thm:mlr_general}. On the other hand, the sample complexity in Theorem \ref{thm:binary} depends logarithmically on $\gamma^{-1}$.

\subsubsection{General Case}

\begin{algorithm}[!htbp]
\caption{\textsc{Recover $|\bigcup_{i \in \ca{C}}\ca{S}(i)|$ in MLR setting (General)} \label{algo:mlr_general}}
\begin{algorithmic}[1]

\REQUIRE Samples $(\fl{x}^{(1)},y^{(1)}),(\fl{x}^{(2)},y^{(2)}),\dots,(\fl{x}^{(m)},y^{(m)}) \sim \ca{P}_r$. Set $\ca{C}\subseteq [n]$ and parameters $\alpha,\epsilon>0$. 

\STATE Sample the vector $\fl{a}\in \bb{R}^{\left|\ca{C}\right|}$ by generating each entry independently from $\ca{N}(0,\alpha^2)$.

\STATE Take samples $\{y^{(j)}\}_{j=1}^{m/2}$ (see \ref{eq:dist1}) and use algorithm in \cite{moitra2010settling} (parameter estimation in mixture of Gaussians) to compute an estimate of the component parameters. Denote the set of estimated parameter triplets by $\ca{T}$ where each triplet comprises of estimated weight, estimated mean and estimated variance of some component.

\STATE Take samples $\{y^{(j)}+\langle \fl{a},\fl{x}_{\mid \ca{C}} \rangle\}_{j=m/2+1}^{m}$ (see \ref{eq:dist2}) and use algorithm in \cite{moitra2010settling} to compute an estimate of the component parameters. Denote the set of estimated parameter triplets by $\ca{T}'$ where each triplet comprises of estimated weight, estimated mean and estimated standard deviations of some component. 

\STATE Consider a bipartite graph where the left nodes correspond to elements in $\ca{T}$ and the right nodes correspond to elements in $\ca{T}'$. For each node in $\ca{T}'$, draw an edge with a node in $\ca{T}$ if their estimated standard deviations are $\epsilon$ close to each other. Compute $w$ to be the sum of weights of the right nodes that have an edge.  

\STATE Return $\s{round}\Big(\ell(1-w)\Big)$

\end{algorithmic}
\end{algorithm}

Here, we make the following mild assumption on the unknown vectors in $\ca{V}$. 

\begin{assumption}[Minimum Norm Gap]\label{assum:weak}
We assume that the minimum gap $\Delta$ of a set of unknown vectors $\ca{V}$
\begin{align*}
    \Delta \triangleq \min_{\substack{\fl{v},\fl{v}'\in \ca{V} \\ \left|\left|\fl{v}\right|\right|_2 \neq \left|\left|\fl{v}'\right|\right|_2}} \left|\left|\left|\fl{v}\right|\right|_2-\left|\left|\fl{v}'\right|\right|_2\right|,
\end{align*}
defined to be the smallest gap between the Euclidean norms of any two unknown vectors with distinct Euclidean norms is known.
\end{assumption}

Assumption \ref{assum:weak}, stated for the set of unknown vectors $\ca{V}$, imply that vectors with different euclidean norms have some separation between the norms. Note that Assumption \ref{assum:weak} does not rule out the presence of multiple unknown vectors with the same euclidean norm. We show the following result in the general setting using our techniques that are significantly different from \cite{li2018learning}. 

\begin{thm}\label{thm:mlr_general_new}
Let $\mathcal{V}$ be a set of $\ell$ unknown vectors in $\bb{R}^n$ satisfying Assumptions \ref{assum:weak2} and \ref{assum:weak}. Let $\alpha= \frac{1}{2}\Big(\frac{\delta}{2R\sqrt{\log \gamma^{-1}}}\min\Big(\Delta ,\frac{1}{2}\Big)\Big)$ and $\epsilon=\ell^{-3}/2$.
Then, with probability at least $1-\gamma$, there exists an algorithm (see Algorithm~\ref{algo:mlr_general} with parameters $\alpha,\epsilon>0$ and Algorithm \ref{algo:intersection}) that achieves Exact Support Recovery with  $O\Big(\s{poly}(\Delta^{-1},\delta^{-1},\gamma^{-1},\sigma,R,\log n)\cdot\exp(\log \ell\log (\ell k))\Big)$ samples generated according to $\ca{P}_r$.
\end{thm}

The detailed proof of Theorem  \ref{thm:mlr_general_new} is provided in Section \ref{app:general_mlr_new}. Note that our result improves on the guarantee provided in Theorem \ref{thm:mlr_general} when the number of components in large; this is because the sample complexity stated in Theorem \ref{thm:mlr_general} is exponential in the number of components $\ell$ while the sample complexity in Theorem \ref{thm:general} is polynomial in $\ell$. On the other hand, the sample complexity in Theorem \ref{thm:mlr_general} is linear in the sparsity $k$ (optimal) while that of Theorem \ref{thm:general} is polynomial in $k$. As usual, the main workhorse in the proof of Theorem \ref{thm:general} is the following lemma where we characterize the sample complexity of estimating $\left|\cup_{i\in \ca{C}}\ca{S}(i)\right|$ correctly for a given subset of indices $\ca{C}\subseteq [n]$:

\begin{lemma}\label{lem:set}
Let $\mathcal{V}$ be a set of $\ell$ unknown vectors in $\bb{R}^n$ satisfying Assumptions \ref{assum:weak2} and \ref{assum:weak}. Let $\alpha= \frac{1}{2}\Big(\frac{\delta}{2R\sqrt{\log \gamma^{-1}}}\min\Big(\Delta ,\frac{1}{2}\Big)\Big)$ and $\epsilon=\ell^{-3}/2$.
Then, for any set $\ca{C}\subseteq [n]$, with probability at least $1-\gamma$, there exists an algorithm (see Algorithm~\ref{algo:mlr_general} with parameters $\alpha,\epsilon>0$)
that can compute $\left|\bigcup_{i \in \ca{C}}\ca{S}(i)\right|$ using $O\Big(\s{poly}(\Delta^{-1},\delta^{-1},\gamma^{-1},\sigma,R)\cdot\exp(\left|\ca{C}\right|\log (\ell k))\Big)$ samples generated according to $\ca{P}_r$.
\end{lemma}

Let us provide a proof sketch of Lemma \ref{lem:set}. Assume for now that we know the union of support $\cup_{\fl{v}\in \ca{V}} \s{supp}(\fl{v})$ of the unknown vectors in $\ca{V}$ (see Lemma \ref{lem:union} in Appendix \ref{app:technical} on a simple procedure for recovering the union of support). First of all, we show that for any sample $(\fl{x},y)\sim \ca{P}_r$ and a fixed vector $\fl{a}\in \bb{R}^{\left|\ca{C}\right|}$, the random variable $y+\langle \fl{a},\fl{x}_{\mid \ca{C}} \rangle$ is distributed according to the following mixture of Gaussians:
 \begin{align*}
     y+\langle \fl{a},\fl{x}_{\mid \ca{C}} \rangle \sim \frac{1}{\ell} \sum_{\fl{v}\in \ca{V}}\ca{N}\Big(0,\left|\left|\fl{v}\right|\right|_2^2+\left|\left|\fl{a}\right|\right|_2^2+2\langle \fl{a},\fl{v}_{\mid \ca{C}} \rangle+\sigma^2\Big).
 \end{align*} 
The proof follows by conditioning on each unknown vector $\fl{v}\in \ca{V}$ being used for generating the sample $(\fl{x},y)$ (see Lemma \ref{lem:mlr_general_crucial1} for details). Fix any unknown vector $\fl{v} \in \ca{V}$ such that $\fl{v}_{\mid \ca{C}} = \fl{0}$ i.e. all indices in $\fl{v}$ constrained to the set $\ca{C}$ are zero. In that case, we have 
\begin{align*}
    \left|\left|\fl{v}\right|\right|_2^2+\left|\left|\fl{a}\right|\right|_2^2+2\langle \fl{a},\fl{v}_{\mid \ca{C}} \rangle = \left|\left|\fl{v}\right|\right|_2^2+\left|\left|\fl{a}\right|\right|_2^2.
\end{align*}
On the other hand, for any unknown vector $\fl{v} \in \ca{V}$ such that not all indices in $\fl{v}_{\mid \ca{C}}$ are zero and $\langle \fl{v}_{\mid \ca{C}},\fl{a} \rangle \neq 0$, we must have 
\begin{align*}
    \left|\left|\fl{v}\right|\right|_2^2+\left|\left|\fl{a}\right|\right|_2^2+2\langle \fl{a},\fl{v}_{\mid \ca{C}} \rangle \neq  \left|\left|\fl{v}\right|\right|_2^2+\left|\left|\fl{a}\right|\right|_2^2
\end{align*} 

Our main idea to compute $\left|\bigcup_{i \in \ca{C}}\ca{S}_{\ca{V}}(i)\right|$ for a fixed set $\ca{C}\subseteq [n]$ is the following: we obtain samples from two different mixtures of Gaussians. The first set of samples is obtained from the distribution  
 \begin{align}\label{eq:dist1}
    \ca{D} \equiv \frac{1}{\ell} \sum_{\fl{v}\in \ca{V}}\ca{N}\Big(0,\left|\left|\fl{v}\right|\right|_2^2+\sigma^2\Big).
 \end{align} 
by simply choosing $\fl{a}$ to be the zero vector (i.e. just collecting the $y$'s from the set of samples $\{(\fl{x},y)\}\sim \ca{P}_r$. The second set of samples (by collecting the samples is obtained from the distribution
\begin{align}\label{eq:dist2}
     \ca{D}'\equiv \frac{1}{\ell} \sum_{\fl{v}\in \ca{V}}\ca{N}\Big(0,\left|\left|\fl{v}\right|\right|_2^2+\left|\left|\fl{a}\right|\right|_2^2+2\langle \fl{a},\fl{v}_{\mid \ca{C}} \rangle+\sigma^2\Big).
\end{align}
where the vector $\fl{a}\in \bb{R}^{\left|\ca{C}\right|}$ is chosen carefully. Suppose we recover the parameters of $\ca{D},\ca{D}'$ exactly by using a certain number of samples. In that case, if $\left|\left|\fl{a}\right|\right|_2$ is small enough, then for the components in the mixture $\ca{D}$ parameterized by $\fl{v}\in \ca{V}$ satisfying $\fl{v}_{\mid \ca{C}}=\f{0}$, we can identify it with the corresponding component in $\ca{D}'$ (the variance is $\left|\left|\fl{v}\right|\right|^2_2+\sigma^2$ in $\ca{D}$ and $\left|\left|\fl{v}\right|\right|_2^2+\left|\left|\fl{a}\right|\right|_2^2+\sigma^2$ where $\fl{a},\sigma$ is known). Moreover, we can also show that for the components in $\ca{D}$ parameterized by $\fl{v}\in \ca{V}$ satisfying $\fl{v}_{\mid \ca{C}} \neq 0$, if $\langle \fl{a}, \fl{v}_{\mid \ca{C}} \rangle$ is large enough, then we can show that it will not be separated by $\left|\left|\fl{a}\right|\right|_2^2$ from any of the components in $\ca{D}$. Hence, we can identify the components (or more generally the weight of the components) for which $\fl{v}_{\mid \ca{C}}=0$; we can compute $\left|\bigcup_{i \in \ca{C}}\ca{S}_{\ca{V}}(i)\right|$ from this quantity. Of course, we cannot recover the parameters of $\ca{D},\ca{D}'$ exactly but it is possible to obtain good estimates of the parameters by using the algorithm in \cite{moitra2010settling} for parameter estimation in Gaussian mixtures. Finally, we can show that if we choose the vector $\fl{a}$ by randomly sampling each entry from $\ca{N}(0,\alpha^2)$ independently (where $\alpha $ is another parameter that has to be carefully chosen), our arguments work out. Next, in order to complete the proof of Lemma \ref{lem:mlr_general_crucial1} we simply need to take a union bound over all sets $\ca{C}$ of a fixed size $s$ which are subsets of the union of support of the unknown vectors. As before, by an application of Corollary \ref{coro:prelim1}, we arrive at Theorem \ref{thm:general}.


\section{Maximal Support Recovery}
\label{sec:dedup}

\subsection{Mixtures of Distributions}\label{sec:dedup_md}

In this section, we will present our main results for maximal support recovery in the MD setting. The detailed proofs of all results in this section can be found in Section \ref{sec:detailed_md}.

\begin{algorithm}[!htbp]
\caption{\textsc{Estimate if $|\bigcap_{i \in \ca{C}}\ca{S}(i)|>0$ in MD setting} \label{algo:md2}}
\begin{algorithmic}[1]
\REQUIRE Samples $\fl{x}^{(1)},\fl{x}^{(2)},\dots,\fl{x}^{(m)} \sim \ca{P}_d$. Set $\ca{C}\subseteq [n]$.

\STATE For every $\fl{z}\le 2\fl{1}_{|\ca{C}|}$, compute estimate $\widehat{U}^{\fl{z}}$ of $\bb{E}\prod_{i \in \ca{C}}\fl{x}_i^{\fl{z}_{\pi(\ca{C},i)}}$ using Algorithm \ref{algo:estimate} on the set of samples $\{(\fl{x}_i^{j})^{\fl{z}_{\pi(\ca{C},i)}}\}_{j=1}^{m}$.

\STATE For every $\fl{z}\le 2\fl{1}_{|\ca{C}|}$, compute an estimate $\widehat{V}^{\fl{z}}$ of $\sum_{j \in [\ell]}\prod_{i \in \ca{C}}(\fl{v}^{(j)}_i)^{\fl{z}_{\pi_{\ca{C},i}}}$ recursively using the following equation:
\begin{align*}
    \ell \widehat{U}^{\fl{z}}-\sum_{\fl{u}<\fl{z}} \zeta_{\fl{z},\fl{u}}\cdot \widehat{V}^{\fl{u}} =\zeta_{\fl{z},\fl{z}} \cdot \widehat{V}^{\fl{z}}. 
\end{align*}
\STATE If $\widehat{V}^{2\fl{1}_{\left|\ca{C}\right|}} \ge \delta^{2\left|\ca{C}\right|}/2$, return True and otherwise return False. 

\end{algorithmic}
\end{algorithm}

Now, we provide results on maximal support recovery in the MD setting. Note that from Lemma \ref{lem:partial2}, for Maximal support recovery, we only need to estimate correctly if $\left|\bigcap_{i \in \ca{C}}\ca{S}(i)\right|>0$ for ordered sets $\ca{C}\subseteq [n]$. Notice that $\left|\bigcap_{i \in \ca{C}}\ca{S}(i)\right|>0$ if and only if $\sum_{\fl{v}\in \ca{V}}\prod_{i \in \ca{C}} \fl{v}^2_i >0$. From our previous arguments, $\sum_{\fl{v}\in \ca{V}}\prod_{i \in \ca{C}} \fl{v}^2_i$ can be computed if the quantities $\bb{E}\prod_{i \in \ca{C}} \fl{x}_i^{\fl{u}_{\pi(\ca{C},i)}}$ for all $\fl{u}\le 2\fl{1}_{|\ca{C}|}$ are pre-computed. The following lemma stems from making this computation robust to the randomness in the dataset:

\begin{lemma}\label{lem:partial_md}
Suppose Assumption \ref{assum:weak2} is true. Let 
\begin{align*}
    &\Phi \triangleq  \max_{\fl{z}\le 2\fl{1}_{|\ca{C}|}} \frac{\delta^{2|\ca{C}|}}{2}\Big(\frac{\ell }{\zeta_{\fl{z},\fl{z}}}+\sum_{\fl{u}<\fl{z}} \sum_{\s{M}\in \ca{M}(\fl{z},\fl{u}) }\frac{\ell  \prod_{(\fl{r},\fl{s})\in \s{M}}\zeta_{\fl{r},\fl{s}}}{\prod_{\fl{r}\in \ca{T}(\s{M})}\zeta_{\fl{r},\fl{r}}}\Big)^{-1} \\ 
    &h_{\ell,\ca{V}} \triangleq \frac{\max_{\fl{z}\le 2\fl{1}_{|\ca{C}|}}\bb{E}\prod_{i \in \ca{C}}\fl{x}_i^{2\fl{z}_{\pi(\ca{C},i)}}}{\Phi^{ 2}}  
\end{align*}
where $h_{\ell,\ca{V}}$ is a constant independent of $k$ and $n$ but depends on $\ell$. There exists an algorithm (see Algorithm \ref{algo:md2}) that can compute if $\left|\bigcap_{i \in \ca{C}}\ca{S}(i)\right|>0$ correctly for each set $\ca{C}\subseteq [n]$ with probability at least $1-\gamma$ using $O(h_{\ell,\ca{V}}\log \gamma^{-1})$ samples generated according to $\ca{P}_d$.
\end{lemma}

The subsequent theorem follows from Lemma \ref{lem:partial_md} and Corollary \ref{coro:partial2}. Note that, compared to exact support recovery (Theorem \ref{thm:general_md}) the sample complexity for maximal support recovery has significantly improved dependency on $\delta$ and furthermore, it is also independent of $R$. 

\begin{thm}\label{thm:general_md}
Let $\ca{V}$ be a set of unknown vectors in $\bb{R}^n$ satisfying Assumption \ref{assum:weak2}. Let $\ca{F}_m = \ca{Q}_1([n]) \cup \ca{Q}_m(\cup_{\fl{v}\in \ca{V}}\s{supp}(\fl{v}))$ and
\begin{align*}
    & \Phi_{m} = \max_{\fl{z}\le 2\fl{1}_{|\ca{C}|}} \frac{\delta^{2|\ca{C}|}}{2}\Big(\frac{\ell }{\zeta_{\fl{z},\fl{z}}}+\sum_{\fl{u}<\fl{z}} \sum_{\s{M}\in \ca{M}(\fl{z},\fl{u}) }\frac{\ell  \prod_{(\fl{r},\fl{s})\in \s{M}}\zeta_{\fl{r},\fl{s}}}{\prod_{\fl{r}\in \ca{T}(\s{M})}\zeta_{\fl{r},\fl{r}}}\Big)^{-1}\\ 
    &h'_{\ell,\ca{V}} \triangleq \max_{\substack{\fl{z}\le 2\fl{1}_{\ell} \\ \ca{C}\in \ca{F}_{ \ell}}}\frac{\bb{E}\prod_{i \in \ca{C}}\fl{x}_i^{2\fl{z}_{\pi(\ca{C},i)}}}{\Phi_{\ell}^{2}} 
\end{align*}
where $h'_{\ell,\ca{V}}$ is a constant independent of $k$ and $n$ but depends on $\ell$.
 Accordingly, there exists an algorithm (see Algorithm \ref{algo:md2} and \ref{algo:partial_highprob}) that achieves maximal support recovery with probability at least $1-\gamma$ using $O\Big(h'_{\ell,\ca{V}}\log (\gamma^{-1}(n+(\ell k)^{\ell}))\Big)$ samples generated from $\ca{P}_d$.
\end{thm}

\subsection{Mixtures of Linear Regressions}\label{sec:dedup_mlr}

Our final results are for maximal support recovery in the MLR setting under different assumptions. Below, we state Assumption \ref{assum:second} which is a generic condition and if satisfied by the set of unknown vectors $\ca{V}$ allows for maximal support recovery of $\ca{V}$. All missing proofs in this section can be found in Appendix \ref{sec:detailed_mlr}.

\begin{algorithm}[!htbp]
\caption{\textsc{Estimate if $|\bigcap_{i \in \ca{C}}\ca{S}(i)>0|$ in MLR setting} \label{algo:mlr_2}}
\begin{algorithmic}[1]

\REQUIRE Samples $(\fl{x}^{(1)},y^{(1)}),(\fl{x}^{(2)},y^{(2)}),\dots,(\fl{x}^{(m)},y^{(m)}) \sim \ca{P}_r$. Set $\ca{C}\subseteq [n]$. 

\STATE If $\frac{2\ell}{m} \cdot \sum_{j=1}^{m}  \Big(y^{(j)}\Big)^{|\ca{C}|}\Big(\prod_{i\in \ca{C}}\fl{x}^{(j)}_i\Big) \ge \alpha_{\left|\ca{C}\right|}$, return True else return False.
\end{algorithmic}
\end{algorithm}

\begin{assumption}\label{assum:second}
We assume that there exists positive numbers $\alpha_1,\alpha_2,\dots,\alpha_{\ell}>0$ such that for all sets $\ca{C}\subseteq [n],|C|\le \ell$
the following condition is satisfied by the set of $\ell$ unknown vectors $\fl{v}^{(1)},\fl{v}^{(2)},\dots,\fl{v}^{(\ell)}\in \ca{V}$:
\begin{align*}
&\text{ If there exists } \fl{v}\in\ca{V} \text{ such that }\prod_{j \in \ca{C}} \fl{v}_j \neq 0 \\
&\text{ then }
    \left|\sum_{\fl{v}\in \ca{V}} \Big(\prod_{j \in \ca{C}} \fl{v}_j\Big)\right| \ge \alpha_{|\ca{C}|}.
\end{align*}
\end{assumption}

\begin{thm}\label{thm:general}
Suppose the following conditions are satisfied:
\begin{enumerate}
    \item All unknown vectors in  $\ca{V}$ are bounded within a ball of radius $R$ i.e. $ \left|\left|\fl{v}^{(i)}\right|\right|_2 \le R \text{ for all } i \in [\ell]$.  
    
    \item Assumption \ref{assum:second} is satisfied by the set of unknown vectors $\ca{V}$.
\end{enumerate}
 Accordingly, there exists an algorithm (see Algorithms \ref{algo:mlr_2} and \ref{algo:partial_highprob}) that  achieves maximal support recovery with probability at least $1-\gamma$ using 
 \begin{align*}
     O(\ell^2(R^2+\sigma^2)^{\ell/2}(\log n)^{2\ell}\log((n+(\ell k)^{\ell})\gamma^{-1})/\alpha_{\ell}^2)
 \end{align*}
  samples from $\ca{P}_r$.
\end{thm}

Next, using Theorem \ref{thm:general}, we provide maximal support recovery guarantees in two cases: 1) The set of unknown vectors in $\ca{V}$ satisfies Assumptions \ref{assum:weak2} and all unknown parameters are non-negative 2) The non-zero entries in the unknown vectors in $\ca{V}$ are distributed according to a zero mean Gaussian $\ca{N}(0,\nu^2)$.

\begin{coro}\label{coro:non-negative}
Consider a set of $\ell$ unknown vectors $\ca{V}$ that satisfies Assumptions \ref{assum:weak2} and furthermore, every non-zero entry in all the unknown vectors is positive ($\fl{v}_i \ge 0$ for all $i\in[n],\fl{v}\in \ca{V}$). In that case, Assumption \ref{assum:second} is satisfied with $\alpha_{|\ca{C}|} \ge \delta^{|\ca{C}|}$.  Accordingly, there exists an algorithm that achieves maximal support recovery  with probability at least $1-\gamma$ using
 \begin{align*}
     O(\ell^2(R^2+\sigma^2)^{\ell/2}(\log n/\delta)^{2\ell}\log((n+(\ell k)^{\ell})\gamma^{-1}))
 \end{align*}
samples from $\ca{P}_r$.
\end{coro}

\begin{coro}\label{coro:gaussian}
If all non-zero entries in the set of unknown vectors $\ca{V}$ are sampled i.i.d according to $\ca{N}(0,\nu^2)$, then with probability $1-\eta$, Assumption \ref{assum:second} is satisfied with $\alpha_{|\ca{C}|} \ge \delta_{|\ca{C}|}^{|\ca{C}|}$ where 
\begin{align*}
    \delta_{{|\ca{C}|}}=\Big(\sqrt{\frac{\pi}{8}} \frac{\nu\eta}{\ell|\ca{C}|(\ell k)^{|\ca{C}|}}\Big).
\end{align*}
 Conditioned on this event,  there exists an Algorithm that achieves maximal support recovery with probability at least $1-\gamma$ using
 \begin{align*}
     O(\ell^2(R^2+\sigma^2)^{\ell/2}(\log n)^{2\ell}\log((n+(\ell k)^{\ell})\gamma^{-1})/\delta_{\ell}^2)
 \end{align*}
 samples from $\ca{P}_r$.
\end{coro}

\begin{rmk}[Computational complexity]
All our algorithms described in the MD/MLR/MLC settings are efficient namely their computational complexities are polynomial in the dimension $n$ and sparsity $k$.
\end{rmk}

\paragraph{Acknowledgement:} This research is supported in part by NSF awards CCF 2133484 and CCF 1934846.

\appendix

\onecolumn

\section{Detailed Proofs}\label{sec:detailed}

\subsection{Mixtures of Distributions}\label{sec:detailed_md}

\begin{lemma}\label{lem:poly}
  For each fixed set $\ca{C}\subseteq [n]$ and each vector $\fl{t}\in (\bb{Z}^{+})^{\left|\ca{C}\right|}$, we must have 
  \begin{align*}
      \bb{E} \prod_{i \in \ca{C}}\fl{x}_i^{\fl{t}_{\pi(\ca{C},i)}} = \frac{1}{\ell}\sum_{\fl{u} \le \fl{t}}\zeta_{\fl{t},\fl{u}} \cdot \Big(\sum_{j \in [\ell]}\prod_{i \in \ca{C}}  (\fl{v}^{(j)}_i)^{\fl{u}_{\pi(\ca{C},i)}} \Big).
  \end{align*} 
\end{lemma}

\begin{proof}
 We will have
\begin{align*}
\bb{E}\prod_{i \in \ca{C}}\fl{x}_i^{\fl{t}_{\pi(\ca{C},i)}} = \frac{1}{\ell} \sum_{j \in [\ell]} \Big(\prod_{i \in \ca{C}} q_{\fl{t}_{\pi(\ca{C},i)}}(\fl{v}^{(j)}_i)\Big).     
\end{align*}
From the above equations, note that each summand is a product of polynomials in $\fl{v}^{(j)}_i$ for a fixed $j$. Expanding the polynomial and using the fact that $\zeta_{\fl{t},\fl{u}}=\prod_{i \in \ca{C}} \beta_{\fl{t}_{\pi(\ca{C},i)},\fl{u}_{\pi(\ca{C},i)}+1}$ is the coefficient of the monomial $\prod_{i \in \ca{C}}  (\fl{v}^{(j)}_i)^{\fl{u}_{\pi(\ca{C},i)}}$ for all $j\in [\ell]$, we obtain the proof of the lemma. 

\end{proof}

\begin{lemma}\label{lem:recurse1}
 For each fixed set $\ca{C}\subseteq [n]$ and each vector $\fl{t}\in (\bb{Z}^{+})^{\left|\ca{C}\right|}$, we can compute $\sum_{j \in [\ell]}\prod_{i \in \ca{C}} (\fl{v}^{(j)}_i)^{\fl{t}_{\pi(\ca{C},i)}}$ provided for all $\fl{u}\in (\bb{Z}^{+})^{\left|\ca{C}\right|}$ satisfying $\fl{u}\le\fl{t}$, the quantities $\bb{E}\prod_{i \in \ca{C}} \fl{x}_i^{\fl{u}_{\pi(\ca{C},i)}}$ are pre-computed.
\end{lemma}

\begin{proof}
We will prove this lemma by induction. For the base case, we have from Lemma \ref{lem:poly} that $\ell\bb{E}\fl{x}_i=\beta_{1,2}\sum_{j \in [\ell]}\fl{v}^{(j)}_i+\beta_{1,1}$. Hence $\sum_{j \in [\ell]}\fl{v}^{(j)}_i$ can be computed from $\bb{E}\fl{x}_i$ by using the following equation:
\begin{align*}
    \sum_{j \in [\ell]}\fl{v}^{(j)}_i = \frac{1}{\beta_{1,2}}\Big(\ell\bb{E}\fl{x}_i-\beta_{1,1}\Big).
\end{align*}

Now suppose for all vectors $\fl{u}\in (\bb{Z}^{+})^{\left|\ca{C}\right|}$ satisfying $\fl{u}\le \fl{t}$, the lemma statement is true. Consider another vector $\fl{z}\in (\bb{Z}^{+})^{\left|\ca{C}\right|}$ such that there exists an index $j\in |\ca{C}|$ for which $\fl{z}_j= \fl{t}_j+1$ and $\fl{z}_i= \fl{t}_i$ for all $i\neq j$. From the statement of Lemma \ref{lem:poly}, we know that
\begin{align*}
      \bb{E} \prod_{i \in \ca{C}}\fl{x}_i^{\fl{z}_{\pi(\ca{C},i)}} = \frac{1}{\ell}\sum_{\fl{u} \le \fl{z}}\zeta_{\fl{z},\fl{u}} \cdot \Big(\sum_{j \in [\ell]}\prod_{i \in \ca{C}}  (\fl{v}^{(j)}_i)^{\fl{u}_{\pi(\ca{C},i)}} \Big)
  \end{align*}
where $\zeta_{\fl{z},\fl{u}}=\prod_{i \in \ca{C}} \beta_{\fl{z}_{\pi(\ca{C},i)},\fl{u}_{\pi(\ca{C},i)}+1}$. From our induction hypothesis, we have already computed  $\sum_{j \in [\ell]}\prod_{i \in \ca{C}}  (\fl{v}^{(j)}_i)^{\fl{u}_{\pi(\ca{C},i)}}$ for all $\fl{u}< \fl{z}$ (the set $\{\fl{u}\in (\bb{Z}^{+})^{\left|\ca{C}\right|}\mid \fl{u}<\fl{z}\}$ is equivalent to the set $\{\fl{u}\in (\bb{Z}^{+})^{\left|\ca{C}\right|}\mid \fl{u} \le \fl{t}\}$). Since $\bb{E} \prod_{i \in \ca{C}}\fl{x}_i^{\fl{z}_{\pi(\ca{C},i)}}$ is already pre-computed, we can compute $\sum_{j \in [\ell]}\prod_{i \in \ca{C}}  (\fl{v}^{(j)}_i)^{\fl{z}_{\pi(\ca{C},i)}}$ as follows:
\begin{align*}
      \ell\bb{E} \prod_{i \in \ca{C}}\fl{x}_i^{\fl{z}_{\pi(\ca{C},i)}} - \sum_{\fl{u} < \fl{z}}\zeta_{\fl{z},\fl{u}} \cdot \Big(\sum_{j \in [\ell]}\prod_{i \in \ca{C}}  (\fl{v}^{(j)}_i)^{\fl{u}_{\pi(\ca{C},i)}} \Big) = \zeta_{\fl{z},\fl{z}} \cdot \Big(\sum_{j \in [\ell]}\prod_{i \in \ca{C}}  (\fl{v}^{(j)}_i)^{\fl{z}_{\pi(\ca{C},i)}} \Big). 
 \end{align*}
This completes the proof of the lemma.
\end{proof}

\begin{lemma}\label{lem:recurse2}
 For each fixed set $\ca{C}\subseteq [n]$, we can compute $\left|\bigcap_{i \in \ca{C}}\ca{S}(i)\right|$ provided for all $p\in [\ell]$, the quantity $\sum_{\fl{v}\in \ca{V}} \Big(\prod_{i\in \ca{C}}  \fl{v}_i^{2}\Big)^{p}$ is pre-computed.
\end{lemma}

\begin{proof}
Let us fix a particular subset $\ca{C}\subseteq [n]$. Now, let us define the quantity
\begin{align*}
 \s{A}_{\ca{C},t}=\sum_{\substack{\ca{C}'\subseteq [\ell] \\ \left|\ca{C}'\right|=t}} \prod_{\substack{i \in \ca{C}\\ j \in \ca{C}'}} (\fl{v}^{(j)}_i)^{2}   
\end{align*}
Notice that $\s{A}_{\ca{C},t} > 0$ if and only if there exists a subset $\ca{C}'\subseteq [\ell],|\ca{C}'|=t$ such that $\fl{v}^{(j)}_i \neq 0$ for all $i\in \ca{C},j \in \ca{C}'$. Hence, the maximum value of $t$ such that $\s{A}_{\ca{C},t} > 0$ is the number of unknown vectors in $\ca{V}$ having non-zero value in all the indices in $\ca{C}$. In other words, we have that 
\begin{align*}
    \left|\bigcap_{i \in \ca{C}}\ca{S}(i)\right| = \max_{t\in [\ell]} t\cdot\fl{1}[\s{A}_{\ca{C},t} > 0].  
\end{align*}
Let $t^{\star}$ be the maximum value of $t$ for which $\s{A}_{\ca{C},t}>0$. We will have $\s{A}_{\ca{C},t^{\star}} \ge  \delta^{2\ell|\ca{C}|}$. It is easy to recognize $\sum_{\fl{v}\in \ca{V}} \Big(\prod_{i\in \ca{C}}  \fl{v}_i^{2}\Big)^{p}$ as the power sum polynomial of degree $p$ in the variables $\{\prod_{i \in \ca{C}} \fl{v}_i^2\}_{\fl{v}\in \ca{V}}$. On the other hand, $\s{A}_{\ca{C},t}$ is the elementary symmetric polynomial of degree $t$ in the variables $\{\prod_{i \in \ca{C}} \fl{v}_i^2\}_{\fl{v}\in \ca{V}}$. We can use Newton's identities to state that for all $t \in [\ell]$, 
\begin{align*}
    t\s{A}_{\ca{C},t} = \sum_{p=1}^t (-1)^{p+1} \s{A}_{\ca{C},t-p} \Big(\sum_{\fl{v}\in \ca{V}} \Big(\prod_{i\in \ca{C}}  \fl{v}_i^{2}\Big)^{p}\Big)
\end{align*}
using which, we can recursively compute $\s{A}_{\ca{C},t}$ for all $t\in [\ell]$ if we were given $\sum_{\fl{v}\in \ca{V}} \Big(\prod_{i\in \ca{C}}  \fl{v}_i^{2}\Big)^{p} $ as input for all $p\in [\ell]$. We can also express $\s{A}_{\ca{C},t}$ as a complete exponential Bell polynomial $\s{B}_t$
\begin{align*}
    \s{A}_{\ca{C},t} = \frac{(-1)^n}{n!} \s{B}_t\Big(-\sum_{\fl{v}\in \ca{V}} \prod_{i\in \ca{C}}  \fl{v}_i^{2},-1!\Big(\sum_{\fl{v}\in \ca{V}} \prod_{i\in \ca{C}}  \fl{v}_i^{2}\Big)^2, -2!\Big(\sum_{\fl{v}\in \ca{V}} \prod_{i\in \ca{C}}  \fl{v}_i^{2}\Big)^3,\dots,-(t-1)!\Big(\sum_{\fl{v}\in \ca{V}} \prod_{i\in \ca{C}}  \fl{v}_i^{2}\Big)^t\Big).
\end{align*}

\end{proof}

We are now ready to prove Lemma \ref{lem:md_crucial}.

\begin{lemmau}[Restatement of Lemma \ref{lem:md_crucial}]
Suppose Assumption \ref{assum:weak2} is true. Let  
\begin{align*}
  & \Phi \triangleq \frac{\delta^{2\ell|\ca{C}|}}{2\Big(3\max(\ell R^{2\ell|\ca{C}|},2^{\ell}R^{\ell+|\ca{C}|})\Big)^{(\ell-1)}\ell!}\Bigg(\max_{ \fl{z}\le 2\ell\fl{1}_{|\ca{C}|}} \frac{\ell }{\zeta_{\fl{z,\fl{z}}}}+\sum_{\fl{u}<\fl{z}} \sum_{\s{M}\in \ca{M}(\fl{z},\fl{u}) }\frac{\ell  \prod_{(\fl{r},\fl{s})\in \s{M}}\zeta_{\fl{r},\fl{s}}}{\prod_{\fl{r}\in \ca{T}(\s{M})}\zeta_{\fl{r},\fl{r}}}\Bigg)^{-1}\\
  &g_{\ell,\ca{V}} \triangleq \frac{\max_{\fl{z}\le 2\ell\fl{1}_{|\ca{C}|}}\bb{E}\prod_{i \in \ca{C}}\fl{x}_i^{2\fl{z}_{\pi(\ca{C},i)}}}{\Phi^{ 2}}
\end{align*}
where $g_{\ell,\ca{V}}$ is a constant that is independent of $k$ and $n$ but depends on $\ell$.
  There exists an algorithm (see Algorithm \ref{algo:md}) that can compute $\left|\bigcap_{i \in \ca{C}}\ca{S}(i)\right|$ exactly for each set $\ca{C}\subseteq [n]$ with probability at least $1-\gamma$ using $O\Big(\log (\gamma^{-1}(2\ell)^{|\ca{C}|})f_{\ell,\ca{V}}\Big)$ samples generated according to $\ca{P}_d$.
\end{lemmau}

\begin{proof}
Suppose, for every vector $\fl{z}\in (\bb{Z}^{+})^{|\ca{C}|}$ satisfying $\fl{z} \le 2\ell \fl{1}_{|\ca{C}|}$, 
we compute an estimate $\widehat{U}^{\fl{z}}$ of $\bb{E}\prod_{i \in \ca{C}}\fl{x}_i^{\fl{z}_{\pi(\ca{C},i)}}$ such that $\left|\widehat{U}^{\fl{z}}-\bb{E}\prod_{i \in \ca{C}}\fl{x}_i^{\fl{z}_{\pi(\ca{C},i)}}\right|\le \Phi_{\fl{z}}$ where $\Phi_{\fl{z}}$ is going to be determined later. Recall that in Lemma \ref{lem:recurse2}, we showed 
\begin{align}\label{eq:diffi}
      \ell\bb{E} \prod_{i \in \ca{C}}\fl{x}_i^{\fl{z}_{\pi(\ca{C},i)}} - \sum_{\fl{u} < \fl{z}}\zeta_{\fl{z},\fl{u}} \cdot \Big(\sum_{j \in [\ell]}\prod_{i \in \ca{C}}  (\fl{v}^{(j)}_i)^{\fl{u}_{\pi(\ca{C},i)}} \Big) = \zeta_{\fl{z},\fl{z}} \cdot \Big(\sum_{j \in [\ell]}\prod_{i \in \ca{C}}  (\fl{v}^{(j)}_i)^{\fl{z}_{\pi(\ca{C},i)}} \Big). 
 \end{align}
Using the computed $\widehat{U}^{\fl{z}}$'s , we can compute an estimate $\widehat{V}^{\fl{z}}$ of $\sum_{j \in [\ell]}\prod_{i \in \ca{C}}  (\fl{v}^{(j)}_i)^{\fl{z}_{\pi(\ca{C},i)}}$ for all $\fl{z}\in (\bb{Z}^{+})^{|\ca{C}|}$ satisfying $\fl{z} \le 2\ell \fl{1}_{|\ca{C}|}$. Let us denote the error in estimation by $\epsilon_{\fl{z}}$ i.e. we have $\left|\widehat{V}^{\fl{z}}- \sum_{j \in [\ell]}\prod_{i \in \ca{C}}  (\fl{v}^{(j)}_i)^{\fl{z}_{\pi(\ca{C},i)}}\right| \le \epsilon_{\fl{z}}$. Now, we prove the following claim.  
\begin{claim}
We must have

\begin{align*}
    \epsilon_{\fl{z}} \le \frac{\ell \Phi_{\fl{z}}}{\zeta_{\fl{z},\fl{z}}}+\sum_{\fl{u}<\fl{z}} \sum_{\s{M}\in \ca{M}(\fl{z},\fl{u}) }\frac{\ell \Phi_{\fl{u}} \prod_{(\fl{r},\fl{s})\in \s{M}}\zeta_{\fl{r},\fl{s}}}{\prod_{\fl{r}\in \ca{T}(\s{M})}\zeta_{\fl{r},\fl{r}}}
\end{align*}

\end{claim}

\begin{proof}
We will prove this lemma by induction. Let $\fl{e}_i$ be the standard basis vector having a non-zero entry at the $i^{\s{th}}$ index and is zero everywhere else. For the base case, we have from Lemma \ref{lem:poly} that $\ell\bb{E}\fl{x}_i=\beta_{1,2}\sum_{j \in [\ell]}\fl{v}^j_i+\beta_{1,1}$. 
Therefore, we must have 
\begin{align*}
    &\ell\bb{E}\fl{x}_i-\ell\widehat{U}^{\fl{e}_i}=\beta_{1,2}(\sum_{j \in [\ell]}\fl{v}^j_i-\widehat{U}^{\fl{e}_i}) \\
    &\implies \ell \Phi_{\fl{e}_i} = \beta_{1,2}\epsilon_{\fl{e}_i}. 
\end{align*}
From definition, (recall that $\zeta_{\fl{z},\fl{u}}=\prod_{i \in \ca{C}} \beta_{\fl{z}_{\pi(i)},\fl{u}_{\pi(i)}+1}$), we have $\zeta_{\fl{e}_i,\fl{e}_i}=\beta_{1,2}$ which completes the proof of the base case. Now suppose for all vectors $\fl{u}\in (\bb{Z}^{+})^{\left|\ca{C}\right|}$ satisfying $\fl{u}\le \fl{t}$, the lemma statement is true. Consider another vector $\fl{z}\in (\bb{Z}^{+})^{\left|\ca{C}\right|}$ such that there exists an index $j\in |\ca{C}|$ for which $\fl{z}_j= \fl{t}_j+1$ and $\fl{z}_i= \fl{t}_i$ for all $i\neq j$. From the statement of Lemma \ref{lem:poly}, we know that

\begin{align*}
      \ell\bb{E} \prod_{i \in \ca{C}}\fl{x}_i^{\fl{z}_{\pi(i)}} - \sum_{\fl{u} < \fl{z}}\zeta_{\fl{z},\fl{u}} \cdot \Big(\sum_{j \in [\ell]}\prod_{i \in \ca{C}}  (\fl{v}^j_i)^{\fl{u}_{\pi(i)}} \Big) = \zeta_{\fl{z},\fl{z}} \cdot \Big(\sum_{j \in [\ell]}\prod_{i \in \ca{C}}  (\fl{v}^j_i)^{\fl{z}_{\pi(i)}} \Big). 
 \end{align*}
 
 Hence, we must have 
 \begin{align*}
      &\Big(\ell\bb{E} \prod_{i \in \ca{C}}\fl{x}_i^{\fl{z}_{\pi(i)}} - \ell \widehat{U}^{\fl{z}}\Big)- \Big(\sum_{\fl{u} < \fl{z}}\zeta_{\fl{z},\fl{u}} \cdot \Big(\sum_{j \in [\ell]}\prod_{i \in \ca{C}}  (\fl{v}^j_i)^{\fl{u}_{\pi(i)}} - \widehat{V}^{\fl{u}} \Big)\Big) = \zeta_{\fl{z},\fl{z}} \cdot \Big(\sum_{j \in [\ell]}\prod_{i \in \ca{C}}  (\fl{v}^j_i)^{\fl{z}_{\pi(i)}} - \widehat{V}^{\fl{z}}\Big) \\
     &\implies  \zeta_{\fl{z},\fl{z}} \epsilon_{\fl{z}} \le \ell \Phi_{\fl{z}}+\sum_{\fl{u}<\fl{z}}\zeta_{\fl{z},\fl{u}}\epsilon_{\fl{u}}. 
 \end{align*}
 Now, by using our induction hypothesis, we must have  
 \begin{align*}
     &\zeta_{\fl{z},\fl{z}} \epsilon_{\fl{z}} \le \ell \Phi_{\fl{z}}+\sum_{\fl{u}<\fl{z}}\zeta_{\fl{z},\fl{u}}\Bigg(\frac{\ell \Phi_{\fl{u}}}{\zeta_{\fl{u},\fl{u}}}+\sum_{\fl{v}<\fl{u}} \sum_{\s{M}\in \ca{M}(\fl{u},\fl{v}) }\frac{\ell \Phi_{\fl{v}} \prod_{(\fl{r},\fl{s})\in \s{M}}\zeta_{\fl{r},\fl{s}}}{\prod_{\fl{r}\in \ca{T}(\s{M})}\zeta_{\fl{r},\fl{r}}}\Bigg) \\
     &\implies  \epsilon_{\fl{z}} \le \frac{\ell \Phi_{\fl{z}}}{\zeta_{\fl{z},\fl{z}}}+\sum_{\fl{u}<\fl{z}}\zeta_{\fl{z},\fl{u}}\Bigg(\frac{\ell \Phi_{\fl{u}}}{\zeta_{\fl{z},\fl{z}}\zeta_{\fl{u},\fl{u}}}+\sum_{\fl{v}<\fl{u}} \sum_{\s{M}\in \ca{M}(\fl{u},\fl{v}) }\frac{\ell \Phi_{\fl{v}} \prod_{(\fl{r},\fl{s})\in \s{M}}\zeta_{\fl{r},\fl{s}}}{\zeta_{\fl{z},\fl{z}}\prod_{\fl{r}\in \ca{T}(\s{M})}\zeta_{\fl{r},\fl{r}}}\Bigg) \\
    &\implies  \epsilon_{\fl{z}} \le \frac{\ell \Phi_{\fl{z}}}{\zeta_{\fl{z},\fl{z}}}+\sum_{\fl{u}<\fl{z}} \sum_{\s{M}\in \ca{M}(\fl{z},\fl{u}) }\frac{\ell \Phi_{\fl{u}} \prod_{(\fl{r},\fl{s})\in \s{M}}\zeta_{\fl{r},\fl{s}}}{\prod_{\fl{r}\in \ca{T}(\s{M})}\zeta_{\fl{r},\fl{r}}}.
 \end{align*}
This completes the proof of the claim.
\end{proof}

Hence, for fixed $\Phi_{\fl{z}} = \Phi$ for all $\fl{z} \le 2\ell \fl{1}_{|\ca{C}|}$, we get 
\begin{align*}
    \epsilon_{\fl{z}} \le \Phi\Big(\frac{\ell }{\zeta_{\fl{z},\fl{z}}}+\sum_{\fl{u}<\fl{z}} \sum_{\s{M}\in \ca{M}(\fl{z},\fl{u}) }\frac{\ell  \prod_{(\fl{r},\fl{s})\in \s{M}}\zeta_{\fl{r},\fl{s}}}{\prod_{\fl{r}\in \ca{T}(\s{M})}\zeta_{\fl{r},\fl{r}}}\Big).
\end{align*}
For a fixed $\Phi$, let us write $\epsilon$ to denote the following quantity:
\begin{align*}
   \epsilon \triangleq \max_{ \fl{z} \le 2\ell\fl{1}_{|\ca{C}|}} \Phi\Big(\frac{\ell }{\zeta_{\fl{z},\fl{z}}}+\sum_{\fl{u}<\fl{z}} \sum_{Q\in \ca{Q}(\fl{z},\fl{u}) }\frac{\ell  \prod_{(\fl{r},\fl{s})\in Q}\zeta_{\fl{r},\fl{s}}}{\prod_{\fl{r}\in \ca{T}(Q)}\zeta_{\fl{r},\fl{r}}}\Big)
\end{align*}

Consider a fixed subset of indices $\ca{C}\subseteq [n]$ and a fixed vector $\fl{t}\in (\bb{Z}^{+})^{|\ca{C}|}$.
Using the fact $\max_{\fl{v} \in \ca{V}, i \in [n]} \fl{v}_i^2 \le R^2$, we have that 
\begin{align*}
  \frac{1}{\ell} \sum_{\fl{v}\in \ca{V}} \Big(\prod_{i\in \ca{C}}  \fl{v}_i^{2}\Big)^{p} \le R^{2p|\ca{C}|} \quad \text{and} \quad \s{A}_{\ca{C},t}=\sum_{\substack{\ca{C}'\subseteq [\ell] \\ \left|\ca{C}'\right|=t}} \prod_{\substack{i \in \ca{C}\\ j \in \ca{C}'}} (\fl{v}^{(j)}_i)^{2} \le {\ell \choose t}R^{2(t+|\ca{C}|)} \le 2^{\ell}R^{2(t+|\ca{C}|)}.
\end{align*}

We can compute an estimate $\widehat{\s{A}}_{\ca{C},t}$ of $\s{A}_{\ca{C},t}$ by using $\widehat{V}^{2p\fl{1}_{|\ca{C}|}}$ in the following set of recursive equations
\begin{align*}
    t\widehat{\s{A}}_{\ca{C},t} = \sum_{p=1}^t (-1)^{p+1} \widehat{\s{A}}_{\ca{C},t-p} \widehat{V}^{2p\fl{1}_{|\ca{C}|}}.
\end{align*}

\begin{claim}

\begin{align*}
    \left|\widehat{\s{A}}_{\ca{C},t}-\s{A}_{\ca{C},t}\right| \le \epsilon\Big(3\max(\ell R^{2\ell|\ca{C}|},2^{\ell}R^{\ell+|\ca{C}|})\Big)^{(t-1)}t!  \text{ for all }t\in [\ell].
\end{align*}

\end{claim}

\begin{proof}
We will prove this claim by induction. For the base case i.e. $t=1$, notice that 
\begin{align*}
    \left|\widehat{\s{A}}_{\ca{C},1}-\s{A}_{\ca{C},1}\right| \le \left|\widehat{V}^{2\fl{1}_{|\ca{C}|}}-\sum_{\fl{v}\in \ca{V}} \prod_{i\in \ca{C}}  \fl{v}_i^{2}\right| \le \epsilon.
\end{align*}
Now, suppose for all $t \le k$, the following holds true:
\begin{align*}
    \left|\widehat{\s{A}}_{\ca{C},t}-\s{A}_{\ca{C},t}\right| \le \epsilon\Big(3\max(\ell R^{2\ell|\ca{C}|},2^{\ell}R^{\ell+|\ca{C}|})\Big)^{t-1}t!.
\end{align*}
For ease of notation, let us denote $a=3\max(\ell R^{2\ell|\ca{C}|},2^{\ell}R^{\ell+|\ca{C}|})$. In that case, for $t=k+1$, we must have 
\begin{align*}
    t\left|\widehat{\s{A}}_{\ca{C},t}-\s{A}_{\ca{C},t}\right| &\le \sum_{p \le t} \left|\widehat{\s{A}}_{\ca{C},t-p} \widehat{V}^{2p\fl{1}_{|\ca{C}|}}-\s{A}_{\ca{C},t-p}\cdot  \sum_{\fl{v}\in \ca{V}} \Big(\prod_{i\in \ca{C}}  \fl{v}_i^{2}\Big)^{p}\right| \\
    &\le\left|\widehat{V}^{2\fl{1}_{|\ca{C}|}}- \sum_{\fl{v}\in \ca{V}} \Big(\prod_{i\in \ca{C}}  \fl{v}_i^{2}\Big)^{(k+1)}\right| \\
    &+\sum_{p \le t-1} \left|\epsilon a^{t-2}(t-1)!\cdot \sum_{\fl{v}\in \ca{V}} \Big(\prod_{i\in \ca{C}}  \fl{v}_i^{2}\Big)^{p} +\epsilon \cdot\s{A}_{\ca{C},t-p}+ \epsilon^2a^{t-2}(t-1)!\right| \\
    &\le \epsilon+\sum_{p \le t-1} \left|\epsilon a^{t-2}(t-1)!\ell R^{2\ell|\ca{C}|} +\epsilon \cdot 2^{\ell}R^{2(\ell+|\ca{C}|)}+ \epsilon^2a^{t-2}(t-1)!\right| \\
    &\le \epsilon+\sum_{p \le t-1} \epsilon a^{t-1}(t-1)! \le  \epsilon a^{(t-1)}t!.
\end{align*}
Hence, $\left|\widehat{\s{A}}_{\ca{C},t}-\s{A}_{\ca{C},t}\right|\le \epsilon a^{t-1}t!$ thus proving our claim.
\end{proof}

Hence, to identify $t^{\star}$ correctly, we must have 

\begin{align*}
    &\epsilon\Big(3\max(\ell R^{2\ell|\ca{C}|},2^{\ell}R^{\ell+|\ca{C}|})\Big)^{(\ell-1)}\ell! \le \frac{\delta^{2\ell|\ca{C}|}}{2} \\ 
    &\implies \Phi \le \frac{\delta^{2\ell|\ca{C}|}}{2\Big(3\max(\ell R^{2\ell|\ca{C}|},2^{\ell}R^{\ell+|\ca{C}|})\Big)^{(\ell-1)}\ell!}\Bigg(\max_{\fl{z}\le 2p\fl{1}_{|\ca{C}|}} \frac{1 }{\zeta_{\fl{z,\fl{z}}}}+\sum_{\fl{u}<\fl{z}} \sum_{\s{M}\in \ca{M}(\fl{z},\fl{u}) }\frac{  \prod_{(\fl{r},\fl{s})\in \s{M}}\zeta_{\fl{r},\fl{s}}}{\prod_{\fl{r}\in \ca{T}(\s{M})}\zeta_{\fl{r},\fl{r}}}\Bigg)^{-1}
\end{align*}
where we inserted the definition of $\Phi$.
Therefore, for every vector $\fl{z}\in (\bb{Z}^{+})^{|\ca{C}|}$ satisfying $\fl{z} \le 2\ell \fl{1}_{|\ca{C}|}$,  in order to compute $\widehat{U}^{\fl{z}}$ of $\bb{E}\prod_{i \in \ca{C}}\fl{x}_i^{\fl{z}_{\pi(\ca{C},i)}}$ such that $\left|\widehat{U}^{\fl{z}}-\bb{E}\prod_{i \in \ca{C}}\fl{x}_i^{\fl{z}_{\pi(\ca{C},i)}}\right|\le \Phi$, the number of samples that is sufficient with probability $1-\gamma$ is going to be $$O\Big(\log (\gamma^{-1}(2\ell)^{|\ca{C}|})\frac{\max_{\fl{z}\le 2\ell\fl{1}_{|\ca{C}|}}\bb{E}\prod_{i \in \ca{C}}\fl{x}_i^{2\fl{z}_{\pi(\ca{C},i)}}}{\Phi^{ 2}}\Big).$$
\end{proof}

\begin{thmu}[Restatement of Theorem \ref{thm:md}]

Let $\mathcal{V}$ be a set of $\ell$ unknown vectors in $\bb{R}^n$ satisfying Assumption \ref{assum:weak2}. Let $\ca{F}_m = \ca{Q}_1([n]) \cup \ca{Q}_m(\cup_{\fl{v}\in \ca{V}}\s{supp}(\fl{v}))$ and
\begin{align*}
  &\Phi_{m} = \frac{\delta^{2\ell m}}{2\Big(3\ell\max(R^{2\ell m},2^{\ell}R^{\ell+m})\Big)^{(\ell-1)}\ell!}\Bigg(\max_{ \fl{z}\le 2\ell\fl{1}_{m}} \frac{\ell }{\zeta_{\fl{z,\fl{z}}}}+\sum_{\fl{u}<\fl{z}} \sum_{\s{M}\in \ca{M}(\fl{z},\fl{u}) }\frac{\ell \prod_{(\fl{r},\fl{s})\in \s{M}}\zeta_{\fl{r},\fl{s}}}{\prod_{\fl{r}\in \ca{T}(\s{M})}\zeta_{\fl{r},\fl{r}}}\Bigg)^{-1} \\
  &f_{\ell,\ca{V}}=\max_{\substack{\fl{z}\le 2\ell\fl{1}_{\log \ell+1} \\ \ca{C}\in \ca{F}_{\log \ell+1}}}\frac{\bb{E}\prod_{i \in \ca{C}}\fl{x}_i^{2\fl{z}_{\pi(\ca{C},i)}}}{\Phi_{\log \ell+1}^{ 2}}
\end{align*}
where $f_{\ell,\ca{V}}$ is a constant that is independent of $k$ and $n$ but depends on $\ell$. Then, there exists an algorithm (see Algorithm \ref{algo:md} and \ref{algo:intersection}) that achieves Exact Support Recovery with probability at least $1-\gamma$ using $O\Big(\log (\gamma^{-1}(2\ell)^{\log \ell+1}(n+(\ell k)^{\log \ell+1}))f_{\ell,\ca{V}}\Big)$ samples generated according to $\ca{P}_d$. 

\end{thmu}

\begin{proof}
The proof follows directly from Corollary \ref{coro:prelim1} and Lemma \ref{lem:md_crucial}.
\end{proof}

\begin{corou}[Restatement of Corollary \ref{coro:md}]

Consider the mean estimation problem where $\bb{E}_{\fl{x}\sim \ca{P}_d}[ \fl{x}_i \mid t=j] = \fl{v}^{(j)}_i$.
Let $\mathcal{V}$ be a set of $\ell=O(1)$ unknown vectors in $\bb{R}^n$ satisfying Assumption \ref{assum:weak2} and   $f_{\ell,\ca{V}}$ be as defined in Theorem \ref{thm:general_md}.
Then, there exists an algorithm (see Algorithm \ref{algo:md} and \ref{algo:intersection}) that with probability at least $1-\gamma$, achieves Exact Support Recovery using $O\Big(\log (n\gamma^{-1})\s{poly}(\delta R^{-1})f_{\ell,\ca{V}}\Big)$ samples generated according to $\ca{P}_d$.

\end{corou}

\begin{proof}
We can re-scale the samples (dividing them by $R$) so that Assumption \ref{assum:weak2} will be satisfied with $\delta'=\delta/R$ and $R' \le 1$. Since $\ell$ is a constant, $\Phi_{\log \ell}=O(\s{poly}(\delta R^{-1}))$. Therefore, the corollary follows from Theorem \ref{thm:md}.
\end{proof}

\begin{lemmau}[Restatement of Lemma \ref{lem:partial_md}]
Suppose Assumption \ref{assum:weak2} is true. Let 
\begin{align*}
    &\Phi \triangleq  \max_{\fl{z}\le 2\fl{1}_{|\ca{C}|}} \frac{\delta^{2|\ca{C}|}}{2}\Big(\frac{\ell }{\zeta_{\fl{z},\fl{z}}}+\sum_{\fl{u}<\fl{z}} \sum_{\s{M}\in \ca{M}(\fl{z},\fl{u}) }\frac{\ell  \prod_{(\fl{r},\fl{s})\in \s{M}}\zeta_{\fl{r},\fl{s}}}{\prod_{\fl{r}\in \ca{T}(\s{M})}\zeta_{\fl{r},\fl{r}}}\Big)^{-1} \\ 
    &h_{\ell,\ca{V}} \triangleq \frac{\max_{\fl{z}\le 2\fl{1}_{|\ca{C}|}}\bb{E}\prod_{i \in \ca{C}}\fl{x}_i^{2\fl{z}_{\pi(\ca{C},i)}}}{\Phi^{ 2}}  
\end{align*}
where $h_{\ell,\ca{V}}$ is a constant independent of $k$ and $n$ but depends on $\ell$. There exists an algorithm (see Algorithm \ref{algo:md2}) that can compute if $\left|\bigcap_{i \in \ca{C}}\ca{S}(i)\right|>0$ correctly for each set $\ca{C}\subseteq [n]$ with probability at least $1-\gamma$ using $O(h_{\ell,\ca{V}}\log \gamma^{-1})$ samples generated according to $\ca{P}_d$.
\end{lemmau}

\begin{proof}

For a fixed ordered set $\ca{C}\subseteq [n]$, consider the statistic $\sum_{\fl{v}\in \ca{V}}\prod_{i\in \ca{C}}\fl{v}_i^2$. If $\sum_{\fl{v}\in \ca{V}}\prod_{i\in \ca{C}}\fl{v}_i^2>0$, then $\left|\cap_{i\in \ca{C}}\ca{S}(i)\right|>0$ and otherwise, if $\sum_{\fl{v}\in \ca{V}}\prod_{i\in \ca{C}}\fl{v}_i^2=0$, then $\left|\cap_{i\in \ca{C}}\ca{S}(i)\right|=0$. Hence it suffices to estimate correctly if $\sum_{\fl{v}\in \ca{V}}\prod_{i\in \ca{C}}\fl{v}_i^2>0$ or not. From Lemma \ref{lem:recurse1}, we know that for each set $\ca{C}\subseteq [n]$, we can compute $\sum_{j \in [\ell]}\prod_{i \in \ca{C}} (\fl{v}^{(j)}_i)^{2}$ provided for all $\fl{u}\in (\bb{Z}^{+})^{\left|\ca{C}\right|}$ satisfying $\fl{u}\le 2\fl{1}_{|\ca{C}|}$, the quantity $\bb{E}\prod_{i \in \ca{C}} \fl{x}_i^{\fl{u}_{\pi(\ca{C},i)}}$ is pre-computed.

Suppose, for every vector $\fl{z}\in (\bb{Z}^{+})^{|\ca{C}|}$ satisfying $\fl{z} \le 2\fl{1}_{|\ca{C}|}$, 
we compute an estimate $\widehat{U}^{\fl{z}}$ of $\bb{E}\prod_{i \in \ca{C}}\fl{x}_i^{\fl{z}_{\pi(\ca{C},i)}}$ such that $\left|\widehat{U}^{\fl{z}}-\bb{E}\prod_{i \in \ca{C}}\fl{x}_i^{\fl{z}_{\pi(\ca{C},i)}}\right|\le \Phi$ where $\Phi$ is going to be determined later. Using the computed $\widehat{U}^{\fl{z}}$'s , we can compute an estimate $\widehat{V}^{\fl{z}}$ of $\sum_{j \in [\ell]}\prod_{i \in \ca{C}}  (\fl{v}^{(j)}_i)^{\fl{z}_{\pi(\ca{C},i)}}$ for all $\fl{z}\in (\bb{Z}^{+})^{|\ca{C}|}$ satisfying $\fl{z} \le 2 \fl{1}_{|\ca{C}|}$. As before, let us denote the error in estimation by $\epsilon_{\fl{z}}$ i.e. we have $\left|\widehat{V}^{\fl{z}}- \sum_{j \in [\ell]}\prod_{i \in \ca{C}}  (\fl{v}^{(j)}_i)^{\fl{z}_{\pi(\ca{C},i)}}\right| \le \epsilon_{\fl{z}}$.
Note that we showed in Lemma \ref{lem:recurse2} that for fixed $\Phi$, we get for all $\fl{z}\le 2\fl{1}_{|\ca{C}|}$,
\begin{align*}
    \epsilon_{\fl{z}} \le \Phi\Big(\frac{\ell }{\zeta_{\fl{z},\fl{z}}}+\sum_{\fl{u}<\fl{z}} \sum_{\s{M}\in \ca{M}(\fl{z},\fl{u}) }\frac{\ell  \prod_{(\fl{r},\fl{s})\in \s{M}}\zeta_{\fl{r},\fl{s}}}{\prod_{\fl{r}\in \ca{T}(\s{M})}\zeta_{\fl{r},\fl{r}}}\Big).
\end{align*}
Note that the minimum value of $\sum_{\fl{v}\in \ca{V}}\prod_{i \in \ca{C}}\fl{v}_i^2$ is at least $\delta^{2|\ca{C}|}$ and therefore, it suffices $\epsilon_{\fl{z}}$ to be less than $\delta^{2|\ca{C}|}/2$. Hence, it is sufficient if 
\begin{align*}
    \Phi \le \max_{\fl{z}\le 2\fl{1}_{|\ca{C}|}} \frac{\delta^{2|\ca{C}|}}{2}\Big(\frac{\ell }{\zeta_{\fl{z},\fl{z}}}+\sum_{\fl{u}<\fl{z}} \sum_{\s{M}\in \ca{M}(\fl{z},\fl{u}) }\frac{\ell  \prod_{(\fl{r},\fl{s})\in \s{M}}\zeta_{\fl{r},\fl{s}}}{\prod_{\fl{r}\in \ca{T}(\s{M})}\zeta_{\fl{r},\fl{r}}}\Big)^{-1}.
\end{align*}
Now, we use Lemma \ref{lem:medianofmeans} to complete the proof of the lemma (similar to Lemma \ref{lem:recurse2})

\end{proof}

\begin{thmu}[Restatement of Theorem \ref{thm:general_md}]
Let $\ca{V}$ be a set of unknown vectors in $\bb{R}^n$ satisfying Assumption \ref{assum:weak2}. Let $\ca{F}_m = \ca{Q}_1([n]) \cup \ca{Q}_m(\cup_{\fl{v}\in \ca{V}}\s{supp}(\fl{v}))$ and
\begin{align*}
    & \Phi_{m} = \max_{\fl{z}\le 2\fl{1}_{|\ca{C}|}} \frac{\delta^{2|\ca{C}|}}{2}\Big(\frac{\ell }{\zeta_{\fl{z},\fl{z}}}+\sum_{\fl{u}<\fl{z}} \sum_{\s{M}\in \ca{M}(\fl{z},\fl{u}) }\frac{\ell  \prod_{(\fl{r},\fl{s})\in \s{M}}\zeta_{\fl{r},\fl{s}}}{\prod_{\fl{r}\in \ca{T}(\s{M})}\zeta_{\fl{r},\fl{r}}}\Big)^{-1}\\ 
    &h'_{\ell,\ca{V}} \triangleq \max_{\substack{\fl{z}\le 2\fl{1}_{\ell} \\ \ca{C}\in \ca{F}_{ \ell}}}\frac{\bb{E}\prod_{i \in \ca{C}}\fl{x}_i^{2\fl{z}_{\pi(\ca{C},i)}}}{\Phi_{\ell}^{2}} 
\end{align*}
where $h'_{\ell,\ca{V}}$ is a constant independent of $k$ and $n$ but depends on $\ell$.
 Accordingly, there exists an algorithm (see Algorithm \ref{algo:md2} and \ref{algo:partial_highprob}) that achieves maximal support recovery with probability at least $1-\gamma$ using $O\Big(h'_{\ell,\ca{V}}\log (\gamma^{-1}(n+(\ell k)^{\ell}))\Big)$ samples generated from $\ca{P}_d$.
\end{thmu}

\begin{proof}
The proof follows from Lemma \ref{lem:partial_md} and Corollary \ref{coro:partial2}. 
\end{proof}

\subsection{Mixtures of Linear Classifiers}\label{sec:detailed_mlc}

Recall that in this section, we solve the sparse recovery problem when the observed samples are generated according to $\ca{P}_c$ under Assumption \ref{assum:pos}.








\begin{lemmau}[Restatement of Lemma \ref{lem:mlc_crucial}]
Suppose Assumptions \ref{assum:weak2} and \ref{assum:pos} are true. Let $a=\frac{\sqrt{2(R^2+\sigma^2)}}{\delta}\s{erf}^{-1}\Big(1-\frac{1}{2\ell}\Big)$. There exists an algorithm (see Algorithm \ref{algo:mlc}) that can compute $\left|\bigcup_{i \in \ca{C}}\ca{S}(i)\right|$ for each set $\ca{C}\subseteq [n]$ with probability at least $1-\gamma$ using
 $O\Big( (1-\phi(a))^{-\left|\ca{C}\right|}\ell^2\log \gamma^{-1}\Big)$ i.i.d samples from $\ca{P}_c$. 
\end{lemmau}

\begin{proof}
Without loss of generality, let us assume that all unknown vectors in $\ca{V}$ have positive non-zero entries. for each fixed set $\ca{C}\subseteq [n]$, we will condition on event $\ca{E}_{\ca{C}}$ defined as follows: for all $j\in \ca{C}$, the data-point $\fl{x}$ satisfies $\fl{x}_j > a$ for some suitably chosen $a>0$. Recall that the minimum magnitude of any non-zero entry in an unknown vector in $\ca{V}$ is at least $\delta$. Further condition on the event $\ca{E}_{\fl{v}}$ which is true when a particular unknown vector $\fl{v}$ is being sampled from $\ca{V}$. 
In that case, we show the following claim:

\begin{claim}
\begin{align*}
    &\Pr(y=1 \mid \ca{E}_{\fl{v}},\ca{E}_{\ca{C}}) = \frac{1}{2}  \text{ if } \fl{v}_{\mid \ca{C}} = \fl{0} \\
    &1\ge \Pr(y=1 \mid \ca{E}_{\fl{v}},\ca{E}_{\ca{C}}) \ge\frac{1}{2}+ \frac{1}{2}\cdot\s{erf}\Big(\frac{a\delta}{\sqrt{2(R^2+\sigma^2)}}\Big)  \text{ if } \fl{v}_{\mid \ca{C}} \neq \fl{0}. 
\end{align*}
\end{claim}
\begin{proof}
 In order to see the above equation, note that if $\fl{v}_{\mid \ca{C}}=\fl{0}$, then $\langle \fl{v},\fl{x} \rangle+z \sim \ca{N}(0,||\fl{v}||_2^2+\sigma^2)$ or in other words, conditioning on the event $\ca{E}_{\ca{C}}$ has no effect on the distribution of $y$. On the other hand,  
if $\fl{v}_{\mid \ca{C}}\neq 0$, conditioning on the event $\ca{E}_{\ca{C}}$ modifies the distribution of $y$. Consider an index $j\in \s{supp}(\fl{v})\cap \ca{C}$. Since $\fl{v}_j\fl{x}_j \ge a\delta$, we must have $\langle \fl{v}_{\mid \ca{C}}, \fl{x}_{\mid \ca{C}}\rangle \ge a\delta$ using Assumption \ref{assum:pos}. Therefore, the probability that $y=1$ must be at least 
$\Pr(\langle \fl{v}_{\mid [n]\setminus\ca{C}}, \fl{x}_{\mid [n]\setminus\ca{C}}\rangle+z \ge -a\delta)$. Using the fact that $\langle \fl{v}_{\mid [n]\setminus\ca{C}}, \fl{x}_{\mid [n]\setminus\ca{C}}\rangle+z \sim \ca{N}(0,\nu^2+\sigma^2)$ (where $\nu \le R$) and the property of error function ($\Pr_{u \sim \ca{N}(0,\sigma^2)}(|u|\le a) =\s{erf}(a/\sqrt{2}\sigma)$), we prove the claim.
\end{proof}


Hence we must have 

\begin{align*}
   \frac{1}{2}+\frac{\left|\cup_{i \in \ca{C}}\ca{S}(i)\right|}{2\ell} \ge  \Pr(y=1\mid \ca{E}_{\ca{C}}) \ge \frac{1}{2}+\frac{\left|\cup_{i \in \ca{C}}\ca{S}(i)\right|}{2\ell}\s{erf}\Big(\frac{a\delta}{\sqrt{2(R^2+\sigma^2)}}\Big)
\end{align*}


We choose $a$ such that $\s{erf}\Big(\frac{a\delta}{\sqrt{2(R^2+\sigma^2)}}\Big) \ge 1-\frac{1}{2\ell}$ in which case, we must have 

\begin{align*}
    &\frac{1}{2}\Big(1+\frac{1}{\ell}\left|\bigcup_{i \in \ca{C}}\ca{S}(i)\right| \Big)-\frac{1}{4\ell^2}\cdot\left|\bigcup_{i \in \ca{C}}\ca{S}(i)\right| \le \Pr(y=1\mid \ca{E}_{\ca{C}}) \le  \frac{1}{2}\Big(1+\frac{1}{\ell}\left|\bigcup_{i \in \ca{C}}\ca{S}(i)\right| \Big) 
\end{align*}

Clearly, for each value of $\left|\bigcup_{i \in \ca{C}}\ca{S}(i)\right|\in \{0,1,\dots,\ell\}$, the interval in which $\Pr(y=1\mid \ca{E}_{\ca{C}})$ lies  
are disjoint and each interval is separated by at least $1/4\ell$. Hence, if we are able to estimate $\Pr(y=1\mid \ca{E}_{\ca{C}})$ up to an additive factor of $1/8\ell$, then we can uniquely (and correctly) decode the value of $\left|\bigcup_{i \in \ca{C}}\ca{S}(i)\right|$. By using Chernoff  bound, with $O(\ell^2\log \gamma^{-1})$ samples satisfying the event $\ca{E}_{\ca{C}}$, we can estimate $\Pr(y=1\mid \ca{E}_{\ca{C}})$ (See Step 2 in Algorithm \ref{algo:mlc} for the estimator) with probability at least $1-\gamma/2$. From our previous analysis, we chose $a=\frac{\sqrt{2(R^2+\sigma^2)}}{\delta}\s{erf}^{-1}\Big(1-\frac{1}{2\ell}\Big)$.
The probability that for a sample $(\fl{x},y)\sim \ca{P}_c$, the event $\ca{E}_{\ca{C}}$ is true is exactly $O\Big((1-\phi(a))^{\left|\ca{C}\right|}\Big)$. Therefore, with $ (1-\phi(a))^{-\left|\ca{C}\right|}\ell^2\log \gamma^{-1}\Big)$ samples, we will have $O(\ell^2\log\gamma^{-1})$ samples satisfying the event $\ca{E}_{\ca{C}}$ with probability at least $1-\gamma/2$. Hence, this allows us to recover $\left|\bigcup_{i \in \ca{C}}\ca{S}(i)\right|$ with probability at least $1-\gamma$.
\end{proof}

\begin{thmu}[Restatement of Theorem \ref{thm:mlc}]
Let $\mathcal{V}$ be a set of $\ell$ unknown vectors in $\bb{R}^n$ satisfying Assumptions \ref{assum:weak2} and \ref{assum:pos}. Let $a=\frac{\sqrt{2(R^2+\sigma^2)}}{\delta}\s{erf}^{-1}\Big(1-\frac{1}{2\ell}\Big)$. Then, there exists an algorithm (see Algorithm \ref{algo:mlc} and \ref{algo:intersection}) that achieves Exact Support Recovery with probability at least $1-\gamma$ using $O\Big( (1-\phi(a))^{-(\log \ell+1)}\ell^2\log (\gamma^{-1}(n+(\ell k)^{\log \ell+1}))\Big)$ samples generated according to $\ca{P}_c$.
\end{thmu}

\begin{proof}
The proof follows directly from Lemma \ref{lem:mlc_crucial} and Corollary \ref{coro:prelim1}.
\end{proof}

\subsection{Mixtures of Linear Regressions}\label{sec:detailed_mlr}

\subsubsection{Unknown Binary Vectors}

\begin{lemmau}[Restatement of Lemma \ref{lem:binary_mlr}]
If the unknown vectors in the set $\ca{V}$ are all binary i.e. $\fl{v}^{(1)},\fl{v}^{(2)},\dots,\fl{v}^{(\ell)}\in \{0,1\}^n$, then, with  probability at least $1-\gamma$, for each set $\ca{C}\subseteq [n]$, there exists an algorithm (see Algorithm \ref{algo:mlr_binary}) that can compute $\left|\bigcap_{i \in \ca{C}}\ca{S}(i)\right|$  using $O(\ell^2(k+\sigma^2)^{|C|/2}(\log n)^{2|C|}\log \gamma^{-1})$ i.i.d samples from $\ca{P}_r$.
\end{lemmau}

\begin{proof}
Consider the random variable $y^{\left|\ca{C}\right|} \cdot\Big(\prod_{i\in \ca{C}}\fl{x}_i\Big)$ where $(\fl{x},y)\sim \ca{P}_{r}$. Clearly, we can write $y=\langle \fl{v},\fl{x}\rangle+\zeta$ where $\zeta \sim \ca{N}(0,\sigma^2)$ and $\fl{v}$ is uniformly sampled from the set of unknown vectors $\ca{V}$. Therefore, we must have
\begin{align*}
    &\bb{E}_{(\fl{x},y)\sim \ca{P}_r} y^{\left|\ca{C}\right|} \cdot\Big(\prod_{i\in \ca{C}}\fl{x}_i\Big) = \bb{E}_{\fl{x},\zeta}  \ell^{-1}\sum_{\fl{v}\in \ca{V}}\Big(\prod_{i\in \ca{C}}\fl{x}_i\Big)\cdot\Big(\langle \fl{v},\fl{x}\rangle+\zeta\Big)^{\left|\ca{C}\right|} \\
    &\bb{E} y^{\left|\ca{C}\right|} \cdot\Big(\prod_{i\in \ca{C}}\fl{x}_i\Big) = \frac{1}{\ell}\sum_{\fl{v}\in \ca{V}}\Big(\prod_{i\in \ca{C}} \bb{E}_{\fl{x}}\fl{x}_i^2\cdot\fl{v}_i\Big)= \frac{\left|\bigcap_{i \in \ca{C}}\ca{S}(i)\right|}{\ell}. 
\end{align*}

This is because in the expansion of $(\langle \fl{v},\fl{x}\rangle+\zeta)^{\left|\ca{C}\right|}$, the only monomial containing $\fl{x}_i$ for all $i\in \ca{C}$ is $\prod_{i \in \ca{C}}\fl{v}_i\fl{x}_i$. For any other monomial, the product with $\prod_{i\in \ca{C}}\fl{x}_i$ will contain some $\fl{x}_j,j \in \ca{C}$ such that the degree of $\fl{x}_j$ in the monomial is 1; the expectation of this monomial goes to zero as all the $\fl{x}_i$'s are independent. Since $\bb{E}\fl{x}_i^2=1$ for all $i\in [n]$ and $\prod_{i\in \ca{C}}\fl{v}_i$ is \texttt{1} iff $\fl{v}_i=1$ for all $i \in \ca{C}$ (and \texttt{0} otherwise), we obtain the desired equations. We estimate  $\left|\bigcap_{i \in \ca{C}}\ca{S}(i)\right|$ by computing the following sample average \begin{align*}
    \frac{\ell}{m} \cdot \sum_{j=1}^{m}  \Big(y^{(j)}\Big)^{|\ca{C}|}\Big(\prod_{i\in \ca{C}}\fl{x}^{(j)}_i\Big).
\end{align*}
From definition for $(\fl{x},y) \sim \ca{P}_r$, we must have $y\sim \ell^{-1}\sum_{\fl{v}\in \ca{V}}\ca{N}(0,\left|\left|\fl{v}\right|\right|_0^2+\sigma^2)$.
Therefore, we must have $\bb{E}y^2 \le k+\sigma^2$ since $\fl{v}\in \{0,1\}^n, \left|\left|\fl{v}\right|\right|_0 \le k$ for all $\fl{v}\in \ca{V}$. By using Gaussian concentration inequalities, we must have $\Pr(|y|>t) \le \exp(-t^2/2(k+\sigma^2))$. Therefore, with probability $1-n^{-10}$, we have $|y|<20\sqrt{k+\sigma^2}\log n$. Similarly, with probability $1-n^{-10}$, $|\fl{x}_i|$ is bounded from above by $20\log n$. We take a union bound over all $|C|+1$ random variables and all $m$ samples to infer that $\Big(y^{(j)}\Big)^{|\ca{C}|}\Big(\prod_{i\in \ca{C}}\fl{x}^{(j)}_i\Big)$ is bounded within a ball of radius $O((k+\sigma^2)^{|C|/2}(\log n)^{2|C|})$ with probability at least $1-O(m|C|n^{-10})$. Subsequently, we use Hoeffding's inequality (see Lemma \ref{lem:Hoeffding}) to say that  
\begin{align*}
    \Pr\Big(\left|\frac{1}{m} \cdot \sum_{j=1}^{m}  \Big(y^{(j)}\Big)^{|\ca{C}|}\Big(\prod_{i\in \ca{C}}\fl{x}^{(j)}_i\Big)-\frac{\left|\bigcap_{i \in \ca{C}}\ca{S}(i)\right|}{\ell}\right|\ge \frac{1}{2\ell}\Big) \le \exp\Big(-\Omega\Big(\frac{m}{\ell^2(k+\sigma^2)^{|C|/2}(\log n)^{2|C|}}\Big)\Big). 
\end{align*}
Hence, with $m=O(\ell^2(k+\sigma^2)^{|C|/2}(\log n)^{2|C|}\log \gamma^{-1})$ samples, we can estimate 
$\left|\bigcap_{i \in \ca{C}}\ca{S}(i)\right|$ exactly with  probability at least $1-\gamma$.
\end{proof}


We can now show the following result:



\begin{thmu}[Restatement of Theorem \ref{thm:binary}]
Let $\mathcal{V}$ be a set of $\ell$ unknown binary vectors in $\{0,1\}^n$. Then, with probability at least $1-\gamma$, there exists an algorithm (see Algorithms \ref{algo:mlr_binary} and \ref{algo:partial_highprob}) that achieves Exact Support Recovery with 
\begin{align*}
    O\left(\ell^{2} (k+\sigma^2)^{(\log \ell+1)/2}(\log n)^{2(\log \ell+1)}  \log((n+(\ell k)^{\log \ell})\gamma^{-1})\right)
\end{align*}
 samples generated according to $\ca{P}_r$. 
\end{thmu}

\begin{proof}
The proof follows directly from Lemma \ref{lem:binary_mlr} and Corollary \ref{coro:prelim1}.
\end{proof}

\subsubsection{Separability Assumption on Parameters}


 

Below, we show that if Assumption \ref{assum:second} is satisfied, then we can recover the support of the unknown vectors. We start with the following theorem:

\begin{thmu}[Restatement of Theorem \ref{thm:general}]
Suppose the following conditions are satisfied:
\begin{enumerate}
    \item All unknown vectors in  $\ca{V}$ are bounded within a ball of radius $R$ i.e. $ \left|\left|\fl{v}^{(i)}\right|\right|_2 \le R \text{ for all } i \in [\ell]$.  
    
    \item Assumption \ref{assum:second} is satisfied by the set of unknown vectors $\ca{V}$.
\end{enumerate}
 Accordingly, with probability at least $1-\gamma$, there exists an algorithm (see Algorithms \ref{algo:mlr_2} and \ref{algo:partial_highprob}) that achieves maximal support recovery using 
 \begin{align*}
     O(\ell^2(R^2+\sigma^2)^{\ell/2}(\log n)^{2\ell}\log((n+(\ell k)^{\ell})\gamma^{-1})/\alpha_{\ell}^2)
 \end{align*}
  samples from $\ca{P}_r$.
\end{thmu}

\begin{proof}

Again, we look at the random variable $y^{\left|\ca{C}\right|} \cdot\Big(\prod_{i\in \ca{C}}\fl{x}_i\Big)$ where $(\fl{x},y)\sim \ca{P}_{r}$ and therefore, we must have
\begin{align*}
    &y^{\left|\ca{C}\right|} \cdot\Big(\prod_{i\in \ca{C}}\fl{x}_i\Big) =  \frac{1}{\ell}\sum_{\fl{v}\in \ca{V}}\Big(\prod_{i\in \ca{C}}\fl{x}_i\Big)\cdot\Big(\langle \fl{v},\fl{x}\rangle+\zeta\Big)^{\left|\ca{C}\right|} \\
    &\bb{E} y^{\left|\ca{C}\right|} \cdot\Big(\prod_{i\in \ca{C}}\fl{x}_i\Big) = \frac{1}{\ell}\sum_{\fl{v}\in \ca{V}}\Big(\prod_{i\in \ca{C}} \bb{E}\fl{x}_i^2\cdot\fl{v}_i\Big)= \frac{1}{\ell}\sum_{\fl{v}\in \ca{V}} \Big(\prod_{j \in \ca{C}} \fl{v}_j\Big).
\end{align*}

Notice that $\bb{E} y^{\left|\ca{C}\right|} \cdot\Big(\prod_{i\in \ca{C}}\fl{x}_i\Big)=0$ if $\left|\bigcap_{i \in \ca{C}}\ca{S}(i)\right|=0$ and $\left|\bb{E} y^{\left|\ca{C}\right|} \cdot\Big(\prod_{i\in \ca{C}}\fl{x}_i\Big)\right| \ge \alpha_{\left|\ca{C}\right|}/\ell$ otherwise (by using Assumption \ref{assum:second}).
 We estimate  $\bb{E} y^{\left|\ca{C}\right|} \cdot\Big(\prod_{i\in \ca{C}}\fl{x}_i\Big)$ by computing the following sample average \begin{align*}
    \frac{\ell}{m} \cdot \sum_{j=1}^{m}  \Big(y^{(j)}\Big)^{|\ca{C}|}\Big(\prod_{i\in \ca{C}}\fl{x}^{(j)}_i\Big).
\end{align*}
From the definition of $\ca{P}_r$, we must have $y\sim \ell^{-1}\sum_{\fl{v}\in \ca{V}}\ca{N}(0,\left|\left|\fl{v}\right|\right|_2^2+\sigma^2)$. 
Therefore, we have that $\bb{E}y^2 \le R^2+\sigma^2$ since $ \left|\left|\fl{v}\right|\right|_2 \le R$ for all $\fl{v}\in \ca{V}$ from the statement of the Theorem. By using Gaussian concentration inequalities, we must have $\Pr(|y|>t) \le \exp(-t^2/2(R^2+\sigma^2))$. Therefore, with probability $1-n^{-10}$, we have $|y|<20\sqrt{R^2+\sigma^2}\log n$. Similarly, with probability $1-n^{-10}$, $|\fl{x}_i|$ is bounded from above by $20\log n$. We take a union bound over all $|C|+1$ random variables and all $m$ samples to infer that $\Big(y^{(j)}\Big)^{|\ca{C}|}\Big(\prod_{i\in \ca{C}}\fl{x}^{(j)}_i\Big)$ is bounded within a ball of radius $O((R^2+\sigma^2)^{|C|/2}(\log n)^{2|C|})$ with probability at least $1-O(m|C|n^{-10})$. Subsequently, we use Hoeffding's inequality (see Lemma \ref{lem:Hoeffding})  to say that  
\begin{align*}
    \Pr\Big(\left|\frac{1}{m} \cdot \sum_{j=1}^{m}  y^{(j)\;|\ca{C}|}\Big(\prod_{i\in \ca{C}}\fl{x}^{(j)}_i\Big)-\frac{1}{\ell}\sum_{\fl{v}\in \ca{V}} \Big(\prod_{j \in \ca{C}} \fl{v}_j\Big)\right|\ge \frac{\alpha_{\left|\ca{C}\right|}}{2\ell}\Big) \le \exp\Big(-\Omega\Big(\frac{m\alpha_{\left|\ca{C}\right|}^{2}}{\ell^2(R^2+\sigma^2)^{|C|/2}(\log n)^{2|C|}}\Big)\Big). 
\end{align*}
Hence, with $m=O(\ell^2(R^2+\sigma^2)^{|C|/2}(\log n)^{2|C|}\log\gamma^{-1}/\alpha_{\left|\ca{C}\right|}^2)$ samples, we can estimate if 
$\left|\bigcap_{i \in \ca{C}}\ca{S}(i)\right|>0$ or not correctly with probability at least $1-\gamma$. The proof now follows directly from using Corollary \ref{coro:partial2}.
\end{proof}

\begin{corou}[Restatement of Corollary \ref{coro:non-negative}]
Consider a set of $\ell$ unknown vectors $\ca{V}$ that satisfies Assumptions \ref{assum:weak2} and furthermore, every non-zero entry in all the unknown vectors is positive ($\fl{v}_i \ge 0$ for all $i\in[n],\fl{v}\in \ca{V}$). In that case, Assumption \ref{assum:second} is satisfied with $\alpha_{|\ca{C}|} \ge \delta^{|\ca{C}|}$.  Accordingly, there exists an algorithm that achieves maximal support recovery  with probability at least $1-\gamma$ using
 \begin{align*}
     O(\ell^2(R^2+\sigma^2)^{\ell/2}(\log n/\delta)^{2\ell}\log((n+(\ell k)^{\ell})\gamma^{-1}))
 \end{align*}
samples from $\ca{P}_r$.
\end{corou}

\begin{proof}
Note that when all the unknown vectors in set $\ca{V}$ are non-negative, it must happen that for each set $\ca{C} \subseteq [n]$, $\left|\sum_{\fl{v}\in \ca{V}} \Big(\prod_{j \in \ca{C}} \fl{v}_j\Big)\right| \ge \alpha_{|\ca{C}|}$ is a sum of positive terms (provided is it non-zero) each of which is at least $\delta^{\left|\ca{C}\right|}$. Therefore, it must happen that $\alpha_{\left|\ca{C}\right|}\ge \delta^{\left|\ca{C}\right|}$. The above argument also holds true when all the unknown vectors in set $\ca{V}$ are non-positive. We can directly use Theorem \ref{thm:general} to arrive at the statement of the corollary.
\end{proof}

\begin{corou}[Restatement of Corollary \ref{coro:gaussian}]
If all non-zero entries in the set of unknown vectors $\ca{V}$ are sampled i.i.d according to $\ca{N}(0,\nu^2)$, then with probability $1-\eta$, Assumption \ref{assum:second} is satisfied with $\alpha_{|\ca{C}|} \ge \delta_{|\ca{C}|}^{|\ca{C}|}$ where 
\begin{align*}
    \delta_{{|\ca{C}|}}=\Big(\sqrt{\frac{\pi}{8}} \frac{\nu\eta}{\ell|\ca{C}|(\ell k)^{|\ca{C}|}}\Big).
\end{align*}
 Conditioned on this event, there exists an Algorithm that achieves maximal support recovery with probability at least $1-\gamma$ using
 \begin{align*}
     O(\ell^2(R^2+\sigma^2)^{\ell/2}(\log n)^{2\ell}\log((n+(\ell k)^{\ell})\gamma^{-1})/\delta_{\ell}^2)
 \end{align*}
 samples from $\ca{P}_r$.
\end{corou}

\begin{proof}
For a fixed set $\ca{C}\subseteq [n]$, consider the random variable $\sum_{\fl{v}\in \ca{V}} \Big(\prod_{j \in \ca{C}} \fl{v}_j\Big)$. For each vector $\fl{v}\in \ca{V}$ such that $\prod_{j \in \ca{C}}\fl{v}_j \neq 0$, we denote the minimum index $i\in \ca{C}$ such that $\fl{v}_{i} \neq 0$ by $i^{\star}$ and therefore $\fl{v}_{i^{\star}} \sim \ca{N}(0,\nu^2)$. Now, for each $\fl{v}\in \ca{V}$, let us condition on a fixed realization of non-zero indices of $\fl{v}$ in $\ca{C}$ other than $i^{\star}$. 
Let $\ca{V}_{\ca{C}}\subseteq \ca{V}$ be the set of vectors such that $\prod_{j \in \ca{C}} \fl{v}_j \neq 0$. Therefore, we must have
\begin{align}
    \sum_{\fl{v}\in \ca{V}} \Big(\prod_{j \in \ca{C}} \fl{v}_j\Big) \mid \fl{v}_j \text{ for all } j \in \ca{C}\setminus i^{\star}, \fl{v} \in \ca{V}_{\ca{C}} \sim \ca{N}\Big(0,\nu^2\sum_{\fl{v}\in \ca{V}_{\ca{C}}}\prod_{j \in \ca{C}\setminus i^{\star}} \fl{v}_j^2 \Big). 
\end{align}
Therefore, conditioned on $\fl{v}_j \text{ for all } j \in \ca{C}\setminus i^{\star}, \fl{v} \in \ca{V}_{\ca{C}}$, by standard Gaussian anti-concentration inequality (see Lemma \ref{lem:gaanconc}), we must have with probability $1-\rho$, 
\begin{align}
    \left|\sum_{\fl{v}\in \ca{V}} \Big(\prod_{j \in \ca{C}} \fl{v}_j\Big) \right| \ge  \sqrt{\frac{\pi}{8}} \rho\nu \sqrt{\sum_{\fl{v}\in \ca{V}_{\ca{C}}}\prod_{j \in \ca{C}\setminus i^{\star}} \fl{v}_j^2}.
\end{align}

for each vector $\fl{v}\in \ca{V}_{\ca{C}}$, we must have with probability at least $1-(|\ca{C}|-1)\rho$ that
\begin{align}\label{eq:cond}
    \left|\prod_{j \in \ca{C}\setminus i^{\star}} \fl{v}_j\right| \ge \Big(\sqrt{\frac{\pi}{8}}\rho \nu\Big)^{(|\ca{C}|-1)}. 
\end{align}
By taking a union bound, we can conclude that with probability at least $1-\ell\rho$, we must have
\begin{align*}
    \left|\sum_{\fl{v}\in \ca{V}}\Big(\prod_{j \in \ca{C}} \fl{v}_j\Big)\right| \ge \Big(\sqrt{\frac{\pi}{8}}\rho \nu\Big)^{|\ca{C}|} 
\end{align*}
since there exists at least one vector $\fl{v}\in \ca{V}_{\ca{C}}$ such that equation \ref{eq:cond} holds true for $\fl{v}$. Next, after taking another union bound over all subsets of size $|\ca{C}|$ restricted to the union of support (at most $(\ell k)^{|\ca{C}|}$ of them), we have that with probability $1-|\ca{C}|(\ell k)^{|\ca{C}|}\rho$,
\begin{align*}
    \left|\sum_{\fl{v}\in \ca{V}} \Big(\prod_{j \in \ca{C}} \fl{v}_j\Big)\right| \ge  \Big(\sqrt{\frac{\pi}{8}} \rho\nu\Big)^{|\ca{C}|}.
\end{align*}
Subsequently, we have with probability at least $1-\eta/\ell$
\begin{align*}
    \left|\sum_{\fl{v}\in \ca{V}} \Big(\prod_{j \in \ca{C}} \fl{v}_j\Big) \right| \ge  \Big(\sqrt{\frac{\pi}{8}} \frac{\nu\eta}{\ell|\ca{C}|(\ell k)^{|\ca{C}|}}\Big)^{|\ca{C}|}.
\end{align*}
After taking a final union bound over $|\ca{C}|\le \ell$ and subsequently using Theorem \ref{thm:general}, we complete the proof of the corollary.
\end{proof}

    


\subsubsection{General Case}\label{app:general_mlr_new}

In this section, we will solve the general MLR problem with only the following mild assumptions on the unknown parameters. 

\begin{assumption}[Restatement of \ref{assum:weak}]
We assume that the minimum gap $\Delta$ of a set of unknown vectors $\ca{V}$
\begin{align*}
    \Delta \triangleq \min_{\substack{\fl{v},\fl{v}'\in \ca{V} \\ \left|\left|\fl{v}\right|\right|_2 \neq \left|\left|\fl{v}'\right|\right|_2}} \left|\left|\left|\fl{v}\right|\right|_2-\left|\left|\fl{v}'\right|\right|_2\right|,
\end{align*}
defined to be the smallest gap between the norms of any two unknown vectors with different norms is known.
\end{assumption}

We will begin by stating the following definition and subsequent theorem from \cite{moitra2010settling}.

\begin{defn}
Given a mixture of $k$ Gaussians $F=\sum_{i=1}^{k}w_i\ca{N}(\mu_i,\nu_i^2)$ and another mixture of $k' \le k$ Gaussians, $\hat{F}=\sum_{i=1}^{k'}\widehat{w}_i\ca{N}(\widehat{\mu}_i,\widehat{\nu}_i^2)$, we will call $\hat{F}$ an $\epsilon$-subdivision of $F$ if there exists an onto function $\pi:[k]\rightarrow[k']$ such that 
\begin{enumerate}
    \item $\left|\sum_{i:\pi(i)=j}w_i-\hat{w}_j\right| \le \epsilon $ for all $j \in [k']$.
    \item $\left|\mu_i-\hat{\mu}_{\pi(i)}\right|+\left|\nu_i^2-\widehat{\nu}^2_{\pi(i)}\right| \le \epsilon$ for all $i \in [k]$.
\end{enumerate}
\end{defn}

\begin{thm}[Moitra and Valiant, 2010]\label{thm:moitra}
Suppose we are given access to independent samples from $F=\sum_{i=1}^{k}w_i\ca{N}(\mu_i,\nu_i^2)$ with mean $0$ and variance in the interval $[1/2,2]$,where $w_i \ge \epsilon$.  There exists a Univariate  Algorithm  that,  for  any  fixed known $k$,  has a sample complexity of at most $\s{poly}(\epsilon^{-1},\gamma^{-1})$ samples  and  with  probability  at least $1-\gamma$ will output an $\epsilon$-subdivision of $F$.

\end{thm}

Next, we prove a lemma showing that with a few samples, it is possible to estimate the second moment of the mixture; subsequently we can normalize to obtain another mixture whose second moment is between $1/2$ and $2$.

\begin{lemma}\label{lem:isotropic}
 Consider a mixture of Gaussians $F=\ell^{-1}\sum_{i\in [\ell]}\ca{N}(0,\nu_i^2)$ such that $L\le \nu_i\le U$ for all $i\in [\ell]$. With probability $1-\gamma$, we can compute $V$ with $O((UL^{-1})^2\log \gamma^{-1})$ samples such that $ 0.5\le\bb{E}_{x\sim F}x^2V^{-1} \le 2$. 
\end{lemma}

\begin{proof}
Clearly, $L^2 \le \bb{E}_{x\sim F}x^2= \ell^{-1}\sum_i \nu_i^2 \le U^2$. From $m$ samples $x_1,x_2,\dots,x_m \sim \ca{F}$, we compute $V=m^{-1}\sum_i x_i^2$. Now, by using Hoeffding's inequality, we must have
\begin{align*}
    \Pr\Big(\left|V-\bb{E}x^2 \right| \ge \frac{\bb{E}x^2}{4}\Big) \le \exp\Big(-\frac{25mL^2}{16U^2}\Big). 
\end{align*}
Therefore, with $m=16U^2\log \gamma^{-1}/25L^2$ samples, we obtain that $0.5\bb{E}x^2\le V\le 2\bb{E}x^2$ thus proving the lemma. 
\end{proof}

\begin{coro}\label{coro:isotropic}
Suppose we are given access to independent samples from $F=\ell^{-1}\sum_{i}\ca{N}(0,\nu_i^2)$ such that $L\le \nu_i\le U$ for all $i\in [\ell]$.  There exists an  Algorithm  that with  probability  at least $1-\delta$ will output an $\epsilon$-subdivision of $F$ using $\s{poly(\epsilon^{-1},\gamma^{-1},L^{-1},U)}$ samples.
\end{coro}

\begin{proof}
The proof follows from using Lemma \ref{lem:isotropic} and Theorem \ref{thm:moitra}.
\end{proof}




Now, we will describe an algorithm to compute $\left|\cup_{i \in \ca{C}}\ca{S}_{\ca{V}}(i)\right|$ for any fixed set $\ca{C}\subseteq [n]$. 

\begin{lemma}\label{lem:mlr_general_crucial1}
 For $(\fl{x},y)\sim \ca{P}_r$ and a fixed $\fl{a}\in \bb{R}^{\left|\ca{C}\right|}$, the random variable $y+\langle \fl{a},\fl{x}_{\mid \ca{C}} \rangle$ is distributed according to the following mixture of Gaussians:
 \begin{align*}
     y+\langle \fl{a},\fl{x}_{\mid \ca{C}} \rangle \sim \frac{1}{\ell} \sum_{\fl{v}\in \ca{V}}\ca{N}\Big(0,\left|\left|\fl{v}\right|\right|_2^2+\left|\left|\fl{a}\right|\right|_2^2+2\langle \fl{a},\fl{v}_{\mid \ca{C}} \rangle+\sigma^2\Big)
 \end{align*}
 
\end{lemma}

\begin{proof}
For some $\fl{v}\in \ca{V}$, let us condition on the event $\ca{E}_{\fl{v}}$ that is true when $\fl{v}$ is selected to generate the sample $(\fl{x},y)$. Recall that each element of $\fl{x}$ is generated independently according to $\ca{N}(0,1)$ and $y\mid \fl{x},\ca{E}_{\fl{v}}\sim \langle \fl{x},\fl{v} \rangle+Z$ where $Z\sim \ca{N}(0,\sigma^2)$. Therefore, for a fixed vector $\fl{a}$, we must have that
\begin{align*}
    y+\langle \fl{a},\fl{x}_{\mid \ca{C}} \rangle \mid \ca{E}_{\fl{v}} \sim \ca{N}(0,\left|\left|\fl{v}\right|\right|_2^2+\left|\left|\fl{a}\right|\right|_2^2+2\langle \fl{a},\fl{v}_{\mid \ca{C}} \rangle+\sigma^2).
\end{align*}
Since the unknown vector $\fl{v} \in \ca{V}$ is uniformly sampled at random, we obtain the statement of the lemma.
\end{proof}

Let us introduce some notations for ease of exposition. 
We will define $\ca{T}_{\fl{a}}\triangleq \{\left|\left|\fl{v}\right|\right|_2^2+\left|\left|\fl{a}\right|\right|_2^2+2\langle \fl{a},\fl{v}_{\mid \ca{C}} \rangle+\sigma^2\}_{\fl{v}\in \ca{V}}$ to be the set of distinct component variances in the distribution of $ y+\langle \fl{a},\fl{x}_{\mid \ca{C}} \rangle$ for a fixed vector $\fl{a}$. Also, for a random variable $Z \sim \sum_{i}w_i \ca{N}(0,\nu_i^2)$ distributed according to a mixture of Gaussians with distinct components, for $t\in \{\nu_i^2\}_i$, we will use $w_{Z,t}$ to denote the weight of the Gaussian component with variance $t$ in the distribution of $Z$. Finally, for any $u>0$ we will also define $\ca{\widetilde{T}}_{u}\triangleq \{\left|\left|\fl{v}\right|\right|_2^2+u^2+\sigma^2\}_{\fl{v}\in \ca{V}}$ to be the set of distinct component variances added with $u^2$.

\begin{lemma}\label{lem:interest}
Fix $\gamma>0,\left|\ca{C}\right| \subseteq [n]$ and $\alpha< \Big(\frac{\delta}{2R\sqrt{\log \gamma^{-1}}}\min\Big(\Delta ,\frac{1}{2}\Big)\Big)$. For a fixed vector $\fl{b}\in \bb{R}^{\left|\ca{C}\right|}$ which is generated by independently sampling every entry from $\ca{N}(0,\alpha^2)$, let $\ca{T} = \ca{T}_{\fl{b}} \cap \ca{\widetilde{\ca{T}}}_{\left|\left|\fl{b}\right|\right|_2}$ and let $Z_{\fl{b}}$ denote the random variable $y+\langle \fl{b}, \fl{x}_{\mid \ca{C}} \rangle$. In that case, we must have with probability at least $1-2\ell \gamma$,
\begin{align*}
    &\sum_{t \in \ca{T}} w_{Z_{\fl{b}},t} = 1- \frac{\left|\bigcup_{i \in \ca{C}}\ca{S}_{\ca{V}}(i)\right|}{\ell} \\
& \left| \left|\left|\fl{v}'\right|\right|_2^2+\left|\left|\fl{b}\right|\right|_2^2- (\left|\left|\fl{v}\right|\right|_2+\left|\left|\fl{b}\right|\right|_2)^2 \right| =\Omega \Big(\delta\min(\Delta,\alpha\gamma)\Big) \text{ for all } \fl{v}'\in \ca{V} \quad \text{ if } \fl{v}_{\mid \ca{C}} \neq 0.
\end{align*}
\end{lemma}

\begin{proof}
Fix any unknown vector $\fl{v} \in \ca{V}$ such that $\fl{v}_{\mid \ca{C}} = \fl{0}$ i.e. all indices in $\fl{v}$ constrained to the set $\ca{C}$ are zero. In that case, we have 
\begin{align*}
    \left|\left|\fl{v}\right|\right|_2^2+\left|\left|\fl{b}\right|\right|_2^2+2\langle \fl{b},\fl{v}_{\mid \ca{C}} \rangle = \left|\left|\fl{v}\right|\right|_2^2+\left|\left|\fl{b}\right|\right|_2^2.
\end{align*}
On the other hand, for any unknown vector $\fl{v} \in \ca{V}$ such that not all indices in $\fl{v}_{\mid \ca{C}}$ are zero, we must have \begin{align*}
    \left|\left|\fl{v}\right|\right|_2^2+\left|\left|\fl{b}\right|\right|_2^2+2\langle \fl{b},\fl{v}_{\mid \ca{C}} \rangle = \left|\left|\fl{v}\right|\right|_2^2+\left|\left|\fl{b}\right|\right|_2^2+Z
\end{align*} 
where $Z$ is a random variable distributed according to $\ca{N}(0,4\alpha^2\left|\left|\fl{v}_{\mid \ca{C}}\right|\right|_2^2)$. Since $\fl{v}_{\mid \ca{C}}$ has at least one non-zero entry, the variance of $Z$ is at least $4\alpha^2 \delta^2$ and at most $4\alpha^2 R^2$. Therefore, by standard Gaussian concentration inequalities, we must have for any fixed $0<\widetilde{\gamma},\widehat{\gamma} <1$ (see Lemma \ref{lem:gaconc} and Lemma \ref{lem:gaanconc}),  
\begin{align*}
    &\left|Z\right| \ge \sqrt{\frac{\pi}{2}}\alpha\delta\widetilde{\gamma} \quad \text{ with probability } 1-\widetilde{\gamma} \\ 
    & \left|Z\right| \le \alpha R \sqrt{\log \widehat{\gamma}^{-1}} \le \frac{\Delta \delta}{2} \text{ with probability } 1-\widehat{\gamma}. 
\end{align*}

We will fix $\widetilde{\gamma},\widehat{\gamma}$ to be the chosen $\gamma$ in the statement of the lemma. 
Now, we prove the following claim:
\begin{claim}\label{claim:first}
\begin{align*}
\sum_{t \in \ca{T}} w_{Z_{\fl{b}},t} = 1- \frac{\left|\bigcup_{i \in \ca{C}}\ca{S}_{\ca{V}}(i)\right|}{\ell}
\end{align*}
\end{claim}

\begin{proof}
Note that
\begin{align*}
    \min_{\substack{\fl{v},\fl{v}'\in \ca{V} \\ \left|\left|\fl{v}\right|\right|_2 \neq \left|\left|\fl{v}'\right|\right|_2}} \left|\left|\left|\fl{v}\right|\right|_2^2-\left|\left|\fl{v}'\right|\right|_2^2\right| \ge \min_{\substack{\fl{v},\fl{v}'\in \ca{V} \\ \left|\left|\fl{v}\right|\right|_2 \neq \left|\left|\fl{v}'\right|\right|_2}} \left|\left|\left|\fl{v}\right|\right|_2-\left|\left|\fl{v}'\right|\right|_2\right|\left|\left|\left|\fl{v}\right|\right|_2+\left|\left|\fl{v}'\right|\right|_2\right| \ge 
    2\Delta \delta,
\end{align*}
First of all, we must have with probability $1-\gamma$ that $\Delta\delta/2 \ge \left|\langle \fl{b}, \fl{v}_{\mid \ca{C}} \rangle \right| \ge \sqrt{\frac{\pi}{2}}\alpha\delta\gamma$ and therefore, for any $\fl{v}' \in \ca{V}$ such that $\left|\left|\fl{v}'\right|\right|_2=\left|\left|\fl{v}\right|\right|_2$,
\begin{align*}
     \left| \left|\left|\fl{v}'\right|\right|_2^2+\left|\left|\fl{b}\right|\right|_2^2- (\left|\left|\fl{v}\right|\right|_2+\left|\left|\fl{b}\right|\right|_2)^2 \right| \ge \sqrt{2\pi}\alpha\delta\gamma \text{ if } \fl{v}'_{\mid \ca{C}} \neq \fl{0} 
\end{align*}
Moreover, we also have that for any $\fl{v}' \in \ca{V}$ such that $\left|\left|\fl{v}'\right|\right|_2 \neq \left|\left|\fl{v}\right|\right|_2$,
\begin{align*}
     \left| \left|\left|\fl{v}'\right|\right|_2^2+\left|\left|\fl{b}\right|\right|_2^2- (\left|\left|\fl{v}\right|\right|_2+\left|\left|\fl{b}\right|\right|_2)^2 \right| \ge \left|\left|\left|\fl{v}'\right|\right|_2^2-\left|\left|\fl{v}\right|\right|_2^2\right|- 2\left|\langle \fl{b}, \fl{v}_{\mid \ca{C}} \rangle \right| \ge \Delta \delta   \text{ if } \fl{v}'_{\mid \ca{C}} \neq \fl{0} 
\end{align*}

Therefore, an unknown vector $\fl{v}$ satisfies
\begin{align*}
\left|\left|\fl{v}\right|\right|_2^2+\left|\left|\fl{b}\right|\right|_2^2+2\langle \fl{b},\fl{v}_{\mid \ca{C}} \rangle = \left|\left|\fl{v}\right|\right|_2^2+\left|\left|\fl{b}\right|\right|_2^2
\end{align*} 
and belongs to $\ca{T} =  \ca{T}_{\fl{b}}\cap \ca{\widetilde{T}}_{\left|\left|\fl{b}\right|\right|_2}$ if and only if $\fl{v}_{\mid \ca{C}} = \fl{0}$. Therefore, the sum of weights corresponding to the components in $\ca{T}$ must account for those having a zero entry at all the indices in $\ca{C}$. 

\end{proof}

On the other hand, for all unknown vectors $\fl{v}\in \ca{V}$ such that $\fl{v}_{\mid \ca{C}} \neq  \fl{0}$, then with probability at least $1-2\ell\gamma$ (by taking a union bound over at most $\ell$ unknown vectors), we must have 
\begin{align*}
  \left| \left|\left|\fl{v}'\right|\right|_2^2+\left|\left|\fl{b}\right|\right|_2^2- (\left|\left|\fl{v}\right|\right|_2+\left|\left|\fl{b}\right|\right|_2)^2 \right| =\Omega \Big(\delta\min(\Delta,\alpha\gamma)\Big) \text{ for all } \fl{v}'\in \ca{V}.
\end{align*}
This completes the proof of the lemma.

\end{proof}

We will now describe a simple technique to compute the union of support of the unknown vectors.

\begin{lemma}\label{lem:union}
 Suppose $\left|\cup_{\fl{v}\in \ca{V}}\s{supp}(\fl{v})\right|\le n/2$. In that case, we can compute $\cup_{\fl{v}\in \ca{V}}\s{supp}(\fl{v})$ correctly using $O(\ell^2(R^2+\sigma^2)(\log n)^{3}/\delta^2)$ samples with probability at least $1-1/\s{poly}(n)$.
\end{lemma}

\begin{proof}
For each $i\in [n]$, suppose we want to test whether $i \in \cup_{\fl{v}\in \ca{V}}\s{supp}(\fl{v})$ or not.
Consider the random variable $y^2\fl{x}_i^2$ when $(\fl{x},y)\sim \ca{P}_r$. Notice that 
\begin{align*}
    \bb{E}y^2\fl{x}_i^2 = \frac{1}{\ell}\sum_{\fl{v}\in \ca{V}} \bb{E}y^2\fl{x}_i^2 \mid \fl{v} = \frac{1}{\ell}\sum_{\fl{v}\in \ca{V}} \Big(\sum_{j \in [n]} \fl{v}_j^2 +2\fl{v}_i^2\Big)
    \begin{cases}
    &=\frac{1}{\ell}\sum_{\fl{v}\in \ca{V}} \left|\left|\fl{v}\right|\right|_2^2 \text{ if } \left|\ca{S}_{\ca{V}}(i)\right|=0 \\
    &\ge  \frac{1}{\ell}\sum_{\fl{v}\in \ca{V}} \left|\left|\fl{v}\right|\right|_2^2+ \frac{2\delta^2}{\ell} \text{ if } \left|\ca{S}_{\ca{V}}(i)\right| \neq 0
    \end{cases}
\end{align*}
where the final inequality follows from the fact that the magnitude of any non-zero entry of any unknown vector must be at least $\delta$. For simplicity of notation, we will denote $A=\frac{1}{\ell}\sum_{\fl{v}\in \ca{V}} \left|\left|\fl{v}\right|\right|_2^2$  to be average norm of the unknown vectors. We will estimate $\bb{E}y^2\fl{x}_i^2$ by computing the following sample average \begin{align*}
    \frac{\ell}{m} \cdot \sum_{j=1}^{m}  \Big(y^{(j)}\fl{x}_i^{(j)}\Big)^2.
\end{align*}
From the definition of $\ca{P}_r$, we must have $y\sim \ca{N}(0,\zeta^2+\sigma^2), |\zeta|\le R$ since $\fl{v}\in \{0,1\}^n, \left|\left|\fl{v}\right|\right|_2 \le R$ for all $\fl{v}\in \ca{V}$. By using Gaussian concentration inequalities, we must have $\Pr(|y|>t) \le \exp(-t^2/2(R^2+\sigma^2))$. Therefore, with probability $1-n^{-10}$, we have $|y|<20\sqrt{R^2+\sigma^2}\log n$. Similarly, with probability $1-n^{-10}$, $|\fl{x}_i|$ is bounded from above by $20\log n$. a Subsequently, we use Hoeffding's inequality  to say that
\begin{align*}
    \Pr\Big(\left|\frac{\ell}{m} \cdot \sum_{j=1}^{m}  \Big(y^{(j)}\fl{x}_i^{(j)}\Big)^2-\bb{E}y^2\fl{x}_i^2\right|\ge \frac{\delta^2{2}}{2\ell}\Big) \le \exp\Big(-\Omega\Big(\frac{m\delta^{2}}{\ell^2(R^2+\sigma^2)(\log n)^{2}}\Big)\Big). 
\end{align*}
Hence, with $m=O(\ell^2(R^2+\sigma^2)(\log n)^{3}/\delta^2)$ samples, we can estimate if 
$\left|\bigcap_{i \in \ca{C}}\ca{S}_{\ca{V}}(i)\right|>0$ or not correctly with probability at least $1-1/\s{poly}(n)$. We can take a union bound over all the $n$ indices to estimate $\bb{E}y^2\fl{x}_i^2$ correctly within an additive error of $\delta^2/2\ell$ for all $i\in [n]$. We will cluster all the indices such that a pair of distinct indices $u,v \in [n]$ are in the same group if 
\begin{align*}
    \left|\frac{\ell}{m} \cdot \sum_{j=1}^{m}  \Big(y^{(j)}\fl{x}_u^{(j)}\Big)^2-\frac{\ell}{m} \cdot \sum_{j=1}^{m}  \Big(y^{(j)}\fl{x}_v^{(j)}\Big)^2\right| \le \frac{\delta^2}{\ell}. 
\end{align*}
Clearly, any two indices $u,v \in [n]$ that satisfy $|\ca{S}_{\ca{V}}(u)|=|\ca{S}_{\ca{V}}(v)|=0$ must belong to the same cluster. Since the size of the union of the support is at most $n/2$, the largest cluster must correspond to the indices where the entry is zero in all the unknown vectors. Subsequently, all those indices that do not belong to the largest cluster (after the clustering step) must belong to $\cap_{\fl{v}\in \ca{V}}\s{supp}(\fl{v})$. Furthermore, no index $i \in [n]$ such that $|\ca{S}_{\ca{V}}(i)|\neq 0$ can belong to the largest cluster. This complete the proof of the lemma.
\end{proof}

Next, we show that once the union of support of the unknown vectors $\cup_{\fl{v}\in \ca{V}}\s{supp}(\fl{v})$ is known, we can compute $\left|\cup_{i \in \ca{C}} \ca{S}_{\ca{V}}(i)\right|$ for all sets $\ca{C} \subseteq \cup_{\fl{v}\in \ca{V}}\s{supp}(\fl{v})$ by restricting ourselves to the union of support. 

\begin{lemma}[Restatement of Lemma \ref{lem:set}]
Suppose $\cup_{\fl{v}\in \ca{V}}\s{supp}(\fl{v})$ is known.
We can compute $\left|\bigcup_{i \in \ca{C}}\ca{S}_{\ca{V}}(i)\right|$ for all sets $\ca{C} \subseteq \bigcup_{\fl{v}\in \ca{V}}\s{supp}(\fl{v}), \left|\ca{C}\right|=s$ using $O\Big(\s{poly}(\Delta^{-1},\delta^{-1},\gamma^{-1},\sigma,R)\cdot\exp(s\log (\ell k))\Big)$ samples with probability at least $1-\delta$.
\end{lemma}

\begin{proof}
Fix $0<\gamma<1,\gamma'=(\ell k)^{-(s+1)}k\gamma$ and set $\alpha=\frac{\Delta \delta}{10R\sqrt{\log \gamma'^{-1}}}$.
Since we know the union of support of the unknown vectors, henceforth we will restrict ourselves to this set.
Consider any fixed set $\ca{C}\subseteq \cup_{\fl{v}\in \ca{V}}\s{supp}(\fl{v})$ such that $\left|\ca{C}\right|=s$. Recall that we defined $\ca{Q}([n])$ as the power set of $[n]$. We now define the function $g:\ca{Q}([n]) \rightarrow \bb{R}^n$ such that \begin{align*}
    g(\ca{C})_i 
    \begin{cases}
     \sim \ca{N}(0,\alpha^2) \text{ if } i\in\ca{C} \\
     = 0 \text{ otherwise }.
    \end{cases} 
\end{align*} 
i.e. the $i^{\s{th}}$ entry of the vector $g(\ca{C})$ is sampled from a Gaussian with zero mean and variance $\alpha^2$ if $i \in \ca{C}$ and zero otherwise. Now, let us consider a fixed set $\ca{C}\subseteq [n]$ satisfying $|\ca{C}|=s$.  
For any $\fl{v}\in \ca{V}$, note that if $\fl{v}_{\mid \ca{C}}=\fl{0}$, then we must have $\langle g(\ca{C}),\fl{v} \rangle = 0$. Invoking Lemma \ref{lem:interest} with $\gamma'= (\ell k)^{-(s+1)}\gamma$ and our choice of $\alpha$, we can say with probability at least $1-\ell\gamma'$ that, 
\begin{align*}
    &\sum_{t \in \ca{T}} w_{Z_{g(\ca{C})},t} = 1- \frac{\left|\bigcup_{i \in \ca{C}}\ca{S}_{\ca{V}}(i)\right|}{\ell} \\
    & \left| \left|\left|\fl{v}'\right|\right|_2^2+\left|\left|g(\ca{C})\right|\right|_2^2- (\left|\left|\fl{v}\right|\right|_2+\left|\left|g(\ca{C})\right|\right|_2)^2 \right| =\Omega \Big(\delta\min(\Delta,\alpha\gamma')\Big) \text{ for all } \fl{v}'\in \ca{V} \quad \text{ if } \fl{v}_{\mid \ca{C}} \neq 0.
\end{align*}
 where $\ca{T} =  \ca{T}_{g(\ca{C})}\cap \ca{\widetilde{T}}_{\left|\left|g(\ca{C})\right|\right|_2}$. Recall that the random variables $y$ and $y+\langle g(\ca{C}),\fl{x}_{\mid \ca{C}} \rangle$ are generated according to the distributions
\begin{align*}
    &y \sim  \frac{1}{\ell}\sum_{\fl{v}\in \ca{V}} \ca{N}(0,\left|\left|\fl{v}\right|\right|_2^2) \\ 
    &\text{and} \quad y+\langle g(\ca{C}),\fl{x} \rangle  \sim \frac{1}{\ell}\sum_{\fl{v}\in \ca{V}}\ca{N}(0,\left|\left|\fl{v}\right|\right|_2^2+\left|\left|g(\ca{C})\right|\right|_2^2+2\langle g(\ca{C}),\fl{v}_{\mid \ca{C}} \rangle)
\end{align*}
respectively. For some suitable numerical constant $1>c>0$, let $\epsilon=c\min(\ell^{-3},\delta\min(\Delta,\alpha\gamma'))$ such that
\begin{align}\label{eq:gapn}
    \left| \left|\left|\fl{v}'\right|\right|_2^2+\left|\left|g(\ca{C})\right|\right|_2^2- (\left|\left|\fl{v}\right|\right|_2+\left|\left|g(\ca{C})\right|\right|_2)^2 \right| \ge \epsilon \text{ for all } \fl{v}'\in \ca{V} \quad \text{ if } \fl{v}_{\mid \ca{C}} \neq 0.
\end{align}

\begin{claim}
We must have with probability $1-\gamma'$
\begin{align*}
3\delta^2/4 \le \left|\left|\fl{v}\right|\right|_2^2+\left|\left|g(\ca{C})\right|\right|_2^2+2\langle g(\ca{C}),\fl{v}_{\mid \ca{C}} \rangle +\sigma^2\le 2R^2+\alpha^2 s^2 \log \gamma'^{-1} \quad \text{ for all }  \fl{v}\in \ca{V}     
\end{align*}
\end{claim}

\begin{proof}
We know $ \delta^2\le \left|\left|\fl{v}\right|\right|_2^2 \le R^2$ for all $\fl{v}\in \ca{V}$. Again, since each entry of the $s$-sparse vector $g(\ca{C})$ is generated according to $\ca{N}(0,\alpha^2)$, we must have $\left|\left|g(\ca{C})\right|\right|_2^2 \le \alpha^2 s^2 \log \gamma'^{-1}$ with probability $1-\gamma'$. Finally, as described in Claim \ref{claim:first}, we must have $|\langle g(\ca{C}),\fl{v}_{\mid \ca{C}} \rangle| \le \alpha R \sqrt{\log \gamma'^{-1}} \le \delta/4$ with probability $1-\gamma'$. Therefore, $\left|\left|\fl{v}\right|\right|_2^2+\left|\left|g(\ca{C})\right|\right|_2^2+2\langle g(\ca{C}),\fl{v}_{\mid \ca{C}} \rangle \ge \left|\left|\fl{v}\right|\right|_2^2+2\langle g(\ca{C}),\fl{v}_{\mid \ca{C}} \rangle \ge 3\delta^2/4$.  
\end{proof}

Therefore, we invoke Corollary \ref{coro:isotropic} to compute $\epsilon/4$-subdivisions of the Gaussian Mixtures according to which the random variables $y$ and $y+\langle g(\ca{C}),\fl{x} \rangle$ are distributed with probability at least $1-\gamma'$. From Corollary \ref{coro:isotropic} and definitions of $\epsilon,\alpha,\gamma'$, this is possible by using $O\Big(\s{poly}(\Delta^{-1},\delta^{-1},\gamma^{-1},R,\sigma)\exp(s\log(\ell k)) \Big)$ samples with probability at least $1-\gamma'$.  
Let $\ca{\widehat{T}}_0 \triangleq \{(w_i,\nu_i^2)\}_{i \in [\ell_1]}$ and $\ca{\widehat{T}}_{g(\ca{C})} \triangleq \{(w_i,\nu_i^2)\}_{i \in [\ell_2]}$ denote the set of distinct tuples of estimated weights and variances obtained from estimating the parameters of Gaussian mixtures corresponding to the random variables $y$ and $y+\langle g(\ca{C}),\fl{x} \rangle$ respectively; here $\ell_1,\ell_2<\ell$ is the number of distinct components in the estimated mixtures. For a tuple $(w,\sigma^2)$ in $ \ca{\widehat{T}}_{g(\ca{C})}$, we align it with a tuple $(w',\sigma'^2)$ in $ \ca{\widehat{T}}_{0}$ if $\left|\sigma^2-||g(\ca{C})||_2^2-\sigma'^2\right|\le \epsilon/2$; note that this is equivalent to aligning 
with the tuple $(w',\sigma'^2+||g(\ca{C})||_2^2)$ in $\ca{\widetilde{T}}_{\left|\left|g(\ca{C})\right|\right|_2}$.
In this process, we put the constraint that the same tuple in $\ca{\widehat{T}}_{0}$ will not be aligned with two distinct tuples in $\ca{\widehat{T}}_{g(\ca{C})}$. From equation \ref{eq:gapn} and by using triangle inequality, we can see that such a pair of tuples can be aligned only if either of the following events occur:
\begin{enumerate}
    \item Both $((w,\sigma^2))\in \ca{\widehat{T}}_{g(\ca{C})}$ and $(w',\sigma'^2)\in \ca{\widehat{T}}_{0}$ correspond to the same unknown vector $\fl{v}$ where $\fl{v}_{\mid \ca{C}}=\fl{0}$. In this case, we have 
    \begin{align*}
        \left|\sigma^2-||g(\ca{C})||_2^2-\sigma'^2\right|\le \left|\sigma^2-\left|\left|\fl{v}\right|\right|_2^2-||g(\ca{C})||_2^2\right|+\left|\sigma'^2-\left|\left|\fl{v}\right|\right|_2^2\right| \le  \frac{\epsilon}{4}+\frac{\epsilon}{4} \le \frac{\epsilon}{2}.
    \end{align*}
    \item $((w,\sigma^2))\in \ca{\widehat{T}}_{g(\ca{C})}$ and $(w',\sigma'^2)\in \ca{\widehat{T}}_{0}$ might correspond to distinct unknown vectors  $\fl{v},\fl{v}' \in \ca{V}$ where $\left|\left|\fl{v}\right|\right|_2=\left|\left|\fl{v}'\right|\right|_2$ and $\fl{v}'_{\mid \ca{C}}=\fl{0}$. In this case, we again have 
    \begin{align*}
        \left|\sigma^2-||g(\ca{C})||_2^2-\sigma'^2\right|\le \left|\sigma^2-\left|\left|\fl{v}\right|\right|_2^2-||g(\ca{C})||_2^2\right|+\left|\sigma'^2-\left|\left|\fl{v}'\right|\right|_2^2\right| \le \frac{\epsilon}{4}+\frac{\epsilon}{4} \le \frac{\epsilon}{2}.
    \end{align*}
\end{enumerate}
Hence, it must happen that every tuple in $\ca{\widehat{T}}_{0}$ for which the corresponding unknown vector has a disjoint support from that of $\ca{C}$  will be aligned with some tuple in $\ca{\widehat{T}}_{g(\ca{C})}$. We denote the set of aligned tuples in $ \ca{\widehat{T}}_{0}$ by $\widehat{\ca{T}}$ and hence, by our analysis, we have $\widehat{\ca{T}}=\ca{T}=\ca{T}_{g(\ca{C})}\cap \ca{\widetilde{T}}_{\left|\left|g(\ca{C})\right|\right|_2}$. Therefore, we have  
\begin{align*}
    \left|\sum_{t \in \ca{T}} w_{Z_{g(\ca{C})},t}-\sum_{(w,\sigma^2) \in \ca{\widehat{T}}} w\right| \le \frac{\ell \epsilon}{4} \le \frac{1}{4\ell^2}.
\end{align*}
where we used the fact that $\epsilon<\ell^{-3}$.
Note that the quantity $|\cup_{i \in \ca{C}}\ca{S}_{\ca{V}}(i)|$ is integral and therefore if we can estimate it to an additive error of less than $1/2$, then we can round up the estimate to its nearest integer to compute the quantity correctly. Therefore we can estimate $|\cup_{i \in \ca{C}}\ca{S}_{\ca{V}}(i)|$ correctly by computing $\s{round}(\ell(1-\sum_{(w,\sigma^2) \in \ca{\widehat{T}}} w_{Z_{g(\ca{C})},t}))$. 
Hence, using $O\Big(\s{poly}(\Delta^{-1},\delta^{-1},\gamma^{-1},R)\cdot\exp(s\log (\ell k))\Big)$ samples, with probability $1-\gamma$, we can compute $\left|\bigcup_{i \in \s{supp}(\fl{b})} \ca{S}_{\ca{V}}(i) \right|$ for every set $\ca{C}\subseteq [n],\left|\ca{C}\right|=s$  by taking a union bound over all the $O((\ell k)^s)$
events corresponding to successfully computing $\epsilon/4$-subdivisions of the mixtures $y+\langle g(\ca{C}),\fl{x} \rangle$. 
 This completes the proof of the lemma. 
\end{proof}

In Lemma \ref{lem:union}, we needed to assume that the $\left|\cup_{\fl{v}\in \ca{V}}\s{supp}(\fl{v})\right|\le n/2$. In the corollary below, we describe how we can remove this assumption at the cost of higher sample complexity.

\begin{coro}\label{coro:union}
 We can compute $\cup_{\fl{v}\in \ca{V}}\s{supp}(\fl{v})$ correctly using 
 $O\Big(\s{poly}(\Delta^{-1},\delta^{-1},\sigma,\gamma^{-1},R, \log n)\Big)$ samples with probability at least $1-\gamma$.
\end{coro}

\begin{proof}
Suppose we complete the clustering of indices in $[n]$ as described in the proof of Lemma \ref{lem:union}. If the size of the largest cluster is smaller than $n/2$, then we can correctly conclude that $\left|\cup_{\fl{v}\in \ca{V}}\s{supp}(\fl{v})\right|\ge n/2$. On the other hand, if the size of the largest cluster is larger than $n/2$, we just need to test if all the indices in that cluster correspond to indices which do not belong to the union of support. Since we know (from Lemma \ref{lem:union}) that all indices that do not belong to union of support belong to the same cluster, it suffices to test if one particular index in the largest cluster belongs to the union of support or not. From Lemma \ref{lem:set}, for any such particular index $i\in [n]$ in the largest cluster, we can estimate  $\left|\ca{S}(i)\right|$ correctly with $O\Big(\s{poly}(\Delta^{-1},\delta^{-1},\gamma^{-1},\sigma,R)\Big)$ samples (note that we do not need to take a union bound in this case) with probability at least $1-\gamma$. If the test returns  $\left|\ca{S}(i)\right|=0$, then we can conclude that indices belong to the largest cluster if and only if they are outside the union of support. On the other hand, if the test return $\left|\ca{S}(i)\right|>0$, then again we conclude that $\left|\cup_{\fl{v}\in \ca{V}}\s{supp}(\fl{v})\right|\le n/2$.
\end{proof}

Now, we are ready to state our main result:

\begin{thm}[Restatement of Theorem \ref{thm:mlr_general_new}]
Let $\mathcal{V}$ be a set of $\ell$ unknown vectors in $\bb{R}^n$ satisfying Assumptions \ref{assum:weak2} and \ref{assum:weak}. Let $\alpha= \frac{1}{2}\Big(\frac{\delta}{2R\sqrt{\log \gamma^{-1}}}\min\Big(\Delta ,\frac{1}{2}\Big)\Big)$ and $\epsilon=\ell^{-3}/2$.
Then, with probability at least $1-\gamma$, there exists an algorithm (see Algorithm~\ref{algo:mlr_general} with parameters $\alpha,\epsilon>0$ and Algorithm \ref{algo:intersection}) that achieves Exact Support Recovery with  $O\Big(\s{poly}(\Delta^{-1},\delta^{-1},\gamma^{-1},\sigma,R,\log n)\cdot\exp(\log \ell\log (\ell k))\Big)$ samples generated according to $\ca{P}_r$.
\end{thm}

\begin{proof}
The proof follows directly from combining Lemmas \ref{lem:set}, \ref{lem:union} and Corollary \ref{coro:prelim1}.
\end{proof}

\section{Missing Proofs from Section \ref{sec:prelims}}\label{app:prelim}

\begin{proof}[Proof of Lemma \ref{lem:prelim1} when $|\cup_{i \in \ca{C}} \ca{S}(i)|$ is provided]

 Suppose we are given $|\cup_{i \in \ca{C}} \ca{S}(i)|$ for all sets $\ca{C}\subseteq [n]$ satisfying $|\ca{C}|\le s$.  Notice that the set $\cap_{i \in \ca{C}} \ca{S}(i)$ is equivalent to the set $\s{occ}(C,\fl{1}_{|C|})$ or the number of unknown vectors in $\ca{V}$ whose restriction to the indices in $\ca{C}$ is the all one vector and in particular, $\s{occ}((i),1)=\ca{S}(i)$. Note that for each family of $t$ sets $\ca{A}_1,\ca{A}_2,\dots,\ca{A}_t$, we must have
\begin{align*}
    \left|\bigcup_{i=1}^{t} \ca{A}_{i}\right|= \sum_{u=1}^{t}(-1)^{u+1} \sum_{1 \le i_1 < i_2 <\dots <i_u \le t} \left|\bigcap_{b=1}^{u} \ca{A}_{i_b}\right|.
\end{align*}

We now show using induction on $s$ that the quantities $\left\{ \left|\bigcup_{i \in \ca{S}} \s{occ}((i), 1)\right| ~\forall~\ca{T} \subseteq [n], |\ca{T}| \le s \right\}$
are sufficient to compute $|\s{occ}(C, \fl{a})|$ for all subsets $C$ of indices of size at most $s$, and any binary vector $\fl{a} \in \{0,1\}^{\le s}$.

\textit{Base case ($t=1$):} 

 The base case follows  since we can infer $|\s{occ}((i), 0)| = \ell - |\s{occ}((i), 1)|$  from $|\s{occ}((i), 1)|$  for all $ i \in [n]$. 

\textit{Inductive Step:}
Let us assume that the statement is true for $r < s$ i.e., we can compute $|\s{occ}(\ca{C},\fl{a})|$ for all subsets $\ca{C}$ satisfying $|\ca{C}|\le r$ and any binary vector $\fl{a}\in \{0,1\}^{\le r}$ from the quantities $\left\{ \left|\bigcup_{i \in \ca{S}} \s{occ}((i), 1)\right| ~\forall~\ca{T} \subseteq [n], |\ca{T}| \le r \right\}$ provided as input. Now, we prove that the statement is true for $r+1$ under the induction hypothesis. Note that we can also rewrite $\s{occ}(\ca{C},\fl{a})$ for each set $\ca{C}\subseteq [n], \fl{a}\in \{0,1\}^{|\ca{C}|}$ as \begin{align*}
    \s{occ}(\ca{C},\fl{a}) = \bigcap_{j \in \ca{C}'}\ca{S}(j) \bigcap_{j \in \ca{C}\setminus \ca{C}' }\ca{S}(j)^c
\end{align*}
where $\ca{C}'\subseteq \ca{C}$ corresponds to the indices in $\ca{C}$ for which the entries in $\fl{a}$ is \texttt{1}. Fix any set $i_1,i_2,\dots,i_{r+1} \in [n]$. Then we can compute $\left|\bigcap_{b=1}^{r+1} \ca{S}(i_b)\right|$ using the following equation:
\begin{align*}
    (-1)^{r+3}\left|\bigcap_{b=1}^{r+1} \ca{S}(i_b)\right| 
    = \sum_{u=1}^{r}(-1)^{u+1} \sum_{\substack{j_1,j_2,\dots,j_u \in \{i_1,i_2,\dots,i_{r+1}\}\\ j_1 < j_2 <\dots<j_u  }} \left|\bigcap_{b=1}^{u} \ca{S}(j_b)\right|-\left|\bigcup_{b=1}^{r+1} \ca{S}(i_b)\right|.
\end{align*}

Finally for each proper subset $\ca{Y} \subset \{i_1,i_2,\dots,i_{r+1}\}$, we can compute $\left|\bigcap_{ i_b \not \in \ca{Y}} \ca{S}(i_b) \bigcap_{i_b \in \ca{Y}} \ca{S}(i_b)^{c}\right|$ using the following set of equations:

\begin{align*}
    \left|\bigcap_{i_b \not \in \ca{Y}} \ca{S}(i_b) \bigcap_{i_b \in \ca{Y}} \ca{S}(i_b)^{c}\right| &=     \left|\bigcap_{i_b \not \in \ca{Y}} \ca{S}(i_b) \bigcap \Big(\bigcup_{i_b \in \ca{Y}} \ca{S}(i_b) \Big)^{c} \right| \\
    &= \left|\bigcap_{i_b \not \in \ca{Y}} \ca{S}(i_b) \right| - \left| \bigcap_{i_b \not \in \ca{Y}} \ca{S}(i_b) \bigcap \Big(\bigcup_{i_b \in \ca{Y}} \ca{S}(i_b) \Big) \right| \\
    &=\left|\bigcap_{i_b \not \in \ca{Y}} \ca{S}(i_b) \right| - \left|\bigcup_{i_b \in \ca{Y}} \Big(\bigcap_{i_b \not \in \ca{Y}} \ca{S}(i_b) \bigcap \ca{S}(i_b) \Big) \right|. 
\end{align*}
The first term is already pre-computed and the second term is again a union of intersection of sets. for each $j_b \in \ca{Y}$, let us define $\ca{H}(j_b) := \bigcap_{i_b \not \in \ca{Y}} \ca{S}(i_b) \bigcap \ca{S}(j_b)$. Therefore we have 
\begin{align*}
    \left|\bigcup_{j_b \in \ca{Y}} \ca{H}(j_b)\right|= \sum_{u=1}^{\left|\ca{Y}\right|}(-1)^{u+1} \sum_{\substack{j_1,j_2,\dots,j_u \in \ca{Y}\\ j_1 < j_2 <\dots<j_u  }} \left|\bigcap_{b=1}^{u} \ca{H}(j_b)\right|.
\end{align*}
We can compute $\left|\bigcup_{j_b \in \ca{Y}} \ca{H}(j_b)\right|$ because the quantities on the right hand side of the equation have already been pre-computed (using our induction hypothesis). Therefore, the lemma is proved.

\end{proof}

\begin{proof}[Proof of Lemma \ref{lem:prelim1} when $|\cap_{i \in \ca{C}} \ca{S}(i)|$ is provided]

Suppose we are given $|\cap_{i \in \ca{C}} \ca{S}(i)|$ for all sets $\ca{V}\subseteq [n]$ satisfying $|\ca{V}|\le s$. We will omit the subscript $\ca{V}$ from hereon for simplicity. As in Lemma \ref{lem:prelim1}, the set $\cap_{i \in \ca{C}} \ca{S}(i)$ is equivalent to the set $\s{occ}(C,\fl{1}_{|C|})$ or the number of unknown vectors in $\ca{V}$ whose restriction to the indices in $\ca{C}$ is the all one vector and in particular, $\s{occ}((i),1)=\ca{S}(i)$. We will re-use the equation that for $t$ sets $\ca{A}_1,\ca{A}_2,\dots,\ca{A}_t$, we must have
\begin{align*}
    \left|\bigcup_{i=1}^{t} \ca{A}_{i}\right|= \sum_{u=1}^{t}(-1)^{u+1} \sum_{1 \le i_1 < i_2 <\dots <i_u \le t} \left|\bigcap_{b=1}^{u} \ca{A}_{i_b}\right|.
\end{align*}

We now show using induction on $s$ that the quantities $\left\{ \left|\bigcap_{i \in \ca{S}} \s{occ}((i), 1)\right| ~\forall~\ca{T} \subseteq [n], |\ca{T}| \le s \right\}$
are sufficient to compute $|\s{occ}(C, \fl{a})|$ for all subsets $C$ of indices of size at most $s$, and any binary vector $\fl{a} \in \{0,1\}^{\le s}$.

\textit{Base case ($t=1$):} 

 The base case follows  since we can infer $|\s{occ}((i), 0)| = \ell - |\s{occ}((i), 1)|$  from $|\s{occ}((i), 1)|$  for all $ i \in [n]$. 

\textit{Inductive Step:}
Let us assume that the statement is true for $r < s$ i.e., we can compute $|\s{occ}(\ca{C},\fl{a})|$ for all subsets $\ca{C}$ satisfying $|\ca{C}|\le r$ and any binary vector $\fl{a}\in \{0,1\}^{\le r}$ from the quantities $\left\{ \left|\bigcap_{i \in \ca{S}} \s{occ}((i), 1)\right| ~\forall~\ca{T} \subseteq [n], |\ca{T}| \le r \right\}$ provided as input. Now, we prove that the statement is true for $r+1$ under the induction hypothesis. Note that we can also rewrite $\s{occ}(\ca{C},\fl{a})$ for any set $\ca{C}\subseteq [n], \fl{a}\in \{0,1\}^{|\ca{C}|}$ as \begin{align*}
    \s{occ}(\ca{C},\fl{a}) = \bigcap_{j \in \ca{C}'}\ca{S}(j) \bigcap_{j \in \ca{C}\setminus \ca{C}' }\ca{S}(j)^c
\end{align*}
where $\ca{C}'\subseteq \ca{C}$ corresponds to the indices in $\ca{C}$ for which the entries in $\fl{a}$ is \texttt{1}. Fix any set $i_1,i_2,\dots,i_{r+1} \in [n]$. Then we can compute $\left|\bigcup_{b=1}^{r+1} \ca{S}(i_b)\right|$ using the following equation:
\begin{align*}
    \left|\bigcup_{b=1}^{r+1} \ca{S}(i_b)\right|
    = \sum_{u=1}^{r+1}(-1)^{u+1} \sum_{\substack{j_1,j_2,\dots,j_u \in \{i_1,i_2,\dots,i_{r+1}\}\\ j_1 < j_2 <\dots<j_u  }} \left|\bigcap_{b=1}^{u} \ca{S}(j_b)\right|.
\end{align*}

Finally for any proper subset $\ca{Y} \subset \{i_1,i_2,\dots,i_{r+1}\}$, we can compute $\left|\bigcap_{ i_b \not \in \ca{Y}} \ca{S}(i_b) \bigcap_{i_b \in \ca{Y}} \ca{S}(i_b)^{c}\right|$ using the following set of equations:

\begin{align*}
    \left|\bigcap_{i_b \not \in \ca{Y}} \ca{S}(i_b) \bigcap_{i_b \in \ca{Y}} \ca{S}(i_b)^{c}\right| &=     \left|\bigcap_{i_b \not \in \ca{Y}} \ca{S}(i_b) \bigcap \Big(\bigcup_{i_b \in \ca{Y}} \ca{S}(i_b) \Big)^{c} \right| \\
    &= \left|\bigcap_{i_b \not \in \ca{Y}} \ca{S}(i_b) \right| - \left| \bigcap_{i_b \not \in \ca{Y}} \ca{S}(i_b) \bigcap \Big(\bigcup_{i_b \in \ca{Y}} \ca{S}(i_b) \Big) \right| \\
    &=\left|\bigcap_{i_b \not \in \ca{Y}} \ca{S}(i_b) \right| - \left|\bigcup_{i_b \in \ca{Y}} \Big(\bigcap_{i_b \not \in \ca{Y}} \ca{S}(i_b) \bigcap \ca{S}(i_b) \Big) \right|. 
\end{align*}
The first term is already pre-computed and the second term is again a union of intersection of sets. For any $i_b \in \ca{Y}$, let us define $\ca{H}(j_b) := \bigcap_{i_b \not \in \ca{Y}} \ca{S}(i_b) \bigcap \ca{S}(j_b)$. Therefore we have 
\begin{align*}
    \left|\bigcup_{j_b \in \ca{Y}} \ca{H}(j_b)\right|= \sum_{u=1}^{\left|\ca{Y}\right|}(-1)^{u+1} \sum_{\substack{j_1,j_2,\dots,j_u \in \ca{Y}\\ j_1 < j_2 <\dots<j_u  }} \left|\bigcap_{b=1}^{u} \ca{H}(j_b)\right|.
\end{align*}
We can compute $\left|\bigcup_{j_b \in \ca{Y}} \ca{H}(j_b)\right|$ because the quantities on the right hand side of the equation have already been pre-computed (using our induction hypothesis). Therefore, the lemma is proved.

\end{proof}

\begin{proof}[Proof of Corollary \ref{coro:prelim1}]
We know that all vectors $\fl{v}\in \ca{V}$ satisfy $\|\fl{v}\|_0 \le k$ as they are $k$-sparse. Therefore, in the first stage, by computing $|\ca{S}(i)|$ for all $i\in [n]$, we can recover the union of support of all the unknown vectors $ \cup_{\fl{v}\in \ca{V}}\s{supp}(\fl{v})$ by computing $\ca{T} = \{i \in [n]\mid \ca{S}(i)>0\}$. The probability of failure in finding the union of support exactly is at most $n\gamma$. Once we recover $\ca{T}$, we compute $ \left|\cup_{i \in \ca{C}}\ca{S}(i)\right|$ for all $\ca{C}\subseteq \ca{T},|\ca{C}| \le \log \ell+1$ (or alternatively $|\cap_{i \in \ca{C}}\ca{S}(i)|$ for all $\ca{C}\subseteq \ca{T},|\ca{C}| \le \log \ell+1$). The probability of failure for this this event $(\ell k)^{\log \ell+1}\gamma$. 
From Lemma \ref{thm:prelim1}, we know that computing $ \left|\cup_{i \in \ca{C}}\ca{S}(i)\right|$ for all $\ca{C}\subseteq [n],|\ca{C}| \le \log \ell+1$ (or alternatively $|\cap_{i \in \ca{C}}\ca{S}(i)|$ for all $\ca{C}\subseteq \ca{T},|\ca{C}| \le \log \ell+1$) exactly will allow us to recover the support of all the unknown vectors in $\ca{V}$. However $\left|\cup_{i \in \ca{C}}\ca{S}(i)\right|=0$ for all $\ca{C}\subseteq [n] \setminus \ca{T}$ provided $\ca{T}$ is computed correctly. Therefore, we can recover the support of all the unknown vectors in $\ca{V}$ with $\s{T}\log \gamma^{-1}$ samples with probability at least $1-((\ell k)^{\log \ell+1}+n)\gamma$. Rewriting the previous statement so that the failure probability is $\gamma$ leads to the statement of the lemma.
\end{proof}


\begin{proof}[Proof of Lemma \ref{lem:partial}]

 As stated in the Lemma, suppose it is known if $\left|\cap_{i \in \ca{C}}\ca{S}(i)\right|>0$ or not for all sets $\ca{C}\subseteq [n]$ satisfying $|\ca{C}|\le s+1$. 
 Assume that $\s{Maximal}(\ca{V})$ is $s$-good
  Consider a set $A\in \s{Maximal}(\ca{V})$. Since $\s{Maximal}(\ca{V})$ is $s$-good, there must exist an ordered set $\ca{C}\subseteq [n], |\ca{C}| =s$ such that $\ca{C} \subseteq A$  but $\ca{C} \not\subseteq A'$ for all $A' \in \s{Maximal}(\ca{V})\setminus \{A\}$. Therefore, we must have $\left|\cap_{i \in \ca{C}}\ca{S}(i)\right|>0$. But, on the other hand, notice that if $\left|\s{Maximal}(\ca{V})\right|\ge 2$, there must exist an index $j\in \cup_{A\in \s{Maximal}(\ca{V})} A$ such that $\left|\cap_{i \in \ca{C}\cup\{j\}}\ca{S}(i)\right|=0$ since $A$ does not contain the support of all other vectors.
 Algorithm \ref{algo:partial1} precisely checks for this condition and therefore this completes the proof.

\end{proof}

\begin{proof}[Proof of Corollary \ref{coro:partial2}]
Again, we know that all vectors $\fl{v}\in \ca{V}$ satisfy $\|\fl{v}\|_0 \le k$ as they are $k$-sparse. Therefore, in the first stage, by computing if $|\ca{S}(i)|>0$ for all $i\in [n]$, we can recover the union of support of all the unknown vectors $ \cup_{\fl{v}\in \ca{V}}\s{supp}(\fl{v})$ by computing $\ca{T} = \{i \in [n]\mid \ca{S}(i)>0\}$. The probability of failure in finding the union of support correctly is at most $n\gamma$. Once we recover $\ca{T}$ correctly, we compute $ \left|\cap_{i \in \ca{C}}\ca{S}(i)\right|$ for all $\ca{C}\subseteq \ca{T},|\ca{C}| \le \ell$. The probability of failure for this event $(\ell k)^{\ell}\gamma$. 
From Lemma \ref{lem:partial2}, we know that computing $ \left|\cap_{i \in \ca{C}}\ca{S}(i)\right|$ for all $\ca{C}\subseteq [n],|\ca{C}| \le \ell$ exactly will allow us to recover the support of all the unknown vectors in $\ca{V}$. On the other hand, we will have $\left|\cap_{i \in \ca{C}}\ca{S}(i)\right|=0$ for all $\ca{C}\subseteq [n] \setminus \ca{T}$ provided $\ca{T}$ is computed correctly. Therefore, we can achieve maximal support recovery of all the unknown vectors in $\ca{V}$ with $\s{T}\log \gamma^{-1}$ samples with probability at least $1-((\ell k)^{\ell}+n)\gamma$. Rewriting, so that the failure probability is $\gamma$ leads to the statement of the lemma.
\end{proof}

\begin{proof}[Proof of Lemma \ref{lem:partial2}]
Consider the special case when $\left|\s{Maximal}(\ca{V})\right| = 1$ i.e. there exists a particular vector $\fl{v}$ in $\ca{V}$ whose support subsumes the support of all the other unknown vectors in $\ca{V}$.
In that case, for each set $\ca{C} \subseteq A \in \s{Maximal}(\ca{V})$, $\left|\ca{C}\right| \le \ell$, we must have that $\left|\bigcup_{i \in \ca{C}}\ca{S}(i)\right|>0$ (as there is only a single set in $\s{Maximal}(\ca{V})$). On the other hand, if $\left|\s{Maximal}(\ca{V})\right| \ge  2$, then we know that $\s{Maximal}(\ca{V})$ is $(\ell-1)$-good and therefore, for each set $A\in \s{Maximal}(\ca{V})$, there exists an ordered set $\ca{C}$ and an index $j\subseteq \cup_{A'\in \s{Maximal}(\ca{V})}A'$, $\left|\ca{C}\right| \le \ell-1$ such that $\ca{C} \subseteq A$ but $\ca{C} \not\subseteq A'$ for any other set $A'$; hence $\left|\bigcap_{i \in \ca{C}}\ca{S}(i)\right|>0$ but  $\left|\bigcap_{i \in \ca{C}\cup\{j\}}\ca{S}(i)\right|=0$. In other words, there exists a set of size $\ell$ that is a subset of the union of sets in $\s{Maximal}(\ca{V})$ but there does not exist any unknown vector that has \texttt{1} in all the indices indexed by the aforementioned set. 
Again, Algorithm \ref{algo:partial1} precisely checks this condition and therefore this completes the proof.    
\end{proof}

\remove{
\begin{lemma}\label{lem:trimmed}
 The set $\s{Maximal}(\ca{V})$ is unique.
\end{lemma}

\begin{proof}
We will prove this lemma by contradiction. Suppose there exists two distinct sets $\ca{T}_1,\ca{T}_2\subset \ca{V}$ such that $\left|\ca{T}_1\right|= \left|\ca{T}_2\right|=\left|\s{Maximal}(\ca{V})\right|$. Since $\ca{T}_1,\ca{T}_2$ are distinct, there must exist a vector $\fl{v} \in \ca{T}_2\setminus \ca{T}_1$. If $\s{supp}(\fl{v})$ is not contained with the support of some vector in $\ca{T}_1$ and there is no other vector in $\ca{V}$ whose support contains $\fl{v}$, then clearly, $\fl{v}$ can be added to $\ca{T}_1$ implying that $\ca{T}_1$ cannot be the largest maximal set. On the other hand, suppose $\s{supp}(\fl{v})$ is contained within the support of some vector $\fl{v}'$ in $\ca{T}_1$. However, this implies that $\ca{T}_2$ cannot be a valid deduplicated set as the support of $\fl{v}$ is contained with the support of $\fl{v}'$ and therefore, $\fl{v}$  cannot belong to a deduplicated set. This implies that the vector $\fl{v}$ cannot exist without violating some constrained of $\s{Maximal}(\ca{V})$ and therefore, the set $\s{Maximal}(\ca{V})$ is unique.  
\end{proof}

}

\section{Proof of Lemma \ref{thm:prelim1} (Theorem 1 in \citep{gandikota2021support})}\label{app:lemma1}


We will start with a few additional notations and definitions: 

For a set of unknown vectors $\ca{V}\equiv \{\fl{v}^1,\fl{v}^2,\dots,\fl{v}^{\ell}\}$, let
$\fl{A} \in \{0,1\}^{n \times \ell}$ denote the support matrix corresponding to $\ca{V}$ where each column vector $\fl{A}_i \in \{0,1\}^n$ represents the support of the $i^{\s{th}}$ unknown vector $\fl{v}^i$.

\begin{defn}[$p$-identifiable]\label{def:iden}
The $i^{\s{th}}$ column $\fl{A}_i$ of  a binary matrix $\fl{A} \in \{0,1\}^{n \times \ell}$ with all distinct columns is called $p$-identifiable if 
there exists a 
  set $S \subset [n]$ of at most $p$-indices and a binary string $\fl{a} \in \{0,1\}^p$ such that $\fl{A}_i|_S = \f a$, and $\fl{A}_j|_S \neq \f a$ for all $j \neq i$.
  
  A binary matrix $\fl{A} \in \{0,1\}^{n \times \ell}$ with all distinct columns is called  $p$-identifiable if there exists a permutation
$\sigma:[\ell]\rightarrow[\ell]$
such that for all $i\in [\ell]$, the sub-matrix $\fl{A}^i$ formed by deleting the columns indexed by the set $\{\sigma(1),\sigma(2),\dots,\sigma(i-1)\}$ has at least one $p$-identifiable column.

Let $\mathcal{V}$ be set of $\ell$ unknown vectors in $\bb{R}^n$, and $\fl{A} \in \{0,1\}^{n \times \ell}$ be its support matrix. Let $\fl{B}$ be the matrix obtained by deleting duplicate columns of $\fl{A}$. The set $\mathcal{V}$ is called $p$-identifiable if  $\fl{B}$ is $p$-identifiable. 
\end{defn}

\begin{thmu}[Theorem 2 in \citep{gandikota2021support}]\label{lem:suff-t}
Any $n \times \ell$, (with $n > \ell$) binary matrix with all distinct columns is  $p$-identifiable for some $p \le \log \ell$.
\end{thmu}

\begin{proof}

Suppose $\fl{A}$ is the said matrix. Since all the columns of $\fl{A}$ are distinct, there must exist an index $i \in [n]$ which is not the same for all columns in $\fl{A}$. We must have $\left|\s{occ}((i), a)\right| \le \ell/2$ for some $a \in \{0,1\}$. Subsequently, we consider the columns of $\fl{A}$ indexed by the set $\s{occ}((i), a)$ and can repeat the same step. Evidently, there must exist an index $j \in [n]$ such that $\left|\s{occ}((i),\fl{a})\right| \le \ell/4$ for some $\fl{a} \in \{0,1\}^2$. Clearly, we can repeat this step at most $\log \ell$ times to find $C \subset [n]$ and $\fl{a}\in \{0,1\}^{\le \log \ell}$ such that $\left|\s{occ}(C,\fl{a})\right| = 1$ and therefore the column in $\s{occ}(C,\fl{a})$ is $p$-identifiable. We denote the index of this column  as $\sigma(1)$ and form the sub-matrix $\fl{A}^1$ by deleting the column. Again, $\fl{A}^1$ has $\ell-1$ distinct columns and by repeating similar steps, $\fl{A}^1$ has a column that is $\log(\ell-1)$ identifiable. More generally, $\fl{A}^i$ formed by deleting the columns indexed in the set $\{\sigma(1),\sigma(2),\dots,\sigma(i-1)\}$, has a column that is $\log(\ell-i)$ identifiable with the index (in $\fl{A}$) of the column having the unique sub-string (in $\fl{A}^i$) denoted by $\sigma(i)$. Thus the lemma is proved.
\end{proof}

Next, we present an algorithm (see Algorithm \ref{algo:t-iden-supp-rec}) for support recovery of all the $\ell$ unknown vectors $\ca{V}\equiv \{\f v^1, \ldots, \f v^\ell\}$  when $\ca{V}$ is $p$-identifiable. 
In particular, we show that if $\mathcal{V}$ is $p$-identifiable, then computing $|\s{occ}(C, \fl{a})|$ for every subset of $p$ and $p+1$ indices is sufficient to recover the supports. 

The proof follows from the observation that for any subset of $p$ indices $C \subset [n]$, index $j \in [n] \setminus C$ and $\fl{a} \in \{0,1\}^p$, 
$|\Su(C, \fl{a})| = |\Su(C\cup\{j\}, (\fl{a}, 1))| + |\Su(C\cup\{j\}, (\fl{a}, 0))|$. Therefore if one of the terms in the RHS is $0$ for all $j \in [n] \setminus C$, then all the vectors in $\Su(C, \fl{a})$ share the same support. 

Also, if some two vectors $\fl{u}, \fl{v} \in \Su(C, \fl{a})$ do not have the same support, then there will exist at least one index $j \in [n] \setminus C$ such that $ \fl{u} \in \Su(C\cup\{j\}, (\fl{a}, 1))|$ and $ \fl{v} \in \Su(C\cup\{j\}, (\fl{a}, 0))$ or the other way round, and therefore $|\Su(C\cup\{j\}, (\fl{a}, 1))| \not \in \{0,|\Su(C, \fl{a})|\} $. 
Algorithm~\ref{algo:t-iden-supp-rec} precisely checks for this condition. The existence of some vector $\fl{v} \in \ca{V}$ ($p$-identifiable column), a subset of indices $C \subset [n]$ of size $p$, and a binary sub-string $\fl{b}\in \{0,1\}^{\le p}$ follows from the fact that $\ca{V}$ is $p$-identifiable.
Let us denote the subset of unknown  vectors with distinct support in $\ca{V}$ by $\ca{V}^1$.

Once we recover the $p$-identifiable column of $\ca{V}^1$, we mark it as $\fl{u}^{1}$ and remove it from the set (if there are multiple $p$-identifiable columns, we arbitrarily choose one of them). 
Subsequently,  we can modify the $\left|\Su{(\cdot)}\right|$ values for all $S \subseteq [n],|S|\in \{p,p+1\}$ and  $\fl{t} \in \{0,1\}^{p} \cup \{0,1\}^{p+1} $ as follows:
\begin{align}\label{eq:comp_occ2}
    &\left|\Su^2(S, \fl{t})\right| \triangleq \left|\Su(S, \fl{t})\right| 
    - \left|\Su(C, \fl{b})\right|\times \mathbf{1}[ \supp{\fl{u}^{1}}|_S = \fl{t}].
\end{align}
Notice that, Equation~\ref{eq:comp_occ2} computes 
$
    \left|\Su^2(S, \fl{t})\right| = \left|\{ \fl{v}^i \in \mathcal{V}^2 \mid \supp{\fl{v}^i}|_S = \fl{t} \}\right|
$
where $\mathcal{V}^2$ is formed by deleting all copies of $\fl{u}^{1}$ from $\ca{V}$. Since $\ca{V}^1$ is $p$-identifiable, there exists a $p$-identifiable column in $\ca{V}^1 \setminus \{\fl{u}^{1}\}$ as well which we can recover. More generally for $q>2$, if $\fl{u}^{q-1}$ is the $p$-identifiable column with the unique binary sub-string $\fl{b}^{q-1}$ corresponding to the set of indices $C^{q-1}$, we will have for all $S \subseteq [n],|S|\in \{p,p+1\}$ and  $\fl{t} \in \{0,1\}^{p} \cup \{0,1\}^{p+1} $
\begin{align*}
    \left|\Su^q(S, \fl{t})\right| \triangleq \left|\Su^{q-1}(S, \fl{t})\right| 
    - \left|\Su^{q-1}(C^{q-1}, \fl{b}^{q-1})\right|\times\mathbf{1}[ \supp{\fl{u}^{q-1}}|_S = \fl{t}] 
\end{align*}    
 implying 
 $\left|\Su^q(S, \fl{t})\right| = \left|\{ \fl{v}^i \in \mathcal{V}^q \mid \supp{\fl{v}^i}|_S = \fl{t} \}\right|
$
where $\mathcal{V}^q$ is formed deleting all copies of $\fl{u}^{1},\fl{u}^{2},\dots,\fl{u}^{q-1}$ from $\ca{V}$. Applying these steps recursively and repeatedly using the property that $\ca{V}$ is $p$-identifiable, we can recover all the vectors present in $\ca{V}$. 


\section{Technical Lemmas}\label{app:technical}

\begin{lemma}[Hoeffding's inequality for bounded random variables]\label{lem:Hoeffding}
Let $X_1,X_2,\dots,X_m$ be independent random variables strictly bounded in the interval $[a,b]$. Let $\mu=m^{-1}\sum_i\bb{E}X_i$. In that case, we must have
\begin{align*}
    \Pr\Big(\left|\frac{1}{m}\sum_{i=1}^{m}X_i-\mu\right|\ge t\Big) \le 2\exp\Big(-\frac{2mt^2}{(b-a)^2}\Big).
\end{align*}
\end{lemma}

\begin{lemma}[Gaussian concentration inequality]\label{lem:gaconc}
Consider a random variable $Z$ distributed according to $\ca{N}(0,\sigma^2)$. In that case, we must have $\Pr(|Z|\ge t) \le 2\exp(-t^2/2)$ for any $t>0$.
\end{lemma}

\begin{lemma}[Gaussian anti-concentration inequality]\label{lem:gaanconc}
Consider a random variable $Z$ distributed according to $\ca{N}(0,\sigma^2)$. In that case, we must have $\Pr(|Z|\le t) \le \sqrt{\frac{2}{\pi}}\cdot \frac{t}{\sigma}$ for any $t<\sigma\sqrt{\pi}/\sqrt{2}$.
\end{lemma}

\begin{proof}
By simple calculations, we can have 
\begin{align*}
\Pr(|Z|<t) \le  \int_{-t}^{t} \frac{e^{-x^2/2\sigma^2}}{\sqrt{2\pi \sigma}}dx   \le  \sqrt{\frac{2}{\pi}}\cdot \frac{t}{\sigma}. 
\end{align*}
\end{proof}

\begin{proof}
For each $i\in [n]$, suppose we want to test whether $i \in \cup_{\fl{v}\in \ca{V}}\s{supp}(\fl{v})$ or not.
Consider the random variable $y^2\fl{x}_i^2$ when $(\fl{x},y)\sim \ca{P}_r$. Notice that 
\begin{align*}
    \bb{E}y^2\fl{x}_i^2 = \frac{1}{\ell}\sum_{\fl{v}\in \ca{V}} \bb{E}y^2\fl{x}_i^2 \mid \fl{v} = \frac{1}{\ell}\sum_{\fl{v}\in \ca{V}} \Big(\sum_{j \in [n]} \fl{v}_j^2 +2\fl{v}_i^2\Big)
    \begin{cases}
    &=\frac{1}{\ell}\sum_{\fl{v}\in \ca{V}} \left|\left|\fl{v}\right|\right|_2^2 \text{ if } \left|\ca{S}_{\ca{V}}(i)\right|=0 \\
    &\ge  \frac{1}{\ell}\sum_{\fl{v}\in \ca{V}} \left|\left|\fl{v}\right|\right|_2^2+ \frac{2\delta^2}{\ell} \text{ if } \left|\ca{S}_{\ca{V}}(i)\right| \neq 0
    \end{cases}
\end{align*}
where the final inequality follows from the fact that the magnitude of any non-zero entry of any unknown vector must be at least $\delta$. For simplicity of notation, we will denote $A=\frac{1}{\ell}\sum_{\fl{v}\in \ca{V}} \left|\left|\fl{v}\right|\right|_2^2$  to be average norm of the unknown vectors. We will estimate $\bb{E}y^2\fl{x}_i^2$ by computing the following sample average \begin{align*}
    \frac{\ell}{m} \cdot \sum_{j=1}^{m}  \Big(y^{(j)}\fl{x}_i^{(j)}\Big)^2.
\end{align*}
From the definition of $\ca{P}_r$, we must have $y\sim \ca{N}(0,\zeta^2+\sigma^2), |\zeta|\le R$ since $\fl{v}\in \{0,1\}^n, \left|\left|\fl{v}\right|\right|_2 \le R$ for all $\fl{v}\in \ca{V}$. By using Gaussian concentration inequalities, we must have $\Pr(|y|>t) \le \exp(-t^2/2(R^2+\sigma^2))$. Therefore, with probability $1-n^{-10}$, we have $|y|<20\sqrt{R^2+\sigma^2}\log n$. Similarly, with probability $1-n^{-10}$, $|\fl{x}_i|$ is bounded from above by $20\log n$. a Subsequently, we use Hoeffding's inequality  to say that
\begin{align*}
    \Pr\Big(\left|\frac{\ell}{m} \cdot \sum_{j=1}^{m}  \Big(y^{(j)}\fl{x}_i^{(j)}\Big)^2-\bb{E}y^2\fl{x}_i^2\right|\ge \frac{\delta^2{2}}{2\ell}\Big) \le \exp\Big(-\Omega\Big(\frac{m\delta^{2}}{\ell^2(R^2+\sigma^2)(\log n)^{2}}\Big)\Big). 
\end{align*}
Hence, with $m=O(\ell^2(R^2+\sigma^2)(\log n)^{3}/\delta^2)$ samples, we can estimate if 
$\left|\bigcap_{i \in \ca{C}}\ca{S}_{\ca{V}}(i)\right|>0$ or not correctly with probability at least $1-1/\s{poly}(n)$. We can take a union bound over all the $n$ indices to estimate $\bb{E}y^2\fl{x}_i^2$ correctly within an additive error of $\delta^2/2\ell$ for all $i\in [n]$. We will cluster all the indices such that a pair of distinct indices $u,v \in [n]$ are in the same group if 
\begin{align*}
    \left|\frac{\ell}{m} \cdot \sum_{j=1}^{m}  \Big(y^{(j)}\fl{x}_u^{(j)}\Big)^2-\frac{\ell}{m} \cdot \sum_{j=1}^{m}  \Big(y^{(j)}\fl{x}_v^{(j)}\Big)^2\right| \le \frac{\delta^2}{\ell}. 
\end{align*}
Clearly, any two indices $u,v \in [n]$ that satisfy $|\ca{S}_{\ca{V}}(u)|=|\ca{S}_{\ca{V}}(v)|=0$ must belong to the same cluster. Since the size of the union of the support is at most $n/2$, the largest cluster must correspond to the indices where the entry is zero in all the unknown vectors. Subsequently, all those indices that do not belong to the largest cluster (after the clustering step) must belong to $\cap_{\fl{v}\in \ca{V}}\s{supp}(\fl{v})$. Furthermore, no index $i \in [n]$ such that $|\ca{S}_{\ca{V}}(i)|\neq 0$ can belong to the largest cluster. This complete the proof of the lemma.
\end{proof}

Finally, we will also use the following well-known lemma stating that we can compute estimates of the expectation of any one-dimensional random variable with only a few samples similar to sub-gaussian random variables. 

\begin{lemma}\label{lem:medianofmeans}
 For a random variable $x\sim \ca{P}$, there exists an algorithm (see Algorithm \ref{algo:estimate}) that can compute an estimate $u$ of $\bb{E}x$ such that $\left|u-\bb{E}x\right| \le \epsilon$ with $O( \log \gamma^{-1}\bb{E}x^2/\epsilon^2)$ with probability at least $1-\gamma$.
\end{lemma}

\begin{proof}[Proof of Lemma \ref{lem:medianofmeans}]
Suppose we obtain $m$ independent samples $x^{(1)},x^{(2)},\dots,x^{(m)}\sim \ca{P}$.
We use the median of means trick to compute $u$, an estimate of $\bb{E}x$. We will partition $m$ samples obtained from $\ca{P}$ into $B=\lceil m/m' \rceil$ batches each containing $m'$ samples each. In that case let us denote $S^j$ to be the sample mean of the $j^{th}$ batch i.e. 
\begin{align*}
S^j = \sum_{s \in \text{Batch } j} \frac{x^{(s)}}{m'}.
\end{align*}
We will estimate the true mean $\bb{E}x$ by computing $u$  where
$u\triangleq \s{median}(\{S^j\}_{j=1}^{B})$. For a fixed batch $j$, we can use Chebychev's inequality to say that 
\begin{align*}
\Pr\Big(\left|S^j-\bb{E}x\right| \ge \epsilon \Big) \le \frac{\bb{E}x^2}{t\epsilon^2} \le \frac{1}{3} 
\end{align*} 
for $t=O(\bb{E}x^2/\epsilon^2)$. Therefore for each batch $j$, we define an indicator random variable $Z_j=\mathbf{1}[\left|S^j-\bb{E}x\right| \ge \epsilon]$ and from our previous analysis we know that the probability of $Z_j$ being \textsc{1} is less than $1/3$. It is clear that $\bb{E} \sum_{j=1}^{B} Z_j \le B/3$ and on the other hand $|u-\bb{E}x| \ge \epsilon$ iff  $\sum_{j=1}^{B} Z_j \ge B/2$. Therefore, due to the fact that $Z_j$'s are independent, we can use Chernoff bound to conclude the following:
\begin{align*}
\Pr\Big(|u-\bb{E}x| \ge \epsilon\Big)\le \Pr \Big(\left|\sum_{j=1}^{B} Z_j-\bb{E}\sum_{j=1}^{B} Z_j \right| \ge \frac{\bb{E}\sum_{j=1}^{B} Z_j}{2} \Big) \le 2e^{-B/36}. 
\end{align*}
Hence, for $B=36\log \gamma^{-1}$, the estimate $u$ is at most $\epsilon$ away from the true mean $\bb{E}x$ with probability at least $1-\gamma$. Therefore the sufficient sample complexity is $m=O(\log \gamma^{-1} \bb{E}x^2/\epsilon^2)$.
\end{proof}

\end{document}